\documentclass{article}

\usepackage[preprint]{neurips_2025}

\usepackage[utf8]{inputenc} 
\usepackage[T1]{fontenc}    
\usepackage{hyperref}       
\usepackage{url}            
\usepackage{booktabs}       
\usepackage{amsfonts}       
\usepackage{nicefrac}       
\usepackage{microtype}      
\usepackage{xcolor}         

\usepackage{kotex}
\usepackage{algorithm}
\usepackage{algorithmic}
\usepackage{multirow}
\usepackage{multicol}
\usepackage{graphicx}

\usepackage{amsmath}    
\usepackage{amssymb}    
\usepackage{mathtools}  
\usepackage{amsthm}     
\usepackage{bm}         

\newcommand{\bs}{\boldsymbol}

\renewcommand{\l}{\lambda}

\newtheorem{theorem}{Theorem}[section]
\newtheorem{lemma}[theorem]{Lemma}

\newtheorem{definition}[]{Definition}

\usepackage{booktabs}               
\usepackage{subcaption}             
\usepackage{caption}                
\usepackage[table]{xcolor}

\usepackage{graphicx}
\usepackage{booktabs}     
\usepackage{multirow}
\usepackage{adjustbox}    
\usepackage{subcaption}   
\usepackage{wrapfig}
\usepackage{capt-of}
\usepackage{makecell}  
\usepackage{enumitem}

\usepackage{array}   
\newcolumntype{M}[1]{>{\centering\arraybackslash}m{#1}}
\usepackage{array, booktabs}
\newcolumntype{M}[1]{>{\centering\arraybackslash}m{#1}}
\usepackage{booktabs, colortbl, tabularx, array}  
\newcolumntype{Y}{>{\centering\arraybackslash}X}  

\renewcommand{\arraystretch}{0.92}  
\usepackage{graphicx, adjustbox, booktabs, array}
\newcolumntype{C}{>{\centering\arraybackslash}m{0.32\textwidth}} 
\usepackage{color}
\usepackage{bm}
\usepackage{hyperref}
\usepackage{amsmath,amsfonts}
\usepackage{blindtext}
\usepackage{listings}
\usepackage{mathtools}
\usepackage[framemethod=TikZ]{mdframed}

\usepackage{booktabs}   
\usepackage{makecell}   
\usepackage{pifont}     

\DeclarePairedDelimiterX{\infdivx}[2]{(}{)}{%
	#1\;\delimsize\|\;#2%
}

\newcommand{\vect}[1]{\bm{#1}}

\newcommand{\argmin}{\operatornamewithlimits{argmin}}

\newcommand{\x}{\xv}

\newcommand{\xv}{\vect x}

\makeatletter
\newtheorem*{rep@theorem}{\rep@title}
\newcommand{\newreptheorem}[2]{%
\newenvironment{rep#1}[1]{%
 \def\rep@title{#2 \ref{##1}}%
 \begin{rep@theorem}}%
 {\end{rep@theorem}}}
\makeatother

\newreptheorem{proposition}{Proposition}
\newreptheorem{theorem}{Theorem}
\newreptheorem{lemma}{Lemma}

\newcommand{\cmark}{\ding{51}}  
\newcommand{\xmark}{\ding{55}}  

\title{RBF-Solver: A Multistep Sampler for Diffusion Probabilistic Models via Radial Basis Functions}

\author{%
  \textbf{Soochul Park}\textsuperscript{1,6} \quad
  \textbf{Yeon Ju Lee}\textsuperscript{2,6} \quad
  \textbf{SeongJin Yoon}\textsuperscript{3,6} \\ [4pt]
  \textbf{Jiyub Shin}\textsuperscript{5} \quad
  \textbf{Juhee Lee}\textsuperscript{4,6} \quad
  \textbf{Seongwoon Jo}\textsuperscript{6} \\ [6pt]
  \textsuperscript{1}SteAI \quad \textsuperscript{2}Korea University \quad \textsuperscript{3}SoftAI \\
  \textsuperscript{4}Bogonet \quad \textsuperscript{5}SBS \quad \textsuperscript{6}MODULABS \\ [4pt]
  \texttt{scpark20@gmail.com}
}

\begin{document}

\maketitle

\begin{abstract}
Diffusion probabilistic models (DPMs) are widely adopted for their outstanding generative fidelity, yet their sampling is computationally demanding. Polynomial-based multistep samplers mitigate this cost by accelerating inference; however, despite their theoretical accuracy guarantees, they generate the sampling trajectory according to a predefined scheme, providing no flexibility for further optimization. To address this limitation, we propose RBF-Solver, a multistep diffusion sampler that interpolates model evaluations with Gaussian radial basis functions (RBFs). By leveraging learnable shape parameters in Gaussian RBFs, RBF-Solver explicitly follows optimal sampling trajectories. At first order, it reduces to the Euler method (DDIM). At second order or higher, as the shape parameters approach infinity, RBF-Solver converges to the Adams method, ensuring its compatibility with existing samplers. Owing to the locality of Gaussian RBFs, RBF-Solver maintains high image fidelity even at fourth order or higher, where previous samplers deteriorate. For unconditional generation, RBF-Solver consistently outperforms polynomial-based samplers in the high-NFE regime ($\text{NFE}\ge15$). On CIFAR-10 with the Score-SDE model, it achieves an FID of 2.87 with 15 function evaluations and further improves to 2.48 with 40 function evaluations. For conditional ImageNet 256$\times$256 generation with the Guided Diffusion model at a guidance scale 8.0, substantial gains are achieved in the low-NFE range (5–10), yielding a 16.12–33.73\% reduction in FID relative to polynomial-based samplers.
\end{abstract}

\section{Introduction}
\label{sec:introduction}

\iftrue
Diffusion Probabilistic Models (DPMs)~\citep{sohl2015deep, ho2020denoising, song2021score} have demonstrated state-of-the-art performance in various generative modeling tasks, including image~\cite{ho2020denoising,dhariwal2021diffusion,rombach2022high,ho2022cascaded}, video~\cite{ho2022video,blattmann2023align,esser2023structure}, speech synthesis~\citep{chen2020wavegrad,liu2022diffsinger,chen2021wavegrad}, and text-to-image generation~\cite{rombach2022high, nichol2021glide, gu2022vector}. 
Despite their expressiveness, DPMs are hindered by the high computational cost of iterative denoising steps~\citep{ho2020denoising}.
To mitigate this, fast samplers based on exponential integrators have been proposed~\citep{Zhang2023fast, lu2022dpm}, offering accelerated convergence through accurate integral approximation. More recently, high-order multistep methods~\citep{lu2022dpmpp, zhao2023unipc} have gained attention for further improving sampling efficiency. By reusing model predictions, these methods reduce redundant evaluations and enable stable long-step integration. Such frameworks often rely on polynomial-based approximations, including Taylor expansion and Lagrange interpolation, which—despite their simplicity—may suffer from limited flexibility and numerical instability in high-order settings.
\fi

Radial Basis Function (RBF) interpolation is a well-established method for interpolation problems with accuracy analysis and local approximation properties~\citep{Lee2013rbf, Fornberg2002FlatRBF}.
Algebraic polynomial interpolation is commonly used in data approximation but 
its lack of flexibility
limits its ability to capture localized features in the data 
and causes it to suffer from the \textit{Runge phenomenon} unless nodes are specially distributed.
In contrast, RBF interpolation can represent data more effectively by selecting an appropriate shape parameter for each local region, which enables it to adapt to the underlying features of the data and produce smooth, non-oscillatory results. 
For this reason, RBFs are applied in various domains such as data approximation, signal processing, and the numerical solution of partial differential equations,
and they are known to be effective in high-dimensional and scattered data~\cite{lazzaro2002radial,fornberg2007runge,Dyn2007GaussianSubdivision,lee2010nonlinear,fornberg2011stable,Yoon2022WENO_RBF}.
Among various RBFs, the Gaussian is the most widely used due to its stability and robustness. In particular, Gaussian-based methods have been proposed in~\cite{Dyn2007GaussianSubdivision,lee2010nonlinear,Yoon2022WENO_RBF} for subdivision schemes, image upsampling, and hyperbolic conservation laws, respectively, along with reliable ranges of Gaussian parameters for each problem.

In this work, we propose RBF-Solver, a multistep sampler based on Gaussian RBF interpolation. The proposed method is theoretically grounded, satisfies the coefficient summation condition (Section~\ref{sec:Coefficients-Summation-Condition}), converges to the Adams method~\cite{butcher2016numerical} as the shape parameters tend to infinity (Section~\ref{sec:Lagrange Convergence}), and offers formal accuracy guarantees (Section~\ref{sec:Accuracy}). We also introduce an optimization algorithm for the shape parameters (Section~\ref{sec:Learning of Shape Parameter}).
RBF-Solver is evaluated on both unconditional and conditional generation tasks (Section~\ref{sec:Experiments}).
In unconditional generation, RBF-Solver improves FID scores on CIFAR-10~\cite{Krizhevsky09learningmultiple} using the Score-SDE model~\cite{song2021score}, achieving an FID of 2.87 with 15 function evaluations and 2.48 with 40 function evaluations. For conditional generation on ImageNet 128$\times$128 and 256$\times$256~\cite{deng2009imagenet} with the Guided-Diffusion model~\cite{dhariwal2021diffusion}, the proposed method yields substantial gains at guidance scales 6.0 and 8.0, effectively reducing FID in the low-NFE regime (5–10) compared to the baseline methods.
Finally, RBF-Solver maintains sample quality even at fourth order and above, where polynomial-based samplers deteriorate (Section~\ref{sec:high_order_stability}).

\section{Background}
\subsection{Diffusion Probabilistic Models}
Assume that we have a $D$-dimensional random variable $\bm{x}_0\in\mathbb{R}^{D}$ with an unknown distribution $q_0(\bm{x}_0)$. 
Diffusion Probabilistic Models (DPMs)~\citep{sohl2015deep, ho2020denoising, song2021score} define the process of gradually adding noise to $\bm{x}_0 \sim q_0(\bs x_0)$ at time $0$  to transform it into $q_T(\bs x_T) \approx \mathcal N(\bs x_T|\bm 0,\tilde \sigma^2 \bm I)$  at time $T$ for some $\tilde \sigma > 0$, such that for any $t\in [0,T]$, the distribution of $\bm{x}_t$ conditioned on $\bm{x}_0$ satisfies
\begin{align}
    q_{t|0}(\bs x_t|\bs x_0) = \mathcal N(\bm x_t|\alpha_t \bs x_0, \sigma_t^2 \bm I),\nonumber
\end{align}
where $\alpha_t^2/\sigma^2_t$ (the {\it signal-to-noise-ratio} (SNR)) is strictly decreasing w.r.t. $t$ \cite{kingma2021variational}.
There are two main parameterization strategies: (1) a \emph{noise prediction model} $\bs \epsilon_{\theta}(\bs x_t, t)$, which estimates the added noise $\bs \epsilon$, and whose parameter $\theta$ is optimized by minimizing the objective
    $\mathbb E_{\bs x_0, \bs \epsilon, t}[\omega(t) \| \bs \epsilon_{\theta}(\bs x_t,t) -\bs \epsilon\|^2_2] \nonumber$
where the weight function $\omega(t)>0$, and (2) a \emph{data prediction model} $\bm x_\theta(\bm x_t, t)$ that attempts to directly predicts $\bm x_0$. These two models are closely related via a deterministic transform, $\bm x_\theta(\bm x_t,t) \coloneqq (\bm x_t - \sigma_t\bm{\epsilon}_\theta(\bm x_t,t)) / \alpha_t$~\cite{kingma2021variational}.
Sampling in DPMs is performed by solving the diffusion ODE, the exact form of which depends on whether a $\bs \epsilon_{\theta}(\bs x_t, t)$~\cite{lu2022dpm} or a $\bm x_\theta(\bm x_t, t)$~\cite{lu2022dpmpp} is used. In this paper, we focus on the data prediction model $\bm x_\theta(\bm x_t, t)$ and explain the sampling process defined as follows
\begin{equation}
\label{eq:diffusion_ode_x0}
    \frac{\mathrm{d} \bm{x}_t}{\mathrm{d} t} = \left(f(t) + \frac{g^2(t)}{2\sigma_t^2} \right)\bm{x}_t - \frac{\alpha_t g^2(t)}{2\sigma^2_t}\bm{x}_\theta(\bm{x}_t,t), 
\end{equation}
where the coefficients $f(t)= \frac{\mathrm{d} \log \alpha_t}{\mathrm{d} t}$, $g^2(t)=\frac{\mathrm{d}\sigma_t^2}{\mathrm{d} t} - 2\frac{\mathrm{d} \log\alpha_t}{\mathrm{d} t}\sigma_t^2$~\citep{lu2022dpm,kingma2021variational}.

\subsection{High-Order Diffusion ODE Solvers}
Recently, high-order diffusion ODE solvers~\citep{lu2022dpm, lu2022dpmpp, zhao2023unipc, zhang2022fast} that leverage exponential integrators~\citep{hochbruck2010exponential} have shown significantly faster convergence than traditional methods that use either black-box ODE solvers~\citep{song2021score} or first-order and second-order diffusion ODE solvers~\citep{song2020denoising, karras2022elucidating}.
In particular, \citep{lu2022dpm, lu2022dpmpp}  propose the solution $\bm{x}_t$ of the diffusion ODE, where the domain is changed from the time $t$ to the half log-SNR $\lambda$ via a change-of-variables formula. 
Given an initial value $\bm{x}_s$ at time $s>0$, the solution $\bm{x}_t$ to the diffusion ODE (\ref{eq:diffusion_ode_x0}) as proposed by \citep{lu2022dpmpp} is given by:
\begin{equation}
\label{eq:exact_solution_x0}
    \bm{x}_t = \frac{\sigma_t}{\sigma_s}\bm{x}_s + 
                \sigma_t \int_{\lambda_s}^{\lambda_t} 
                e^{\lambda} 
                \bm{\hat x}_\theta(\bm{\hat x}_\lambda,\lambda)
                \mathrm{d}\lambda,
\end{equation}
where $ \bm{\hat x}_\theta(\bm{\hat x}_\lambda,\lambda)\coloneqq \bm{x}_\theta(\bm{x}_{t_\lambda(\lambda)},t_\lambda(\lambda))$ and 
$\bm{\hat x}_{\lambda}\coloneqq \bm{x}_{t_\lambda(\lambda)}$ 
represents the change-of-variables form in the $\lambda$-domain 
and the half log-SNR, \textit{i.e.},  $\lambda(t)\coloneqq\lambda_t\coloneqq\log(\alpha_t / \sigma_t)$
is a strictly decreasing function of $t$ with the inverse function $t_\lambda(\cdot)$ satisfying $t=t_\lambda(\lambda(t))$.

\subsection{Multistep Methods for Fast Sampling of DPMs}
\label{sec:Multistep Methods for Fast Sampling of DPMs}
To compute the transition from $\bm{x}_{t_i}$ to $\bm{x}_{t_{i+1}}$ in Eq.~\eqref{eq:exact_solution_x0} 
where the $M + 1$ time steps $\left\{t_i\right\}_{i=0}^M$ are strictly decreasing from $t_0 = T$ to $t_M = 0$, we need to approximate the exponentially weighted integral of $\bm{\hat x}_\theta$ from $\lambda_{t_i}$ to $\lambda_{t_{i+1}}$. 
Since integrating $\bm{\hat x}_\theta$ is intractable, 
this issue is addressed by approximating $\bm{\hat x}_\theta$
using methods such as Taylor expansion or Lagrange interpolation.
  
\paragraph{Taylor Approximation-based Approach} 
For $p\geq 1$, the Taylor expansion of $\bm{\hat x}_\theta(\bm{\hat x}_{\lambda},\lambda)$ w.r.t. $\lambda$ at time $t_i$ is a polynomial with degree $p-1$ given by:
\begin{equation}
    \label{eq:taylor-epsilonv}
    \bm{\hat x}_\theta(\bm{\hat x}_\lambda, \lambda) =
    \sum_{k=0}^{p-1} \frac{(\lambda-\lambda_{t_i})^k}{k!} 
    \bm{\hat x}_\theta^{(k)}(\bm{\hat x}_{\lambda_{t_i}},\lambda_{t_i})
    + {\mathcal O}(h^{p}),
\end{equation}
where the $k$-th order derivative $\bm{\hat x}_\theta^{(k)}(\bm{\hat x}_\lambda,\lambda)\coloneqq \frac{\mathrm{d}^k \bm{\hat x}_\theta(\bm{\hat x}_\lambda,\lambda)}{\mathrm{d}\lambda^k}$ and $h\coloneqq \lambda_{t_{i+1}}-\lambda_{t_i}$. 
To compute the derivatives $\bm{\hat x}_\theta^{(k)}$ in Eq.~\eqref{eq:taylor-epsilonv}, ~\citep{lu2022dpmpp} employs \textit{multistep methods}~\citep{butcher2016numerical}  which approximate the high-order derivatives via divided differences by using the previous values 
$\{\bm{\hat x}_\theta(\bm{\hat x}_{\lambda_{t_{i-j}}},\lambda_{t_{i-j}})\}_{j=0}^{p-1}$.

\paragraph{Lagrange Interpolation-based Approach}  For $p\geq 1$, the  
Lagrange interpolation of the previous values 
$\{\bm{x}_\theta(\bm{x}_{\lambda_{t_{i-j}}},\lambda_{t_{i-j}})\}_{j=0}^{p-1}$ at time $t_i$
is a polynomial $\mathcal{L}^{(p)}(\lambda)$ with degree $p-1$ w.r.t. $\lambda$:
\begin{equation}
\label{eq:lagrange-interpolation}
\begin{aligned}
    \bm{\hat x}_\theta(\bm{\hat x}_\lambda, \lambda) = \mathcal{L}^{(p)}(\lambda)+ {\mathcal O}(h^{p}), \quad\quad
    \mathcal{L}^{(p)}(\lambda) := 
    \sum_{j=0}^{p-1} \ell_{j}(\lambda) 
    \bm{\hat x}_\theta(\bm{\hat x}_{\lambda_{t_{i-j}}},\lambda_{t_{i-j}}), \\
\end{aligned}
\end{equation}
where $\ell_{j}(\lambda):\mathbb{R}\rightarrow\mathbb{R}$ denotes the Lagrange basis  and $h\coloneqq \lambda_{t_{i+1}}-\lambda_{t_i}$.  
To approximate the $\bm{\hat x}_\theta$,
the method proposed by \citep{Zhang2023fast} adopts the Adams–Bashforth~\citep{hochbruck2010exponential} method, 
which leverages the Lagrange interpolation. 
As a variance-controlled diffusion SDE solver, \citep{xue2024sa} employs stochastic Adams methods within a predictor-corrector scheme,
first approximating $\bm{\hat x}_\theta$ using Adams–Bashforth and then correcting it via Adams–Moulton.
The algorithm presented in \citep{zhao2023unipc} is essentially equivalent to that in \citep{xue2024sa} under the zero-variance setting, though it computes the Lagrange coefficients via coefficient matching (Appendix~\ref{sec:Lagrange Interpolation-based Approach}).
\section{RBF-Solver}

In this section, we present RBF-Solver. Section~\ref{sec:Derivation of Sampling Equations} derives the sampling equations via RBF interpolation, and Section~\ref{sec:Coefficients-Summation-Condition} explains the use of a constant basis. Sections~\ref{sec:Lagrange Convergence}–\ref{sec:Euler_Convergence} study its behavior as the shape parameter varies, and Section~\ref{sec:Accuracy}–\ref{sec:Learning of Shape Parameter} address its accuracy and shape parameter optimization.


\subsection{Derivation of Sampling Equations}
\label{sec:Derivation of Sampling Equations}

\begin{figure}[t]
    \centering
    \begin{minipage}[t]{0.60\textwidth}
        \centering
        \includegraphics[width=\linewidth]{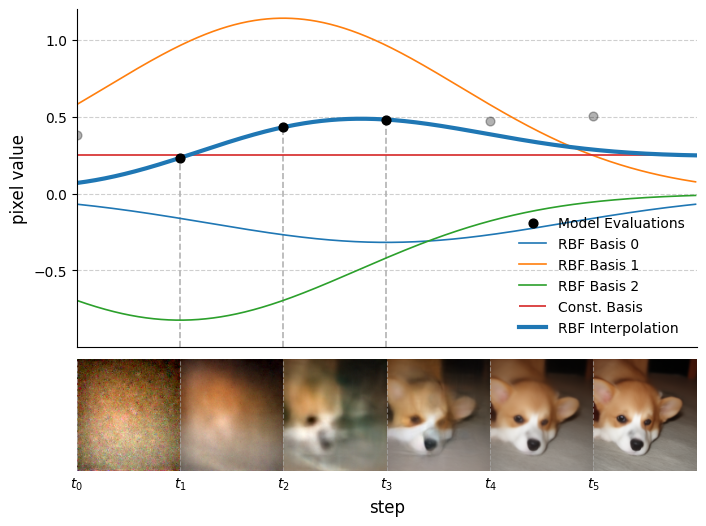}%
        \subcaption{RBF interpolation during sampling}\label{fig:left}
    \end{minipage}%
    \hfill
    \begin{minipage}[t]{0.38\textwidth}
        \centering
        \includegraphics[width=\linewidth]{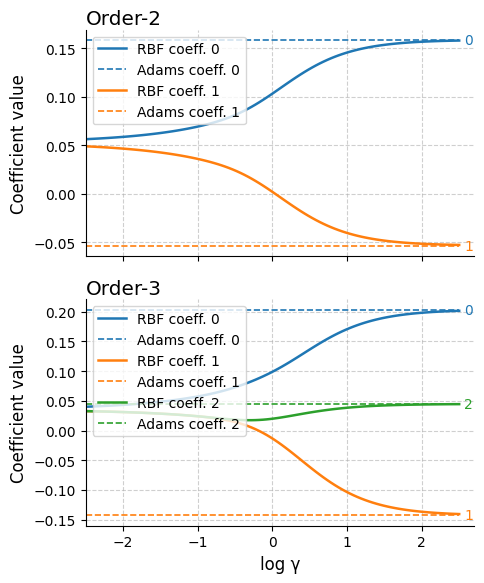}%
        \subcaption{Coefficient variation versus $\log\gamma$}\label{fig:coefficient_trends}
    \end{minipage}
    \caption{%
        \textbf{(a)} During sampling, the three model evaluations at time points $t_{3}$, $t_{2}$, and $t_{1}$ are interpolated using both radial basis functions (RBFs) and a constant basis, each basis being multiplied by its weight. Only the values of the central green pixels are plotted; the points used for interpolation are shown in black, while unused future and past points are shown in gray.\;
        \textbf{(b)} For second- and third-order schemes, the curves show how the interpolation coefficients vary with the shape parameter~$\gamma$. As $\log\gamma \to -\infty$, all coefficients converge to the same value, whereas as $\log\gamma \to \infty$, they approach the coefficients of the Adams method.}
    \label{fig:two-panel}
\end{figure}

\paragraph{RBF Interpolation}
Given the \(p\) most recent model evaluations 
$\{\bm{\hat x}_\theta(\bm{\hat x}_{\lambda_{t_{i-j}}},\lambda_{t_{i-j}})\}_{j=0}^{p-1}$,
we approximate \(\hat{\bm x}_\theta\in\mathbb{R}^D\) with the Gaussian-RBF interpolant
\begin{equation}
\label{eq:RBF_interpolation}
\begin{alignedat}{2}
\mathcal{R}(\lambda\,;\bm{\lambda}_{t_{i}}^{(p)})
&:=
\sum_{j=0}^{p-1}\bm{w}_{{t_{i}}, j}\,\phi_{{t_{i}}, j}(\lambda)
\;+\;
\bm{w}_{t_{i}}^\text{const} \cdot 1
,
\quad
&
\phi_{{t_{i}}, j}(\lambda)
&\coloneqq
\exp\Biggl[
-\biggl(\frac{\lambda - \lambda_{t_{i-j}}}{\,\gamma_{t_{i}}h_{t_i}\,}\biggr)^{2}
\Biggr],
\end{alignedat}
\end{equation}
where $\bm{\lambda}_{t_{i}}^{(p)}
\;:=\;
\bigl(\lambda_{t_i},\lambda_{t_{i-1}},\dots,\lambda_{t_{i-p+1}}\bigr)^{\!\top}\in\mathbb{R}^{p}$ and \(\phi_{{t_{i}}, j}\colon\mathbb{R}\to\mathbb{R}\) is a Gaussian radial basis function.
$\bm{w}_{{t_{i}}, j} \in \mathbb{R}^D$ denotes the weight for the $j$-th RBF, and $\bm{w}_{t_{i}}^{\text{const}} \in \mathbb{R}^D$ is the weight for the constant basis $1$.
The justification for introducing the constant basis is provided in Section \ref{sec:Coefficients-Summation-Condition}.
Define the interval length \(h_{t_{i}} \coloneqq \lambda_{t_{i+1}} - \lambda_{t_i}\). Each Gaussian’s width is given by \(\gamma_{t_{i}}h_{t_i}\), where \(\gamma_{{t_{i}}} \in \mathbb{R}^{+}\) is a learnable shape parameter.
Section~\ref{sec:Learning of Shape Parameter} describes how this parameter is optimized.


To obtain the weights in Eq.~\eqref{eq:RBF_interpolation},
we impose the constraint
$\sum_{j=0}^{p-1}\bm{w}_{t_{i}, j}= \bm{0}$ and solve the $(p+1)\times(p+1)$ linear system
$\mathbf{\Phi}_{t_{i}}^{(p)}\,\bm{W}_{t_{i}}^{(p)} = \bm{X}_{t_{i}}^{(p)}$
where \(\bm{W}_{t_{i}}^{(p)}\in\mathbb{R}^{(p+1)\times D}\) is the weight vector, \(\mathbf{\Phi}_{t_{i}}^{(p)}\in\mathbb{R}^{(p+1)\times(p+1)}\) is the kernel matrix, and \(\bm{X}_{t_{i}}^{(p)}\in\mathbb{R}^{(p+1)\times D}\) is the evaluation vector defined as follows:
\[
\bm{W}_{t_{i}}^{(p)}
  :=\begin{pmatrix}
      \bm{w}_{t_{i}, 0} & \bm{w}_{{t_{i}}, 1} & \cdots & \bm{w}_{{t_{i}}, p-1} & \bm{w}_{t_{i}}^{\text{const}}
    \end{pmatrix}^{\!\top},
\]
\[
\mathbf{\Phi}_{t_{i}}^{(p)}:=
\begin{pmatrix}
  \phi_{t_{i}, 0}(\lambda_{t_i})      & \cdots & \phi_{t_{i}, p-1}(\lambda_{t_i})      & 1 \\
  \phi_{t_{i}, 0}(\lambda_{t_{i-1}})  & \cdots & \phi_{t_{i}, p-1}(\lambda_{t_{i-1}})   & 1 \\
  \vdots                        & \ddots & \vdots                             & \vdots \\
  \phi_{t_{i}, 0}(\lambda_{t_{i-p+1}})& \cdots & \phi_{t_{i}, p-1}(\lambda_{t_{i-p+1}}) & 1 \\
  1 & \cdots & 1 & 0
\end{pmatrix},
\qquad
\bm{X}_{t_{i}}^{(p)}:=
\begin{pmatrix}
  \bm{\hat x}_{\theta}(\bm{\hat x}_{\lambda_{t_i}},        \lambda_{t_i})\\
  \bm{\hat x}_{\theta}(\bm{\hat x}_{\lambda_{t_{i-1}}},    \lambda_{t_{i-1}})\\
  \vdots\\
  \bm{\hat x}_{\theta}(\bm{\hat x}_{\lambda_{t_{i-p+1}}},  \lambda_{t_{i-p+1}})\\
  \bm{0}
\end{pmatrix}.
\]

\paragraph{RBF Sampling Algorithm}
By approximating the model evaluation
$\hat{\bm x}_\theta(\hat{\bm x}_\lambda,\lambda)$ in
Eq.~\eqref{eq:exact_solution_x0} with the RBF interpolant
$\mathcal{R}(\lambda;\bm{\lambda}_{t_i}^{(p)})$ from
Eq.~\eqref{eq:RBF_interpolation}, we derive the RBF sampling equation that
advances the sample from $\tilde{\bm x}_{t_i}$ to
$\tilde{\bm x}_{t_{i+1}}$:

\begin{equation}
\label{eq:RBF approximation}
\tilde{\bm x}_{t_{i+1}} =
\frac{\sigma_{t_{i+1}}}{\sigma_{t_{i}}}
\tilde{\bm x}_{t_{i}}
+
\sigma_{t_{i+1}}
\int_{\lambda_{t_i}}^{\lambda_{t_{i+1}}}
e^{\lambda}
\left[
\sum_{j=0}^{p-1}
\bm{w}_{{t_{i}}, j}
\phi_{{t_{i}}, j}(\lambda) + \bm{w}_{t_{i}}^\text{const} \cdot 1
\right]
\mathrm{d}\lambda.
\end{equation}

Define integrals \(l_{{t_{i}}, j}\in\mathbb{R}\) and \(l_{t_{i}}^{\text{const}}\in\mathbb{R}\) by
\begin{equation}
\label{eq:RBF-Solver Integral}
l_{{t_{i}}, j}
:=
\int_{\lambda_{t_i}}^{\lambda_{t_{i+1}}}
e^{\lambda}\,
\phi_{{t_{i}}, j}(\lambda)\,\mathrm{d}\lambda,
\qquad
l_{t_{i}}^{\text{const}}
:=
\int_{\lambda_{t_i}}^{\lambda_{t_{i+1}}}
e^{\lambda}\,\mathrm{d}\lambda.
\end{equation}
A numerically stable closed-form expression for \(l_{{t_{i}}, j}\) is derived in
Appendix~\ref{sec:how to calc integral}.
Let
\(
\bm{l}_{t_{i}}^{(p)}
\coloneqq
(l_{{t_{i}}, 0},l_{{t_{i}}, 1},\dots,l_{{t_{i}}, p-1},l_{t_{i}}^{\text{const}})^{\top}\in\mathbb{R}^{p+1}.
\)
Using this integral vector,
the RBF sampling equation becomes
\begin{equation*}
\tilde{\bm x}_{t_{i+1}}
=
\frac{\sigma_{t_{i+1}}}{\sigma_{t_i}}\,
\tilde{\bm x}_{t_i}
\;+\;
\sigma_{t_{i+1}}
(\bm{l}_{t_{i}}^{(p)})^{\top}\bm{W}_{t_{i}}^{(p)}.
\end{equation*}
Since \(\bm{W}_{t_{i}}^{(p)} = (\mathbf{\Phi}_{t_{i}}^{(p)})^{-1}\bm{X}_{t_{i}}^{(p)}\),
the RBF sampling equation can be written in terms of the evaluation vector $\bm X_{t_{i}}^{(p)}$
and its coefficient vector $\bm c_{t_{i}}^{(p)}\in\mathbb{R}^{p+1}$:
\begin{equation}
\label{eq:RBF-Solver Equation}
\tilde{\bm x}_{t_{i+1}}
=
\frac{\sigma_{t_{i+1}}}{\sigma_{t_i}}\,
\tilde{\bm x}_{t_i}
\;+\;
\sigma_{t_{i+1}}
(\bm c_{t_{i}}^{(p)})^{\top}\bm X_{t_{i}}^{(p)},
\qquad
(\bm c_{t_{i}}^{(p)})^{\top}
=
(\bm l_{t_{i}}^{(p)})^{\top}(\bm\Phi_{t_{i}}^{(p)})^{-1}.
\end{equation}
Based on the sampling equation, we construct a predictor–corrector pair via the linear multistep method~\citep{butcher2016numerical}.
The predictor uses the previous \(p\) evaluations up to \(t_i\),
whereas the corrector uses the previous \(p+1\) evaluations up to \(t_{i+1}\).
Algorithm \ref{alg:predictor-corrector} presents the complete RBF sampling procedure within the conventional multistep sampling framework~\cite{butcher2016numerical}.
\begin{equation}
\label{eq:RBF-Solver PC}
\left\{
\begin{aligned}
\tilde{\bm x}_{t_{i+1}}^{\text{pred}}
&=
\text{Predictor}\bigl(\tilde{\bm x}_{t_i}^{\text{corr}}\bigr)
=
\frac{\sigma_{t_{i+1}}}{\sigma_{t_i}}
\tilde{\bm x}_{t_i}^{\text{corr}}
+
\sigma_{t_{i+1}}
(\bm c_{t_{i}}^{(p)})^{\top}\bm X_{t_{i}}^{(p)},\\[6pt]
\tilde{\bm x}_{t_{i+1}}^{\text{corr}}
&=
\text{Corrector}\bigl(\tilde{\bm x}_{t_i}^{\text{corr}}\bigr)
=
\frac{\sigma_{t_{i+1}}}{\sigma_{t_i}}
\tilde{\bm x}_{t_i}^{\text{corr}}
+
\sigma_{t_{i+1}}
(\bm c_{t_{i+1}}^{(p+1)})^{\top}\bm X_{t_{i+1}}^{(p+1)}.
\end{aligned}
\right.
\end{equation}

\subsection{Coefficient Summation Condition}
\label{sec:Coefficients-Summation-Condition}


We analyze the case in which 
${\bm{\hat x}}_{\theta}$ is \emph{constant} on an interval $[\lambda_{t_i},\lambda_{t_{i+1}}]$
—for example, when  
(i)~the target data is already fixed or  
(ii)~a first-order Euler approximation is used.  
Under this assumption we have
$
\int_{\lambda_{t_i}}^{\lambda_{t_{i+1}}}
e^{\lambda}\,
{\bm{\hat x}}_{\theta}\!\bigl({\bm{\hat x}}_{\lambda},\lambda\bigr)\,
\mathrm{d}\lambda
=
{\bm{\hat x}}_{\theta}\!\bigl({\bm{\hat x}}_{\lambda},\lambda\bigr)
\int_{\lambda_{t_i}}^{\lambda_{t_{i+1}}}
e^{\lambda}\,\mathrm{d}\lambda$. To remain consistent with this special case, the coefficient vector
$\bm{c}_{t_{i}}^{(p)}$ must satisfy the following summation condition. ($[\,\cdot\,]_j$ denotes the \(j\)-th component.)
\begin{equation} \label{eq:Summation Condition}
\sum_{j=0}^{p-1}[\bm{c}_{t_{i}}^{(p)}]_{j}
=
\int_{\lambda_{t_i}}^{\lambda_{t_{i+1}}}
e^{\lambda}\,\mathrm{d}\lambda.
\end{equation}
The coefficient vector $\bm{c}_{t_{i}}^{(p)}$
of the proposed RBF sampling equation
in Eq.~\eqref{eq:RBF-Solver Equation} satisfies
$\bm{\Phi}_{t_{i}}^{(p)}\bm{c}_{t_{i}}^{(p)}=\bm{l}_{t_{i}}^{(p)}$,
and consequently,
\begin{equation}
\label{eq:summation-condition}
\sum_{j=0}^{p-1}[\bm{c}_{t_{i}}^{(p)}]_{j}
=
\bigl[\bm{\Phi}_{t_{i}}^{(p)}\bm{c}_{t_{i}}^{(p)}\bigr]_{p}
=
[\bm{l}_{t_{i}}^{(p)}]_{p}
=
l_{t_{i}}^{\mathrm{const}}
=
\int_{\lambda_{t_i}}^{\lambda_{t_{i+1}}}
e^{\lambda}\,\mathrm{d}\lambda .
\end{equation}
It is obvious that RBF-Solver in the case of $p=1$ coincides with the standard Euler method (DDIM~\cite{song2020denoising}).
The corrector step of RBF-Solver in Eq.~\eqref{eq:RBF-Solver PC} also satisfies the summation condition. The predictor–corrector pair from the Adams method with Lagrange interpolation also satisfies this summation condition; see Appendix~\ref{sec:Coefficients Summation Condition in Lagrange-Solver} for details.
\subsection{Adams Method Convergence}
\label{sec:Lagrange Convergence}

As established in previous studies on RBF interpolation  \cite{Lee2013rbf,Fornberg2002FlatRBF,lee2015flat}, the RBF interpolation converges to the Lagrange interpolation as the basis functions become increasingly flat.
\begin{lemma} {\rm(\cite{Lee2013rbf})} 
\label{lemma:RBF convergence}
Let the RBF be defined by Eq.~(\ref {eq:RBF_interpolation}) with a shape parameter $\gamma_{t_{i}}$.
Then the RBF interpolation in Eq.~(\ref{eq:RBF_interpolation}) converges to the Lagrange interpolation in Eq.~(\ref{eq:lagrange-interpolation})
as $\gamma_{t_{i}} \to \infty$:
\begin{equation*}
\lim_{\gamma_{t_{i}} \to \infty}
\mathcal{R}(\lambda\,;\bm{\lambda}_{t_{i}}^{(p)}) 
= \mathcal{L}^{(p)}(\lambda).   
\end{equation*}  
\end{lemma}
Consequently, the RBF sampling equation Eq.~\eqref{eq:RBF approximation} defined by the RBF interpolation
converges to the Lagrange
interpolation-based sampling method; that is,
Adams method 
in Appendix \ref{sec:Lagrange-Solver}. Figure~\ref{fig:coefficient_trends} illustrates specific examples of the coefficients for cases $p=2$ and $p=3$, demonstrating their convergence.

\subsection{Equal-Coefficient Sampling Convergence}
\label{sec:Euler_Convergence}
In this section, we show that the RBF-Solver sampling equation (Eq.~\eqref{eq:RBF-Solver Equation}) converges to the equal-coefficient sampler in the limit \(\gamma_{t_i}\to0\). A detailed proof is provided in Appendix~\ref{sec:Euler_Converence_in_Appendix}.
For \(\phi_{{t_{i}}, j}(\lambda) \) in  Eq.~\eqref{eq:RBF_interpolation} and 
$l_{{t_{i}}, j}$ in  Eq.~\eqref{eq:RBF-Solver Integral},
the pointwise limit as \( \gamma_{t_{i}} \to 0 \) is given by, for each \( j=0, \dots, p-1\),
\[ \lim_{\gamma_{t_{i}} \to 0} \phi_{{t_{i}}, j}(\lambda) =
\begin{cases}
1, & \text{if } \lambda = \lambda_{t_{i-j}} \\
0, & \text{if } \lambda \ne \lambda_{t_{i-j}}
\end{cases}
,\qquad
\lim_{\gamma_{t_{i}} \to 0}l_{{t_{i}}, j} = 0.
\]
Since the coefficient vector \( \bm c_{t_{i}}^{(p)} \) satisfies the linear system
\( \bm {\Phi}^{(p)}_{\text{limit}} \bm c_{t_{i}}^{(p)} = \bm l^{(p)}_{\text{limit}} \) in Eq.~\eqref{eq:RBF-Solver Equation} at the limit, 
hence the first $p$ components of $\bm c_{t_{i}}^{(p)}$ are all equal to $\l_{t_{i}}^{\text{const}}/p$.
In that case,
the RBF sampling equation 
reduces to the sampling method with equal coefficients:
\begin{equation*}
\label{eq:RBF-Solver Equation of Euler case}
\tilde{\bm x}_{t_{i+1}} =
\frac{\sigma_{t_{i+1}}}{\sigma_{t_{i}}}
\tilde{\bm x}_{t_{i}}
+
\sigma_{t_{i+1}}
\frac{l_{t_{i}}^{\text{const}}}{p}
\sum_{j=0}^{p-1} \bm{\hat x}_\theta(\bm{\hat x}_{\lambda_{t_{i-j}}},\lambda_{t_{i-j}}).
\end{equation*}

\subsection{Accuracy of RBF-Solver}
\label{sec:Accuracy}
Given the approximation accuracy $p$ of the RBF interpolation (Appendix~\ref{eq:RBF-Accuracy}), we extend this result to demonstrate that the RBF sampling also achieves global accuracy of order $p$.
\begin{theorem} \label{theorem:accuracy}
Suppose that the data prediction model \( \bm{\hat x}_\theta(\bm{\hat x}_\lambda,\lambda) \) is Lipschitz continuous with respect to $\hat{x}_\lambda$ with Lipschitz constant $L$,
and let \( h_{t_{i}} \coloneqq \lambda_{t_{i+1}} - \lambda_{t_i} \) for each 
\( i=0, \dots, M-1\) and
\(  h \coloneqq \max_{t_{i}} h_{t_{i}}. \)
Then, for each $i=p-1,\dots,M-1$, the RBF sampling method in Eq.~(\ref{eq:RBF-Solver Equation})
has local accuracy of order $p+1$ with respect to the exact solution \(\bm{x}_{t_{i+1}}\) in Eq.~(\ref{eq:exact_solution_x0}) 
and RBF-Solver with $p$ evaluations is of \( p\)-th order of accuracy, that is,
\begin{equation*}
|\bm{x}_0 - \bm{\tilde x}_{t_{M}}| = \mathcal{O}(h^{p}).    
\end{equation*}
\end{theorem}

\subsection{Gaussian Shape Parameter Optimization}
\label{sec:Learning of Shape Parameter}

\begin{figure}[t]
  \centering
  \begin{minipage}[t]{0.49\linewidth}
    \centering
    \begin{algorithm}[H]
      \caption{Predictor-Corrector Sampling}
      \label{alg:predictor-corrector}
      \begin{algorithmic}[1]
        \REQUIRE Data prediction model $\xv_{\theta}$, timesteps $\{t_i\}_{i=0}^{M}$, list $L$, initial noise ${\xv}_{t_0}$\vspace{4.0pt}
        \STATE $\tilde{\bm x}_{t_{0}}^{\mathrm{corr}} \leftarrow {\bm{x}}_{t_0}$
        \STATE Evaluate $\xv_{\theta}\!\bigl(\tilde{\bm x}_{t_0}^{\mathrm{corr}}, t_0\bigr)$ and add to $L$
        \FOR{$i \leftarrow 0$ to $M-1$}
          \STATE $\tilde{\bm x}_{t_{i+1}}^{\mathrm{pred}} \leftarrow
                 \text{Predictor}\!\bigl(\tilde{\bm x}_{t_i}^{\mathrm{corr}}\bigr)$
          \STATE \textbf{if} $i \ge M-1$ \textbf{ then break}
          \STATE Evaluate $\xv_{\theta}\!\bigl(\tilde{\bm x}_{t_{i+1}}^{\mathrm{pred}}, t_{i+1}\bigr)$ and add to $L$
          \rule{0pt}{0pt}\vspace{28.2pt}
        
          \STATE $\tilde{\bm x}_{t_{i+1}}^{\mathrm{corr}} \leftarrow
                 \text{Corrector}\!\bigl(\tilde{\bm x}_{t_i}^{\mathrm{corr}}\bigr)$
        \ENDFOR
        \RETURN $\tilde{\bm{x}}_{t_M}^{\mathrm{pred}}$
      \end{algorithmic}
    \end{algorithm}
  \end{minipage}
  \hfill
  \begin{minipage}[t]{0.49\linewidth}
    \centering
    \begin{algorithm}[H]
      \caption{Shape Parameter Optimization}
      \label{alg:pc-shape-parameter-learning}
      \begin{algorithmic}[1]
        \REQUIRE Data prediction model $\bm x_{\theta}$, timesteps $\{t_i\}_{i=0}^{M}$, list $L$, target $({\bm x}_{0}^{\mathrm{target}}, {\bm x}_{T}^{\mathrm{target}})$
        \STATE $\tilde{\bm x}_{t_{0}}^{\mathrm{corr}} \leftarrow {\bm x}_{T}^{\mathrm{target}}$
        \STATE Evaluate $\bm{x}_{\theta}\!\bigl(\tilde{\bm x}_{t_0}^{\mathrm{corr}}, t_0\bigr)$ and add to $L$
        \FOR{$i \leftarrow 0$ \textbf{to} $M-1$}
          \STATE $\tilde{\bm x}_{t_{i+1}}^{\mathrm{pred}} \leftarrow
                 \text{Predictor}\!\bigl(\tilde{\bm x}_{t_i}^{\mathrm{corr}};\,\gamma_{t_{i}}^{\mathrm{pred}}\bigr)$
          \STATE \textbf{if} $i \ge M-1$ \textbf{ then break}
          \STATE Evaluate $\bm{x}_{\theta}\!\bigl(\tilde{\bm x}_{t_{i+1}}^{\mathrm{pred}}, t_{i+1}\bigr)$ and add to $L$
          \STATE \parbox[t]{\dimexpr\linewidth-\algorithmicindent}{
        $\gamma_{t_{i+1}}^{\mathrm{pred}},\,\gamma_{t_{i}}^{\mathrm{corr}}
          \;\leftarrow$\\[-2pt]  
        \hspace*{3.5em}  
        $\displaystyle
          \argmin\,\bigl\|\tilde{\bm x}_{t_{i+2}}^{\mathrm{pred}}
          -\bm{x}_{t_{i+2}}^{\mathrm{target}}\bigr\|^{2}$}
          \STATE $\tilde{\bm x}_{t_{i+1}}^{\mathrm{corr}} \leftarrow
                 \text{Corrector}\!\bigl(\tilde{\bm x}_{t_i}^{\mathrm{corr}};\,\gamma_{t_{i}}^{\mathrm{corr}}\bigr)$
        \ENDFOR
        \RETURN $\tilde{\bm x}_{t_M}^{\mathrm{pred}}$
      \end{algorithmic}
    \end{algorithm}
  \end{minipage}
\end{figure}

This section describes how to optimize the Gaussian shape parameters.
We first construct a set of target image–noise pairs, $\bigl(\bm x_{0}^{\mathrm{target}}, \bm x_{T}^{\mathrm{target}}\bigr)$, by running an existing sampler with a large NFE (e.g., UniPC~\cite{zhao2023unipc} with 200 steps).
We use 128 target pairs in each experiment. At runtime, the intermediate target at step \(i\) is computed via the forward diffusion process
$
\bm x_{t_i}^{\mathrm{target}}
=
\alpha_{t_i}\,\bm x_{0}^{\mathrm{target}}
+
\sigma_{t_i}\,\bm x_{T}^{\mathrm{target}}
$ (see Appendix~\ref{app:intermediate_target_ablation}).

At step~0, the predictor reduces to the first-order Euler method, so $\gamma_{0}^{\mathrm{pred}}$ is not optimized.  
Similarly, because no corrector is applied at step $M\!-\!1$, $\gamma_{t_{M-1}}^{\mathrm{corr}}$ is also excluded from learning.
Following Eq.~\eqref{eq:RBF-Solver PC} and
Algorithm~\ref{alg:pc-shape-parameter-learning},
$\tilde{\bm x}_{t_{i+2}}^{\mathrm{pred}}$ is obtained from
$\tilde{\bm x}_{t_i}^{\mathrm{corr}}$ by applying the corrector at
step $i$ ($i>0$) and then the predictor at step $i+1$.
During this procedure, the shape parameters $\gamma_{t_i}^{\mathrm{corr}}$ and $\gamma_{t_{i+1}}^{\mathrm{pred}}$ are used, respectively, to determine the coefficient vectors $\bm c_{t_{i+1}}^{(p+1)}$ and $\bm c_{t_{i+1}}^{(p)}$, allowing the predicted value $\tilde{\bm x}_{t_{i+2}}^{\mathrm{pred}}$ to be expressed as
\begin{equation}
\label{eq:(i+1)-step predictor}
\begin{aligned}
\tilde{\bm x}_{t_{i+2}}^{\mathrm{pred}}
(\tilde{\bm x}_{t_{i}}^{\mathrm{corr}};\,\gamma_{t_{i}}^{\mathrm{corr}},\,\gamma_{t_{i+1}}^{\mathrm{pred}})
&=
\text{Predictor}
\bigl(
\text{Corrector}
\left(
\tilde{\bm x}_{t_{i}}^{\mathrm{corr}};\,\gamma_{t_{i}}^{\mathrm{corr}}
\right)
;\, \gamma_{t_{i+1}}^{\mathrm{pred}}
\bigr)
\\[6pt]
&=
\frac{\sigma_{t_{i+2}}}{\sigma_{t_{i}}}
\tilde{\bm x}_{t_{i}}^{\mathrm{corr}}
\;+\;
\sigma_{t_{i+2}}\,
\left(
\big({\bm c}_{t_{i+1}}^{(p+1)}\big)^{\!\top}\,{\bm X}_{t_{i+1}}^{(p+1)}
\;+\;
\big({\bm c}_{t_{i+1}}^{(p)}\big)^{\!\top}\,{\bm X}_{t_{i+1}}^{(p)}
\right).
\end{aligned}
\end{equation}
The shape parameters $\gamma_{t_{i+1}}^{\mathrm{pred}}$ and $\gamma_{t_{i}}^{\mathrm{corr}}$ are optimized by minimizing the mean-squared error between the prediction $\tilde{\bm x}_{t_{i+2}}^{\mathrm{pred}}$, computed by Eq.~\eqref{eq:(i+1)-step predictor}, and the target $\bm x_{t_{i+2}}^{\mathrm{target}}$.
Algorithm~\ref{alg:pc-shape-parameter-learning} explains how this loss is incorporated into the sampling procedure.
\begin{equation}
\label{eq:shape-learning}
\gamma_{t_{i+1}}^{\mathrm{pred}},\,\gamma_{t_{i}}^{\mathrm{corr}}
=
\argmin_{\gamma^{\mathrm{pred}},\,\gamma^{\mathrm{corr}}}\,
   \Bigl\|
     \tilde{\bm x}_{t_{i+2}}^{\mathrm{pred}}
       (\tilde{\bm x}_{t_{i}}^{\mathrm{corr}};\,\gamma^{\mathrm{pred}},\,\gamma^{\mathrm{corr}})
     - \bm x_{t_{i+2}}^{\mathrm{target}}
   \Bigr\|^2,
\quad 0 \le i < M-1.
\end{equation}
Because computing this loss requires no additional network evaluations, the parameter space can be explored efficiently via grid search. The runtime of this optimization is detailed in Appendix~\ref{app:elapsed_time}.
Empirically, when $\log \gamma_{t_{i}} \ge 2$ (see Appendix~\ref{app:shape_parameter_range}), the optimum tends to diverge toward infinity; in that case we switch to the Adams method (see Appendix \ref{sec:Lagrange-Solver}).
The influence of the shape parameter on the sampling process is examined in Appendix \ref{app:Analysis of Shape Parameter}.

\section{Experiments}
\label{sec:Experiments}
In this section, we benchmark the proposed RBF-Solver against two state-of-the-art training-free samplers, DPM-Solver++~\cite{lu2022dpmpp} and UniPC~\cite{zhao2023unipc}.  
As discussed in Section~\ref{sec:Multistep Methods for Fast Sampling of DPMs}, DPM-Solver++ relies on a Taylor series approximation, whereas UniPC can be viewed as a variant of the Adams method that employs Lagrange interpolation. In addition, \cite{amed_solver, dc_solver, dpm_solver_v3} also introduce samplers that, like RBF-Solver, leverage auxiliary parameters to enhance sampling quality. We refer to this class of methods as \emph{auxiliary-tuning samplers}.  
Experimental comparisons with these samplers are presented in Appendix~\ref{app:Experiment Results for Auxiliary-Tuning Samplers}.

\subsection{Unconditional Sampling Results}
\label{sec:uncond_results}

\begin{figure*}[t]
  \centering
  \begin{subfigure}[t]{0.31\textwidth}
    \centering
    \includegraphics[width=\linewidth]{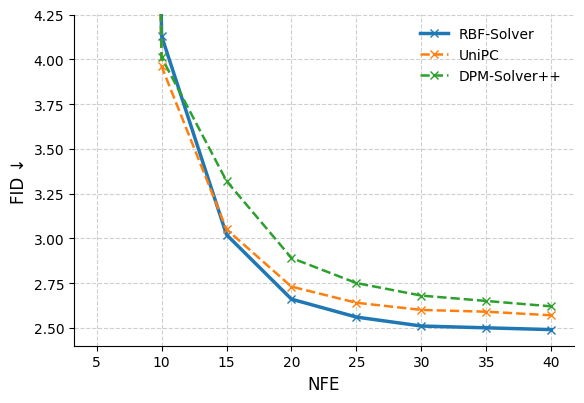}%
    \caption{CIFAR-10, Score-SDE (see Table~\ref{tab:fid_cifar10_score_sde})}
    \label{fig:uncond_cifar10_scoresde}
  \end{subfigure}
  \hfill
  \begin{subfigure}[t]{0.31\textwidth}
    \centering
    \includegraphics[width=\linewidth]{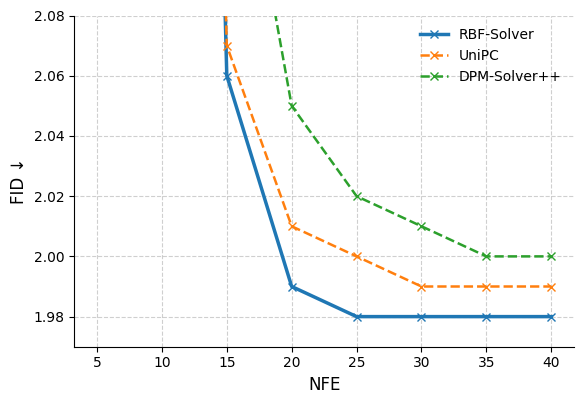}%
    \caption{CIFAR-10, EDM (see Table~\ref{tab:fid_cifar10_edm})}
    \label{fig:uncond_cifar10_edm}
  \end{subfigure}
  \hfill
  \begin{subfigure}[t]{0.31\textwidth}
    \centering
    \includegraphics[width=\linewidth]{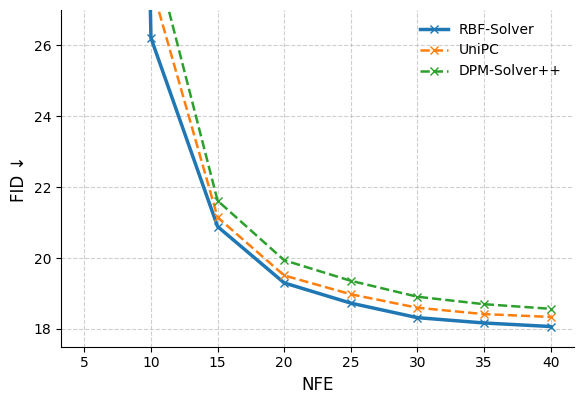}%
    \caption{Imagenet 64$\times$64, Improved-Diffusion (see Table~\ref{tab:fid_imagenet64_improved})}
    \label{fig:uncond_imagenet64}
  \end{subfigure}

  \caption{Unconditional sampling results. FID$\downarrow$ on 50k samples.}
  \label{fig:unconditional_results}
\end{figure*}

We evaluate the samplers on unconditional diffusion models.
For each setting, the training dataset, backbone model, and number of generated samples are as follows:
\begin{itemize}\setlength{\itemsep}{2pt}
  \item CIFAR-10~\cite{Krizhevsky09learningmultiple}, Score-SDE~\cite{song2021score}, 50k samples
  \item CIFAR-10~\cite{Krizhevsky09learningmultiple}, EDM~\cite{karras2022elucidating}, 50k samples
  \item ImageNet $64\times64$~\cite{chrabaszcz2017downsampled}, Improved Diffusion~\cite{nichol2021improved}, 50k samples
\end{itemize}

Performance is measured with the Fréchet Inception Distance (FID)~\cite{heusel2017gans}.  
Figure~\ref{fig:unconditional_results} summarizes the results. Although RBF-Solver shows better performance with order~4 than with order~3  
(see Appendix~\ref{app:Additional Experiment Results - Unconditional Models}),  
we set order~3 for all three samplers to ensure a fair comparison.
On the Score-SDE and EDM backbones, RBF-Solver produces higher (worse) FID scores than the competing samplers for every tested setting with $\text{NFE}\le 10$, 
but for every tested value of $\text{NFE}$ from 15 to 40, it consistently achieves lower (better) FID scores.  
For ImageNet $64\times64$ (Figure~\ref{fig:uncond_imagenet64}), RBF-Solver surpasses the other samplers at every evaluated setting with $\text{NFE}\ge 10$.  
Complete numerical results, including per-NFE FID tables, are provided in Appendix~\ref{app:Additional Experiment Results - Unconditional Models}.

\subsection{Conditional Sampling Results}
\label{sec:cond_results}
We next compare conditional sampling performance under both classifier guidance and classifier-free guidance (CFG).  
For each setting, the dataset, backbone model, sample count, guidance type, and guidance scale are as follows:
\begin{itemize}\setlength{\itemsep}{2pt}
  \item ImageNet $128\times128$~\cite{deng2009imagenet}, Guided-Diffusion~\cite{dhariwal2021diffusion}, 50k samples, Classifier guidance, Guidance scale = 6.0
  \item ImageNet $256\times256$~\cite{deng2009imagenet}, Guided-Diffusion~\cite{dhariwal2021diffusion}, 10k samples, CFG, Guidance scale = 8.0
  \item LAION-5B~\cite{schuhmann2022laion}, Stable-Diffusion v1.4~\cite{rombach2022high}, 10k samples, CFG, Guidance scale = 7.5
\end{itemize}
For the ImageNet tasks we report FID, while for Stable-Diffusion we measure image-to-image cosine similarity against samples generated by DPM-Solver++ with $\text{NFE}=200$; image embeddings are obtained with CLIP ViT-B/32~\cite{radford2021learning}.  
Results are summarized in Figure \ref{fig:conditional_results}. Although RBF-Solver performs better with order~3 than with order~2 (see Appendix~\ref{app:Additional Experiment Results - Conditional Models}), we fix order 3 for all samplers to ensure a fair comparison.
Across ImageNet 128$\times$128, ImageNet 256$\times$256, and Stable-Diffusion, RBF-Solver consistently outperforms the other samplers.  
The performance gap is more pronounced at higher guidance scales than at lower ones. Experimental details and exact numerical results are provided in Appendix~\ref{app:Additional Experiment Results - Conditional Models}.

\begin{figure*}[t]
  \centering
  \begin{subfigure}[t]{0.31\textwidth}
    \centering
    \includegraphics[width=\linewidth]{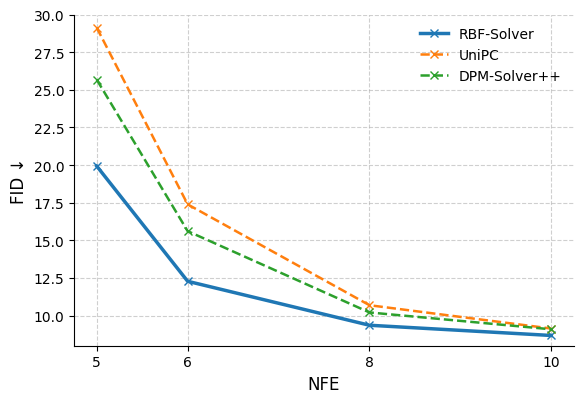}%
    \caption{ImageNet 128$\times$128, Guided-Diffusion, Guidance scale $=6.0$ (see Table~\ref{tab:fid_main_imagenet128})}
    \label{fig:cond_imagenet128}
  \end{subfigure}
  \hfill
  \begin{subfigure}[t]{0.31\textwidth}
    \centering
    \includegraphics[width=\linewidth]{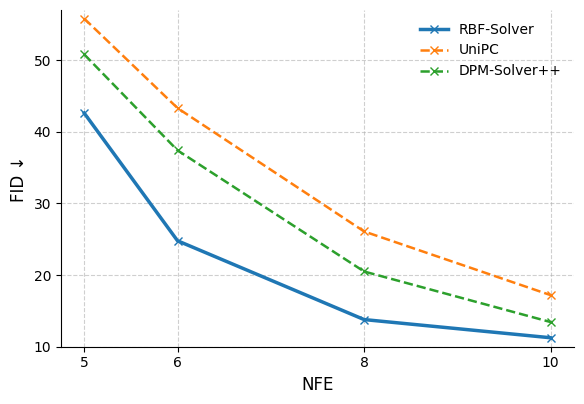}%
    \caption{ImageNet 256$\times$256, Guided-Diffusion, Guidance scale $=8.0$ (see Table~\ref{tab:fid_main_imagenet256})}
    \label{fig:cond_imagenet256}
  \end{subfigure}
  \hfill
  \begin{subfigure}[t]{0.31\textwidth}
    \centering
    \includegraphics[width=\linewidth]{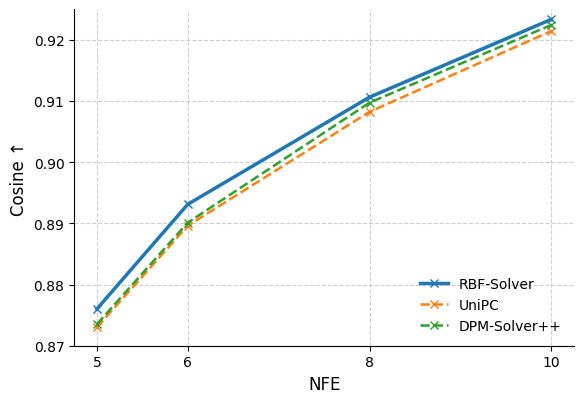}%
    \caption{Stable-Diffusion v1.4, Guidance scale $=7.5$ (see Table~\ref{tab:main_stable})}
    \label{fig:cond_sd14_clip}
  \end{subfigure}

  \caption{Conditional sampling results. ImageNet-128: FID$\downarrow$ on 50k samples; ImageNet-256 \& Stable-Diffusion v1.4: FID$\downarrow$/CLIP$\uparrow$ on 10k.}

  \label{fig:conditional_results}
\end{figure*}

\subsection{Ablation Study}
\label{sec:rbf_ablation}
\paragraph{From Adams Method to RBF--Solver}
\label{sec:adams_to_rbf}
The primary baseline for RBF-Solver is the Adams method, which employs Lagrange interpolation.  
Table~\ref{tab:adams_rbf_ablation} details the step-by-step evolution from the Adams method to the full RBF-Solver.  
\textit{RBF-Solver w/o Constant} omits the constant basis term in Eq.~\eqref{eq:RBF_interpolation}; this variant performs marginally better than RBF-Solver for $\text{NFE}\le 15$ but degrades steadily for $\text{NFE}\ge 20$.  
\textit{Adams Pred. + RBF Corr.} combines the Adams predictor with the RBF corrector, whereas \textit{RBF Pred. + Adams Corr.} applies the RBF predictor followed by the Adams corrector.  
Using the RBF formulation in the corrector yields a greater benefit than restricting it to the predictor, and for all tested settings with $\text{NFE}\ge 20$ the full RBF-Solver matches or surpasses the other configurations.


\begin{table}[h!]
    \centering
    \setlength{\tabcolsep}{6.5pt}     
    \renewcommand{\arraystretch}{1.15}
    \caption{From Adams method to RBF-Solver. Sample quality measured by FID $\downarrow$ on ImageNet 64$\times$64 (Improved-Diffusion), evaluated on 50k samples.}

    \label{tab:adams_rbf_ablation}
    \begin{tabularx}{\textwidth}{lYYYYYYYY}
        \toprule
        & \multicolumn{8}{c}{\textbf{NFE}}\\
        \cmidrule(lr){2-9}
        & \textbf{5} & \textbf{10} & \textbf{15} & \textbf{20}
          & \textbf{25} & \textbf{30} & \textbf{35} & \textbf{40}\\
        \midrule
        Adams Method
            & \textbf{110.02} & 26.94 & 21.07 & 19.49
            & 18.97 & 18.59 & 18.41 & 18.33\\
        RBF-Solver w/o Constant
            & 112.54 & 26.43 & \textbf{20.85} & 19.54
            & 18.95 & 18.51 & 18.31 & 18.16\\    
        Adams Pred.~+ RBF Corr.
            & 111.29 & 26.81 & 20.93 & 19.32
            & 18.73 & 18.32 & \textbf{18.17} & \textbf{18.07}\\
        RBF Pred.~+ Adams Corr.
            & 117.63 & \textbf{26.30} & 21.00 & 19.45
            & 18.94 & 18.58 & 18.40 & 18.33\\
        \rowcolor{gray!20}
        RBF-Solver
            & 116.55 & 26.64 & 20.89 & \textbf{19.30}
            & \textbf{18.72} & \textbf{18.31} & \textbf{18.17} & \textbf{18.07}\\
        \bottomrule
    \end{tabularx}
\end{table}
\paragraph{Shape Parameter Optimization}
\label{sec:shape_opt}
Table~\ref{tab:ablation_gamma} compares four strategies for optimizing the shape parameters
$\gamma_{t_{i+1}}^{\text{pred}}$ and $\gamma_{t_i}^{\text{corr}}$ introduced in
Section~\ref{sec:Learning of Shape Parameter}.  
In the \emph{Shared-$\gamma$} configuration, the predictor and corrector use the same value,
whereas \emph{Split-$\gamma$} allows the two values to differ.  
The \emph{Indep.} variant optimizes the parameters separately, while the \emph{Joint} variant
optimizes them simultaneously.  

For $\text{NFE}=5$ and $10$, the Shared-$\gamma$ setting provides a slight advantage,  
but for every tested value with $\text{NFE}\ge 15$, Split-$\gamma$ achieves lower (better) scores.  
Joint optimization consistently outperforms independent optimization, albeit by a small margin.


\begin{table}[h]
    \centering
    \setlength{\tabcolsep}{6.5pt}     
    \renewcommand{\arraystretch}{1.15}
    \caption{Shape parameter optimization. Sample quality measured by FID $\downarrow$ on ImageNet 64$\times$64 (Improved-Diffusion), evaluated on 50k samples.}
    \label{tab:ablation_gamma}
    \begin{tabularx}{\textwidth}{lYYYYYYYY}
        \toprule
        & \multicolumn{8}{c}{\textbf{NFE}}\\
        \cmidrule(lr){2-9}
        & \textbf{5} & \textbf{10} & \textbf{15} & \textbf{20}
          & \textbf{25} & \textbf{30} & \textbf{35} & \textbf{40}\\
        \midrule
        Shared-$\gamma$
            & \textbf{116.37} & \textbf{26.07} & 21.02 & 19.46
            & 18.95 & 18.58 & 18.40 & 18.28\\
        Split-$\gamma$ / Indep.
            & 117.35 & 26.86 & 20.91 & 19.33
            & 18.73 & 18.32 & \textbf{18.17} & 18.08\\
        \rowcolor{gray!20}
        Split-$\gamma$ / Joint
            & 116.55 & 26.64 & \textbf{20.89} & \textbf{19.30}
            & \textbf{18.72} & \textbf{18.31} & \textbf{18.17} & \textbf{18.07}\\
        \bottomrule
    \end{tabularx}
\end{table}
\subsection{High-Order Stability}
\label{sec:high_order_stability}
Table~\ref{tab:rbfsolver_unipc_ablation} reports sampling quality when RBF-Solver and UniPC are run at orders 3 and above.  
For UniPC, the best FID is obtained with \(p=3\) at \(\text{NFE}=10\); the optimum shifts to \(p=4\) when \(\text{NFE}=20\), 30, or 40 and performance deteriorates sharply once \(p\ge 5\).  
In contrast, RBF-Solver attains its minimum FID with \(p=4\) at \(\text{NFE}=10\) and with \(p=5\) or \(p=6\) when \(\text{NFE}=20\), 30, or 40; it remains stable even at higher orders.  
This robustness is attributed to two factors: (i) the learned shape parameters and (ii) the locality of Gaussian RBF bases, whose influence diminishes rapidly with distance from their centers~\cite{fornberg2007runge}.

\begin{table}[h]
    \centering
    \setlength{\tabcolsep}{3.5pt}
    \renewcommand{\arraystretch}{1.15}
    \caption{High-order stability. Sample quality measured by FID $\downarrow$ on the ImageNet 64$\times$64 (Improved-Diffusion), evaluated on 50k samples.}
    \label{tab:rbfsolver_unipc_ablation}
    \begin{tabular}{lcccccccc}
        \toprule
        & \multicolumn{2}{c}{\textbf{NFE = 10}} & \multicolumn{2}{c}{\textbf{NFE = 20}} &
          \multicolumn{2}{c}{\textbf{NFE = 30}} & \multicolumn{2}{c}{\textbf{NFE = 40}}\\
        \cmidrule(lr){2-3}\cmidrule(lr){4-5}\cmidrule(lr){6-7}\cmidrule(lr){8-9}
        & \textbf{RBF-Solver} & \textbf{UniPC} & \textbf{RBF-Solver} & \textbf{UniPC} &
          \textbf{RBF-Solver} & \textbf{UniPC} & \textbf{RBF-Solver} & \textbf{UniPC}\\
        \midrule
        $p\!=\!3$ & 26.47 & \textbf{27.75} & 19.30 & 19.51 & 18.31 & 18.60 & 18.07 & 18.34\\
        $p\!=\!4$ & \textbf{24.72} & 33.60 & 18.67 & \textbf{19.37} & 18.08 & \textbf{18.39} & 17.95 & \textbf{18.25}\\
        $p\!=\!5$ & 25.37 & 40.55 & 18.58 & 20.25 & \textbf{17.92} & 18.48 & 17.92 & 18.27\\
        $p\!=\!6$ & 25.30 & 40.55 & \textbf{18.44} & 29.93 & 17.99 & 23.89 & \textbf{17.79} & 20.69\\
        $p\!=\!7$ & 26.39 & 55.96 & 18.62 & 58.40 & 18.52 & 68.87 & 18.97 & 80.48\\
        $p\!=\!8$ & 29.26 & 106.09 & 18.68 & 87.22 & 18.08 & 123.89 & 18.16 & 146.54\\
        \bottomrule
    \end{tabular}
\end{table}
\section{Conclusions}
\label{sec:conclusions}
This paper proposes RBF-Solver, a sampling scheme for diffusion models based on radial basis function interpolation.
By satisfying the coefficient summation condition, the first-order case coincides with the Euler method and DDIM \cite{song2020denoising}.
The shape parameters that control the Gaussian widths cause RBF-Solver to converge to the Adams method as they tend to infinity.
The optimal shape parameters can be found by following the target trajectory.
RBF-Solver outperforms strong baselines in unconditional generation on CIFAR-10~\cite{Krizhevsky09learningmultiple} and ImageNet 64$\times$64 \cite{chrabaszcz2017downsampled} whenever NFE ≥ 15, and achieves better FID and cosine similarity on conditional ImageNet 128$\times$128, 256$\times$256~\cite{deng2009imagenet} and Stable Diffusion v1.4 \cite{rombach2022high} in the low-NFE regime.
Unlike polynomial-based samplers, which destabilize from fourth order onward, RBF-Solver remains robust at higher orders.

\paragraph{Limitations} This study investigates only the Gaussian RBF; other kernels—such as multiquadric and Wendland’s compactly supported functions~\cite{lee2015flat,Fornberg2004mq}—warrant further investigation. Moreover, adaptive runtime schemes that adjust the shape parameters on-the-fly, without an additional optimization procedure, remain unexplored.

\paragraph{Broader Impact} The proposed accelerated sampler can significantly lower the computational barrier for generating high-quality visual, auditory, and video content, fostering scientific and creative applications. Conversely, the same efficiency gains may facilitate malicious uses of synthetic media for disinformation or impersonation.



\clearpage
\bibliographystyle{unsrt}  
\bibliography{main}  

\clearpage
\appendix
\section{Proofs of Accuracy and Convergence of RBF-Solver} 
\label{eq:RBF-Accuracy}
In this section, our aim is to prove the accuracy of RBF-Solver.
First, we provide common assumptions in the convergence analysis of multistep approaches. Next, we summarize the existing proof of the accuracy of RBF interpolation (Appendix~\ref{sec:Accuracy RBF-interpolation}) and present a proof of RBF-Solver's accuracy in Theorem~\ref{theorem:accuracy} (Appendix~\ref{sec:proof_Thm_in_Appendix}).  Appendix~\ref{sec:Euler_Converence_in_Appendix} provides a detailed explanation of equal-coefficient sampling convergence of RBF-Solver in Section~\ref{sec:Euler_Convergence} and
Appendix~\ref{sec:RBF coefficients} introduces a reparameterization of RBF coefficients based on Lagrange interpolation convergence. 

{\bf Assumption A.1.} The data prediction model
$\bm{x}_\theta(\bm{x}_t, t)$
is Lipschitz continuous with respect to $\bm{x}$ with Lipschitz constant $L$.

{\bf Assumption A.2.}
\( h := \max_{0\leq i \leq M-1} h_{t_i} = \mathcal{O}(1/M).\)

{\bf Assumption A.3.} The starting values 
\( \bm{\tilde x}_{t_{i+1}}, \, i=0,\dots,p-2,\)
satisfy 
\[|\bm{\tilde x}_{t_{i+1}} - \bm{x}_{t_{i+1}}|= \mathcal{O}(h^{p+1}), \quad i=0,\dots,p-2. \]

\subsection{Accuracy of RBF-interpolation} \label{sec:Accuracy RBF-interpolation}
To provide a complete and self-contained proof, we recall the accuracy properties of RBF interpolation from the existing literature~\cite{lee2010nonlinear}. We summarize and adapt the results to the notation of this paper for consistency. 
\begin{lemma} \label{lemma:accuracy of RBF interpolation}
Suppose that the RBF interpolation
$\mathcal{R}(\lambda\,;\bm{\lambda}_{t_i}^{(p)})$ 
in Eq.~(\ref {eq:RBF_interpolation}) interpolates the values
$\{ \bm{\hat x}_\theta(\bm{\hat x}_{\lambda_{t_{i-j}}},\lambda_{t_{i-j}}) \}_{j=0}^{p-1}$ and  
its weights satisfy
$\sum_{j=0}^{p-1}\bm{w}_{t_{i}, j}=0$.
Then the RBF interpolation has local accuracy of order $p$
for \( \lambda \in (\lambda_{t_i},\lambda_{t_{i+1}} )\) as
\begin{equation*}   
\big| \bm{\hat x}_\theta(\bm{\hat x}_\lambda, \lambda) - \mathcal{R}(\lambda\,;\bm{\lambda}_{t_i}^{(p)}) \big| = \mathcal{O}(h^{p}).
\end{equation*}
\end{lemma} 
\begin{proof}
From the data interpolation property, the RBF interpolation 
$\mathcal{R}(\lambda\,;\bm{\lambda}_{t_i}^{(p)})$ 
can be formulated as a Lagrange-type
\begin{equation} \label{eq:RBF Lagrange-type}
\mathcal{R}(\lambda\,;\bm{\lambda}_{t_i}^{(p)}) = \sum_{j=0}^{p-1} u_{j}(\lambda) 
\bm{\hat x}_\theta(\bm{\hat x}_{\lambda_{t_{i-j}}}, \lambda_{t_{i-j}}) 
\end{equation} 
where 
\( u_{j}(\lambda_{t_{i-k}}) = \delta_{j,k}, \, j,k=0,\dots,p-1\).
This follows from the properties of RBF interpolation,
\begin{equation*}
\sum_{j=0}^{p-1}  u_{j}(\lambda)  \phi_{{t_{i}}, n}(\lambda_{t_j}) 
= \phi_{{t_{i}}, n}(\lambda), \quad n=0,\dots,p-1.  
\end{equation*} 
To prove the interpolation accuracy, we define an auxiliary function 
$g(\cdot) \in \text{span}\{\phi_{{t_{i}}, 0},\dots, \phi_{{t_{i}}, p-1} \}$
for a fixed \( \lambda \), 
as
\[ g(\cdot) := \sum_{n=0}^{p-1} d^{\lambda}_n \phi_{{t_{i}}, n}(\cdot) \] 
such that the $k-$th derivatives of $g(\cdot)$ are equal to those of 
$\bm{\hat x}_\theta$ at the fixed $\lambda$ for all $k=0,\dots,p-1$, that is,
\[ g^{(k)}(\lambda) = \bm{\hat x}_\theta^{(k)}(\bm{\hat x}_\lambda,\lambda), \quad k=0,\dots,p-1.\]
Then the function \( g(\cdot)\) satisfies the relation
\begin{align*}  
\sum_{j=0}^{p-1}  u_{j}(\lambda)g(\lambda_{t_{i-j}})
= \sum_{n=0}^{p-1} d^\lambda_n \sum_{j=0}^{p-1}  u_{j}(\lambda)  \phi_{{t_{i}}, n}(\lambda_{t_{i-j}}) 
= \sum_{n=0}^{p-1} d^\lambda_n 
\phi_{{t_{i}}, n}(\lambda) 
= g(\lambda)
= \bm{\hat x}_\theta(\bm{\hat x}_\lambda,\lambda).
\end{align*}
We now calculate the interpolation error using 
$\bm{\hat x}_\theta(\bm{\hat x}_\lambda, \lambda) =\sum_{j=0}^{p-1}  u_{j}(\lambda)g(\lambda_{t_{i-j}})$, as follows:
\begin{align*}
\big| \bm{\hat x}_\theta(\bm{\hat x}_\lambda, \lambda) - \mathcal{R}(\lambda\,;\bm{\lambda}_{t_i}^{(p)}) \big| 
&= \bigg| \sum_{j=0}^{p-1}  
u_{j}(\lambda) g(\lambda_{t_{i-j}})-\sum_{j=0}^{p-1}  
u_{j}(\lambda) \bm{\hat x}_\theta(\bm{\hat x}_{\lambda_{t_{i-j}}}, \lambda_{t_{i-j}})\bigg| \\
& =  \bigg| \sum_{j=0}^{p-1}  u_{j}(\lambda) 
\big(g(\lambda_{t_{i-j}})-\bm{\hat x}_\theta(\bm{\hat x}_{\lambda_{t_{i-j}}}, \lambda_{t_{i-j}})\big) \bigg| 
\end{align*}
By taking Taylor expansions of $g(\lambda_{t_{i-j}})$ and $\bm{\hat x}_\theta(\bm{\hat x}_{\lambda_{t_{i-j}}}, \lambda_{t_{i-j}})$ at $\lambda$ for each $j=0,\dots,p-1$
and applying 
$g^{(k)}(\lambda) = \bm{\hat x}_\theta^{(k)}(\bm{\hat x}_\lambda,\lambda), k=0,\dots,p-1$,
we conclude that
\begin{align*}
\big| \bm{\hat x}_\theta(\bm{\hat x}_\lambda, \lambda) - \mathcal{R}(\lambda\,;\bm{\lambda}_{t_i}^{(p)}) \big| 
& =  \bigg| \sum_{j=0}^{p-1}  u_{j}(\lambda) 
\frac{(\lambda_{t_{i-j}}-\lambda)^p}{p !}( g^{(p)}(\lambda)
- \bm{\hat x}^{(p)}_\theta(\bm{\hat x}_\lambda,\lambda))  \bigg|+ \mathcal{O}(h^{p+1}) \\
&= \bigg| \sum_{j=0}^{p-1}  u_{j}(\lambda) 
\frac{(\lambda_{t_{i-j}}-\lambda)^p}{p !}  \bigg| 
\bigg| g^{(p)}(\lambda)
- \bm{\hat x}^{(p)}_\theta(\bm{\hat x}_\lambda,\lambda) \bigg| + \mathcal{O}(h^{p+1}) \\
&= \mathcal{O}(h^p).
\end{align*}
\end{proof}

\subsection{Proof of accuracy of RBF-Solver}
\label{sec:proof_Thm_in_Appendix}
\begin{proof}
Since RBF interpolation in Eq.~(\ref{eq:RBF_interpolation}) is of order $p$, as established in
Lemma \ref{lemma:accuracy of RBF interpolation}, that is, 
\[ \big| \bm{\hat x}_\theta(\bm{\hat x}_\lambda, \lambda) - 
\mathcal{R}(\lambda\,;\bm{\lambda}_{t_{i}}^{(p)}) \big| = \mathcal{O}(h^{p}),\]
for each $i \geq p-1,$ 
we evaluate the error of the RBF sampling equation
in Eq.~(\ref{eq:RBF-Solver Equation}) as follows:
\begin{align*}
\bm{\tilde x}_{t_{i+1}} -\bm{x}_{t_{i+1}} 
& =  \sigma_{t_{i+1}}  \left( \int_{\lambda_{t_{i}}}^{\lambda_{t_{i+1}}} 
                e^{\lambda} 
                \bm{\hat x}_\theta(\bm{\hat x}_\lambda,\lambda)
                \mathrm{d}\lambda
            \,\, - \, (\bm c_{t_{i}}^{(p)})^{\top}\,\bm X_{t_{i}}^{(p)} \right) \\
& =  \sigma_{t_{i+1}} \int_{\lambda_{t_{i}}}^{\lambda_{t_{i+1}}}
    e^{\lambda} 
    \left( \bm{\hat x}_\theta(\bm{\hat x}_\lambda, \lambda) -     
    \mathcal{R}(\lambda\,;\bm{\lambda}_{t_{i}}^{(p)}) \right) \mathrm{d}\lambda\\
& = \mathcal{O}(h^{p+1}).
\end{align*}
For $i=0,\dots,p-2$, we adopt a common assumption in multistep approaches for the starting values 
\( \bm{\tilde x}_{t_{i+1}}\), namely,
\[|\bm{\tilde x}_{t_{i+1}} - \bm{x}_{t_{i+1}}|= \mathcal{O}(h^{p+1}). \]
Therefore, for each $i=0,\dots,M-1$, the RBF sampling method has local accuracy of order $p+1$, and
we conclude that RBF-Solver has the global accuracy of order $p$.
\end{proof}

\subsection{Equal-Coefficient Sampling Convergence}
\label{sec:Euler_Converence_in_Appendix}
In this section, we provide a detailed explanation showing that the sampling formula of RBF-Solver converges to an equal-coefficient sampling
as \( \gamma_{t_{i}} \to 0 \).
For \(\phi_{{t_{i}}, j}(\lambda) \) in  Eq.~\eqref{eq:RBF_interpolation},  
the pointwise limit as \( \gamma_{t_{i}} \to 0 \) is given for each \( j=0, \dots, p-1\)  by
\[ \lim_{\gamma_{t_{i}} \to 0} \phi_{{t_{i}}, j}(\lambda) =
\begin{cases}
1, & \text{if } \lambda = \lambda_{t_{i-j}} \\
0, & \text{if } \lambda \ne \lambda_{t_{i-j}}
\end{cases},
\]
and for $l_{t_{i}, j}, j=0,\dots,p-1$ 
in Eq.~\eqref{eq:RBF-Solver Integral},
the limit is given by
\[ \lim_{\gamma_{t_{i}} \to 0}l_{{t_{i}}, j} 
= \lim_{\gamma_{t_{i}} \to 0} \int_{\lambda_{t_i}}^{\lambda_{t_{i+1}}} e^{\lambda} \phi_{{t_{i}}, j}(\lambda) \,\mathrm{d}\lambda =0. \]

Therefore, for the Kernel matrix $\mathbf{\Phi}_{t_i}^{(p)}$ defined in Section \ref{sec:Derivation of Sampling Equations}
and the integral vector $\bm{l}_{t_{i}}^{(p)}$ introduced in Eq.~(\ref{eq:RBF-Solver Equation}),
as the shape parameter $\gamma_{t_{i}}$ tends to zero, the matrices converge to the following limiting forms:
\begin{equation*}
\begin{aligned}
\mathbf{\Phi}^{(p)}_{\text{limit}} 
&\coloneqq
\begin{pmatrix}
1 
& 0
& \cdots 
& 0 
& 1
\\[6pt]
0 
& 1
& \cdots 
& 0
& 1
\\
\vdots 
& \vdots
& \ddots 
& \vdots 
& \vdots
\\
0
& 0
& \cdots 
& 1
& 1
\\[6pt]
1 & 1& \cdots & 1 & 0
\end{pmatrix}
=\begin{bmatrix}
\mathbf{I} & &\mathbf{1} \\
 & &\\
\mathbf{1}^T & &0
\end{bmatrix},
\,\,\,
\bm{l}^{(p)}_{\text{limit}} 
&\coloneqq
\begin{pmatrix}
0 \\[6pt] 
0\\
\vdots\\
0\\[6pt]
l_{t_{i}}^{\text{const}}
\end{pmatrix}
\end{aligned}
\end{equation*}
where $\mathbf{I}$ is the $p \times p$ identity matrix
and $\mathbf{1}$ is a vector of ones of length $p$.
Since the coefficient vector \( \bm c_{t_{i}}^{(p)} \) of RBF sampling equation in Eq.~\eqref{eq:RBF-Solver PC}
satisfies the linear system
\( \bm {\Phi}^{(p)}_{\text{limit}} \bm c_{t_{i}}^{(p)} = \bm l^{(p)}_{\text{limit}} \) in Eq.~\eqref{eq:RBF-Solver Equation},
the coefficients can be computed by inverting
\( \big(\mathbf{\Phi}^{(p)}_{\text{limit}}\big)^{-1}  \) as
\[
\big(\mathbf{\Phi}^{(p)}_{\text{limit}}\big)^{-1} = 
\begin{bmatrix}
\mathbf{I} - \frac{1}{p} \mathbf{1}\mathbf{1}^T & \frac{1}{p} \mathbf{1} \\
 & \\
\frac{1}{p}\mathbf{1}^T & -p
\end{bmatrix}
=
\begin{pmatrix}
1-1/p
& -1/p
& \cdots 
& -1/p 
& 1/p
\\[6pt]
-1/p 
& 1-1/p
& \cdots 
& -1/p
& 1/p
\\
\vdots 
& \vdots
& \ddots 
& \vdots 
& \vdots
\\
-1/p
& -1/p
& \cdots 
& 1-1/p
& 1/p
\\[6pt]
1/p & 1/p& \cdots & 1/p & -p
\end{pmatrix},
\]
and by directly computing
\( \bm c_{t_{i}}^{(p)} =\big(\mathbf{\Phi}^{(p)}_{\text{limit}}\big)^{-1} 
\bm l^{(p)}_{\text{limit}}  \). 
Finally, we obtain
\[  \big[\bm c_{t_{i}}^{(p)}\big]_0^{p-1} =
    \left(
    \frac{1}{p}
    l_{t_{i}}^{\text{const}} \,,
    \frac{1}{p}
    l_{t_{i}}^{\text{const}}\,,\dots\,,
    \frac{1}{p}
    l_{t_{i}}^{\text{const}}
    \right)^{\top} 
    \in
    \mathbb{R}^{\raisebox{.3ex}{$\scriptstyle p$}}.
\]
(For a vector $\mathbf{v}\in\mathbb{R}^{p+1}$,
we denote by $[\mathbf{v}]_{j}^k$ the subvector in $\mathbb{R}^{k+j-1}$ consisting of the $j$-th through $k$-th components of $\mathbf{v}$.)
As a result, the sampling equation in 
Eq.~\eqref{eq:RBF-Solver Equation} reduces to
\begin{equation*}
\label{eq:RBF-Solver Equation of Euler case}
\tilde{\bm x}_{t_{i+1}} =
\frac{\sigma_{t_{i+1}}}{\sigma_{t_{i}}}
\tilde{\bm x}_{t_{i}}
+
\sigma_{t_{i+1}}
\frac{l_{t_{i}}^{\text{const}}}{p}
\sum_{j=0}^{p-1}
{\bm{\hat x}}_\theta(\bm{\hat x}_{{t_{i-j}}},\lambda_{t_{i-j}})
\end{equation*}
which results in a sampling method
with equal coefficients, $l_{i}^{\text{const}}/p$.
This can be derived via a piecewise-constant regression quadrature approximation for the integral 
$ \int_{\lambda_{t_{i}}}^{\lambda_{t_{i+1}}}
e^{\lambda} \bm{\hat x}_\theta(\bm{\hat x}_\lambda,\lambda)
\mathrm{d}\lambda $ 
using the previous values 
$\{\bm{\hat x}_\theta(\bm{\hat x}_{\lambda_{t_{i-j}}},\lambda_{t_{i-j}})\}_{j=0}^{p-1}$.

\subsection{Reparameterization of RBF Coefficients
Using Lagrange Interpolation Convergence}
\label{sec:RBF coefficients}
Recall the Lagrange-type RBF interpolation 
interpolating the values
$\{ \bm{\hat x}_\theta(\bm{\hat x}_{\lambda_{t_{i-j}}},\lambda_{t_{i-j}}) \}_{j=0}^{p-1}$, in Eq.~\eqref{eq:RBF Lagrange-type} 
\begin{equation*} 
\mathcal{R}(\lambda\,;\bm{\lambda}_{t_i}^{(p)}) = \sum_{j=0}^{p-1} u_{j}(\lambda) 
\bm{\hat x}_\theta(\bm{\hat x}_{\lambda_{t_{i-j}}}, \lambda_{t_{i-j}}) 
\end{equation*} 
and its convergence result given
in Lemma \ref{lemma:RBF convergence}
\[   \lim_{\gamma_{t_{i}} \to \infty}
\mathcal{R}(\lambda\,;\bm{\lambda}_{t_{i}}^{(p)}) 
= \mathcal{L}^{(p)}(\lambda).
\]
We consider a reparameterization of RBF coefficients \( u_j(\cdot)\) by Taylor expansion of 
\( u_j(\cdot)\) with respect to \( \frac{h}{\beta} \), where $ \beta:=\gamma_{t_{i}}\cdot h_{t_{i}} $ 
as follows: 
\begin{equation} \label{eq:Taylor_of_coefficients}
u_j(\cdot)= \ell_j(\cdot) + a_{j2}(\cdot) \bigg( \frac{h}{\beta} \bigg)^2  + a_{j4}(\cdot) \bigg( \frac{h}{\beta} \bigg)^4 + a_{j6}(\cdot) \bigg( \frac{h}{\beta} \bigg)^6 + \cdots,     
\end{equation}
where $\ell_j(\cdot)$ is the Lagrange basis in Eq.~\eqref{eq:lagrange-interpolation}.
\newline \noindent
As an example, we provide explicit expressions for the case $p=3$ in Eq.~\eqref{eq:RBF_interpolation} when 
\( \mathcal{R}(\lambda\,;\bm{\lambda}_{t_{i}}^{(3)}) \) 
is defined on $\lambda_{t_i}=0$ and
\( \bm{\lambda}_{t_{i}}^{(3)}:=
(\lambda_{t_{i}}, \lambda_{t_{i-1}}, \lambda_{t_{i-2}} )
= (0,-0.1,-0.2) \).
The explicit expressions of the functions 
$ \ell_j(\lambda), a_{j2}(\lambda), a_{j4}(\lambda)$ in this case are given in the following table:
\[ 
\begin{array}{c|>{\columncolor{gray!20}}c|c|c|c}
& \ell_j(\lambda) & a_{j2}(\lambda)
& a_{j4}(\lambda) & \cdots
\\ \hline & & & &\\
j=0 
& {\left(10\,\lambda+1\right)}\,{\left(5\,\lambda+1\right)} & -\frac{\lambda\,{\left(2500\,\lambda^3 +1300\,\lambda^2 +215\,\lambda+11\right)}}{60} 
& \cdots
\\  & & & &\\
j=1& -20\,\lambda\,{\left(5\,\lambda+1\right)} & \frac{\lambda\,{\left(5\,\lambda+1\right)}\,(10\,\lambda + 1)^2 }{6} 
&  \cdots
\\  & & & &\\
j=2 & 5\,\lambda\,{\left(10\,\lambda+1\right)} & -\frac{\lambda\,{\left(2500\,\lambda^3 +700\,\lambda^2 +35\,\lambda-1\right)}}{60} 
&\cdots
\\  & & & &
\end{array} 
\]

We rearrange the representation of $\mathcal{R}(\lambda\,;\bm{\lambda}_{t_{i}}^{(3)})$ to show that it has the same third-order Taylor expansion as that of 
$\bm{\hat x}_\theta(\bm{\hat x}_{\lambda}, \lambda)$. 
For simplicity,
we use the simplified notation $\bm{\hat x}_\theta(\lambda_{t_{i-j}}) $ instead of full form $\bm{\hat x}_\theta(\bm{\hat x}_{\lambda_{t_{i-j}}}, \lambda_{t_{i-j}})$ for each $j=0,1,2$ in this section.
\begin{align*}
\mathcal{R}(\lambda\,;\bm{\lambda}_{t_{i}}^{(3)}) = \sum_{j=0}^{2} u_j(\lambda)\, \bm{\hat x}_\theta(\lambda_{t_{i-j}}) 
&= \sum_{j=0}^{2} u_j(\lambda) \bigg( \sum_{k=0} \frac{\bm{\hat x}_\theta^{(k)}(\lambda)}{k!} (\lambda_{t_{i-j}} - \lambda_{t_{i}})^k \bigg) \\
&= \sum_{k=0} \frac{\bm{\hat x}_\theta^{(k)}(\lambda)}{k!} \bigg( \sum_{j=0}^{2} u_j(\lambda) \cdot (\lambda_{t_{i-j}} - \lambda_{t_{i}})^k \bigg).
\end{align*}
By using the Taylor expansion of \( u_j(\cdot)\) in Eq.~\eqref{eq:Taylor_of_coefficients}, we represent the interpolant as
\begin{align*}
&\mathcal{R}(\lambda\,;\bm{\lambda}_{t_{i}}^{(3)}) \\
&= \sum_{k=0}^{2} \frac{\bm{\hat x}_\theta^{(k)}(\lambda)}{k!} 
\bigg( \sum_{j=0}^{2} \ell_j(\lambda)(\lambda_{t_{i-j}} - \lambda_{t_{i}})^k \bigg)
+ \bm{\hat x}_\theta(\lambda) 
\bigg( \sum_{j=0}^{2} a_{j2}(\lambda)(\lambda_{t_{i-j}} - \lambda_{t_{i}})^k  \bigg) \bigg( \frac{h}{\beta} \bigg)^2
+ \mathcal{O}(h^3). 
\end{align*}
Since the functions 
\(\ell_j(\lambda), \, a_{j2}(\lambda), \, j=0,1,2 \, \) in the above table satisfy
\begin{align*}
\sum_{j=0}^{2} \ell_j(\lambda)\cdot(\lambda_{t_{i-j}} - \lambda_{t_{i}})^k  
&=  (\lambda - \lambda_{t_{i}})^k, \quad k=0,1,2 \\
\sum_{j=0}^{2} a_{j2}(\lambda)\cdot(\lambda_{t_{i-j}} - \lambda_{t_{i}})^k  
&= 0, \quad k=0,
\end{align*}
with $\lambda_{t_i}=0$ and
\( \bm{\lambda}_{t_{i}}^{(3)}:=
(\lambda_{t_{i}}, \lambda_{t_{i-1}}, \lambda_{t_{i-2}} )
= (0,-0.1,-0.2) \),
we have that
\[ \mathcal{R}(\lambda\,;\bm{\lambda}_{t_{i}}^{(3)})
= \sum_{k=0}^{2} \frac{\bm{\hat x}_\theta^{(k)}(\lambda)}{k!} (\lambda - \lambda_{t_{i}})^k + \mathcal{O}(h^3)\]
which coincides with the third-order Taylor expansion of $\bm{\hat x}_\theta(\lambda) = \bm{\hat x}_\theta(\bm{\hat x}_{\lambda}, \lambda)$, thus confirming third-order accuracy.
Furthermore, if we let $\lambda_{t_{i+1}}:=0.1$ for this case, we can directly calculate the coefficients of Lagrange interpolation-based sampling (Adams method in Appendix \ref{sec:Lagrange-Solver}) using 
$ \int_{0}^{0.1}
e^{\lambda}\,\ell_{j}(\lambda) \mathrm{d}\lambda, j=0,1,2 $
and have the coefficients of Adams method as follows:
\begin{equation*}
\bm c_{t_{i}}^{(p)} 
= (78\,{\mathrm{e}}^{1/10} -86,
180-163\,{\mathrm{e}}^{1/10},
86\,{\mathrm{e}}^{1/10} -95)
\approx
(0.2033, -0.1429, 0.0447 )
\end{equation*}
which coincides with
Figure~\ref{fig:two-panel} (b) for the case Order-$3$.

\section{Stable Evaluation of the Integral in RBF-Solver}
\label{sec:how to calc integral}

In this section, we explain how to evaluate the integral of the Gaussian function in Eq.~\eqref{eq:RBF-Solver Integral} while ensuring numerical stability. Gaussian function can be integrated using either analytical or numerical methods. 

\subsection{Analytic Methods} 
One of the analytical methods is to integrate the Gaussian function using the error function~($\mathrm{erf}$). The procedure is as follows:
We first rescale the integration interval to $[0,1]$
\begin{align*}
l_{{t_{i}}, j}
&=
\int_{\lambda_{t_i}}^{\lambda_{t_{i+1}}}
\exp\!\Biggl[
\lambda -
\Bigl(\tfrac{\lambda - \lambda_{t_{i-j}}}{\gamma_{t_{i}} h_{t_{i}}}\Bigr)^{2}
\Biggr]
\,\mathrm{d}\lambda \\[4pt]
&=
\exp(\lambda_{t_{i+1}})
\int_{\lambda_{t_i}}^{\lambda_{t_{i+1}}}
\exp\!\Biggl[
(\lambda - \lambda_{t_{i+1}})
-
\tfrac{1}{\gamma_{t_{i}}^{2}}
\Bigl(\tfrac{\lambda - \lambda_{t_{i-j}}}{h_{t_{i}}}\Bigr)^{2}
\Biggr]
\,\mathrm{d}\lambda \\[4pt]
&=
\exp(\lambda_{t_{i+1}})
\int_{\lambda_{t_i}}^{\lambda_{t_{i+1}}}
\exp\!\Biggl[
(r-1)h_{t_{i}}
-
\tfrac{1}{\gamma_{t_{i}}^{2}}
\bigl(r - r_{i-j}\bigr)^{2}
\Biggr]
\,\mathrm{d}\lambda \\[4pt]
&=
\exp(\lambda_{t_{i+1}})\,
h_{t_{i}}
\int_{0}^{1}
\exp\!\Biggl[
(r-1)h_{t_{i}}
-
\tfrac{1}{\gamma_{t_{i}}^{2}}
\bigl(r - r_{i-j}\bigr)^{2}
\Biggr]
\,\mathrm{d}r
\\[6pt]
&\quad\text{where } 
r \coloneqq \tfrac{\lambda - \lambda_{t_i}}{h_{t_{i}}},
\quad
r_{i-j} \coloneqq \tfrac{\lambda_{t_{i-j}} - \lambda_{t_i}}{h_{t_{i}}}.
\end{align*}

The above integral has the following closed-form solution in terms of the error function~($\mathrm{erf}$):

\begin{equation*}
=
\frac{\sqrt{\pi}\gamma_{t_{i}}}{2}
h_{t_{i}}
\exp
\left[\lambda_{t_{i+1}}
+
\frac{\gamma_{t_{i}}^2 h_{t_{i}}^2}{4}
+
h_{t_{i}}(r_{i-j} - 1)
\right]
\left[
\mathrm{erf}
\left(
\frac{r_{i-j} + \frac{\gamma_{t_{i}}^2h_{t_{i}}}{2}}{\gamma_{t_{i}}}
\right)
-
\mathrm{erf}
\left(
\frac{r_{i-j} + \frac{\gamma_{t_{i}}^2h_{t_{i}}}{2} - 1}{\gamma_{t_{i}}}
\right)
\right]
\end{equation*}

When $\gamma_{t_{i}}$ is very large, the exponential factor grows rapidly while the difference between the error functions becomes negligibly small; to mitigate precision loss, we therefore evaluate the entire expression in log space. Moreover, because this small difference $\mathrm{erf}(a)-\mathrm{erf}(b)$ is prone to catastrophic cancellation for large $\gamma_{t_{i}}$, we rewrite it using the complementary error function and carry out the computation in log space as follows.

\begin{equation*}
\mathrm{erf}(a)-\mathrm{erf}(b) =
\mathrm{erfc}(b)
\left[
1 - 
\exp
\left(
\log \mathrm{erfc}(a)
- \log \mathrm{erfc}(b)
\right)
\right]
\end{equation*}

\subsection{Numerical Methods}
When the width of the Gaussian function is extremely narrow, numerical issues can arise since $\mathrm{erf}(\infty) - \mathrm{erf}(\infty) = 0$.
Conversely, when the width is extremely large, the value of $\mathrm{erf}(x)$
quickly saturates toward $1$, making it difficult to accurately compute the integral.
In such cases, Gaussian Quadrature, one of the numerical methods, efficiently places nodes near the region with the highest contribution, enabling highly accurate integration even with a small number of nodes.

Gaussian Quadrature can exactly integrate polynomials of up to degree \(2n - 1\)  using \(n\) evaluation points,
\[
\int_a^b f(x) \, dx \approx \sum_{i=1}^{n} w_i f(x_i)
\]
where the nodes \(x_i\) are determined by the roots of orthogonal polynomials, such as Legendre, Hermite, and Laguerre polynomials, \( w_i \) are the weights corresponding to each node, and \( n \) is the number of nodes used.
If we apply Gaussian-Legendre quadrature to the integral of the Gaussian function in Eq.~\eqref{eq:RBF-Solver Integral}, the integration interval 
$[\lambda_{t_i}, \lambda_{t_{i+1}}]$ must be transformed to $[-1,1]$:
\begin{align*}
l_{{t_{i}}, j}
&=
\int_{\lambda_{t_i}}^{\lambda_{t_{i+1}}}
\exp\!\Biggl[
\lambda -
\Bigl(\tfrac{\lambda - \lambda_{t_{i-j}}}{\gamma_{t_{i}} h_{t_{i}}}\Bigr)^{2}
\Biggr]
\,\mathrm{d}\lambda \\[4pt]
&=
\frac{\lambda_{t_{i+1}} - \lambda_{t_i}}{2} \int_{-1}^{1} f\left( 
\frac{\lambda_{t_{i+1}} - \lambda_{t_i}}{2} \xi + \frac{\lambda_{t_{i+1}} + \lambda_{t_i}}{2} 
\right) d\xi, 
\quad f(\lambda) = \exp\!\Biggl[
\lambda -
\Bigl(\tfrac{\lambda - \lambda_{t_{i-j}}}{\gamma_{t_{i}} h_{t_{i}}}\Bigr)^{2}
\Biggr].\\[4pt]
\end{align*}
The integral can be approximated as follows:
\[
l_{t_i, j} \approx \frac{\lambda_{t_{i+1}} - \lambda_{t_i}}{2} \sum_{k=1}^{n} w_k \, f \left( 
\frac{\lambda_{t_{i+1}} - \lambda_{t_i}}{2} \xi_k + \frac{\lambda_{t_{i+1}} + \lambda_{t_i}}{2} 
\right)
\]
where \( \xi_k \) are the Gaussian-Legendre nodes, \( w_k \) are the corresponding weights, and \( n \) is the number of quadrature points.

\section{Adams Method}
This section derives the Adams-method sampling equations from Lagrange interpolation and shows that it satisfies the \emph{coefficient summation condition} in Eq.~\eqref{eq:Summation Condition}.
\subsection{Sampling Equations in Adams Method}
\label{sec:Lagrange-Solver}
As shown in Section~\ref{sec:Lagrange Convergence}, RBF interpolation converges to Lagrange interpolation as the shape parameter $\gamma_{t_i}$ approaches infinity.  
For numerical stability, therefore, RBF-Solver switches to an Adams scheme whenever the optimal 
$\gamma_{t_i}$ is large—using Adams–Bashforth as the predictor and Adams–Moulton as the corrector.  
We detail the resulting sampling equations and their implementation below.

We first approximate the evaluation function $\bm{x}_\theta$ in the exact solution (Eq.~\ref{eq:exact_solution_x0}) with the Lagrange interpolation (Eq.~\ref{eq:lagrange-interpolation}).  
This substitution yields the following sampling equation, which advances a sample $\tilde{\bm x}_{t_i}$ to the next time step $t_{i+1}$:

\begin{equation}
\label{eq:Lagrange-Solver Interpolation}
\tilde{\bm\x}_{t_{i+1}} =
\frac{\sigma_{t_{i+1}}}{\sigma_{t_{i}}}
\tilde{\bm\x}_{t_{i}}
+
\sigma_{t_{i+1}}
\int_{\lambda_{t_i}}^{\lambda_{t_{i+1}}}
e^{\lambda}
\left[
\sum_{j=0}^{p-1} \ell_{j}(\lambda) 
\hat{\bm{x}}_\theta(\hat{\bm{x}}_{\lambda_{t_{i-j}}},\lambda_{t_{i-j}})
\right]
\mathrm{d}\lambda
,\,\,\,
\ell_{j}(\lambda)
= 
\prod_{\substack{k=0 \\ k \neq j}}^{p-1}
\frac{\lambda - \lambda_{t_{i-k}}}
{\lambda_{t_{i-j}} - \lambda_{t_{i-k}}}.
\end{equation}

Next, we interchange the integral and the summation, rewriting each Lagrange polynomial $\ell_{j}(\lambda),j=0,\dots,p-1$ as a linear combination of monomials $\lambda^{k}$ with coefficients $\beta_{j,k}$:

\begin{equation*}
\tilde{\bm\x}_{t_{i+1}} =
\frac{\sigma_{t_{i+1}}}{\sigma_{t_{i}}}
\tilde{\bm\x}_{t_{i}}
+
\sigma_{t_{i+1}}
\sum_{j=0}^{p-1}
\hat{\bm{x}}_\theta(\hat{\bm{x}}_{\lambda_{t_{i-j}}},\lambda_{t_{i-j}})
\int_{\lambda_{t_{i}}}^{\lambda_{t_{i+1}}}
e^{\lambda}
\sum_{k=0}^{p-1}
\beta_{j,k}\lambda^k
\mathrm{d}\lambda
\end{equation*}

The coefficient vector $\bm{\beta}_{j}\in\mathbb{R}^{p}:=(\beta_{j,0},\beta_{j,1},\ldots,\beta_{j,p-1})$ can be obtained from the Vandermonde matrix $\bm{V}_{t_{i}}^{(p)}\in\mathbb{R}^{p\times p}$ and the standard basis vector $\bm{e}_{j}\in\mathbb{R}^{p}$ as follows:

\begin{equation*}
\bm{\beta}_{j}=(\bm{V}_{t_{i}}^{(p)})^{-1}\bold e_j,
\,\,\,
\bm{V}_{t_{i}}^{(p)} \;:=\;
\begin{pmatrix}
1 & \lambda_{t_{i-0}} & \cdots & (\lambda_{t_{i-0}})^{p-1}\\
1 & \lambda_{t_{i-1}} & \cdots & (\lambda_{t_{i-1}})^{p-1}\\
\vdots & \vdots & \ddots & \vdots\\
1 & \lambda_{t_{i-(p-1)}} & \cdots & (\lambda_{t_{i-(p-1)}})^{p-1}
\end{pmatrix}. 
\end{equation*}

Applying this, we obtain

\begin{equation*}
\tilde{\bm\x}_{t_{i+1}} =
\frac{\sigma_{t_{i+1}}}{\sigma_{t_{i}}}
\tilde{\bm\x}_{t_{i}}
+
\sigma_{t_{i+1}}
\sum_{j=0}^{p-1}
\hat{\bm{x}}_\theta(\hat{\bm{x}}_{\lambda_{t_{i-j}}},\lambda_{t_{i-j}})
\sum_{k=0}^{p-1}
[(\bm{V}_{t_{i}}^{(p)})^{-\top}]_{jk}
\int_{\lambda_{t_{i}}}^{\lambda_{t_{i+1}}}
e^{\lambda}
\lambda^k
\mathrm{d}\lambda.
\end{equation*}

Defining the scalar integral $l_{{t_{i}},k}$, we evaluate it in closed form as follows:
\begin{equation*}
l_{{t_{i}},k}\coloneqq
\int_{\lambda_{t_i}}^{\lambda_{t_{i+1}}} e^\lambda \,\lambda^k \, d\lambda
\;=\;
e^\lambda \sum_{m=0}^{k}
(-1)^m \,\frac{k!}{(k - m)!}\,\lambda^{\,k-m}
\;\Biggr|_{\lambda_{t_i}}^{\lambda_{t_{i+1}}}.
\end{equation*}

Define the integral vector
\[
\bm{l}_{t_{i}}^{(p)}:=\bigl(l_{{t_{i}},0},\,l_{{t_{i}},1},\,\ldots,\,l_{{t_{i}},p-1}\bigr)^{\!\top}\in\mathbb{R}^{p},
\]
and the evaluation vector
\[
\bm{X}_{t_{i}}^{(p)}
:=
\Bigl(
\hat{\bm{x}}_{\theta}\!\bigl(\hat{\bm{x}}_{\lambda_{t_i}},\lambda_{t_i}\bigr),
\,
\hat{\bm{x}}_{\theta}\!\bigl(\hat{\bm{x}}_{\lambda_{t_{i-1}}},\lambda_{t_{i-1}}\bigr),
\,
\ldots,
\hat{\bm{x}}_{\theta}\!\bigl(\hat{\bm{x}}_{\lambda_{t_{i-p+1}}},\lambda_{t_{i-p+1}}\bigr)
\Bigr)^{\!\top}\!.
\]
With these definitions, the sampling equation can be expressed as a dot product between the evaluation vector and the coefficient vector:
\begin{equation*}
\tilde{\bm x}_{t_{i+1}} =
\frac{\sigma_{t_{i+1}}}{\sigma_{t_{i}}}
\tilde{\bm x}_{t_{i}}
+
\sigma_{t_{i+1}}
(\bm c_{t_{i}}^{(p)})^{\!\top}\,\bm X_{t_{i}}^{(p)}
,\,\,
\bm c_{t_{i}}^{(p)} = (\bm{V}_{t_{i}}^{(p)})^{-\top}\bm{l}_{t_{i}}^{(p)}.
\end{equation*}

Building on the above sampling equation, we formulate the predictor--corrector pair within the linear multistep framework~\cite{butcher2016numerical}.
The predictor uses the most recent \(p\) evaluations up to the current time step \(t_i\), whereas the corrector utilizes the preceding \(p+1\) evaluations up to the next time step \(t_{i+1}\).
The complete sampling algorithm, summarized in Algorithm~\ref{alg:predictor-corrector}, can be directly employed for sampling in the same way as RBF-Solver.

\begin{equation*}
\left\{
\begin{aligned}
\tilde{\bm x}_{t_{i+1}}^{\text{pred}}
&=
\text{Predictor}\bigl(\tilde{\bm x}_{t_i}^{\text{corr}}\bigr)
=
\frac{\sigma_{t_{i+1}}}{\sigma_{t_i}}
\tilde{\bm x}_{t_i}^{\text{corr}}
+
\sigma_{t_{i+1}}
(\bm c_{t_{i}}^{(p)})^{\top}\bm X_{t_{i}}^{(p)},\\[6pt]
\tilde{\bm x}_{t_{i+1}}^{\text{corr}}
&=
\text{Corrector}\bigl(\tilde{\bm x}_{t_i}^{\text{corr}}\bigr)
=
\frac{\sigma_{t_{i+1}}}{\sigma_{t_i}}
\tilde{\bm x}_{t_i}^{\text{corr}}
+
\sigma_{t_{i+1}}
(\bm c_{{t_{i+1}}}^{(p+1)})^{\top}\bm X_{t_{i+1}}^{(p+1)}.
\end{aligned}
\right.
\end{equation*}

\subsection{Coefficient Summation Condition in Adams Method}
\label{sec:Coefficients Summation Condition in Lagrange-Solver}

This section proves that the sampling equation in Eq.~\eqref{eq:Lagrange-Solver Interpolation}, derived from Lagrange interpolation, satisfies the \emph{coefficient summation condition} in Section~\ref{sec:Coefficients-Summation-Condition}. We first establish the \emph{partition-of-unity} property of the Lagrange basis and then apply it to the proof.

\begin{lemma}[Partition-of-Unity]
For every \(\lambda\in\mathbb{R}\), the Lagrange basis satisfies
\[
\sum_{j=0}^{p-1}\ell_{j}(\lambda)=1.
\]
\end{lemma}

\begin{proof}
The Lagrange interpolation polynomials satisfy
\(\ell_{j}(\lambda_{t_{i-k}})=\delta_{jk}\) at each node.
Hence the polynomial  
\(q(\lambda)\coloneqq\sum_{j=0}^{p-1}\ell_{j}(\lambda)-1\)
is zero at every node \(\{\lambda_{t_{i-k}}\}_{k=0}^{p-1}\).
Because \(\deg q \le p-1\) yet \(q\) has \(p\) distinct zeros, it must be identically zero.
\end{proof}

In the sampling equation of the Adams method Eq.~\eqref{eq:Lagrange-Solver Interpolation}, the sum of the coefficients applied to the model evaluations is
\begin{equation}
\sum_{j=0}^{p-1}
\int_{\lambda_{t_i}}^{\lambda_{t_{i+1}}}
e^{\lambda}\,
\ell_{j}(\lambda)\,
\mathrm{d}\lambda.
\end{equation}

Using the partition-of-unity property, we obtain
\[
\sum_{j=0}^{p-1}\!
\int_{\lambda_{t_i}}^{\lambda_{t_{i+1}}}
e^{\lambda}\,\ell_{j}(\lambda)\,\mathrm{d}\lambda
=
\int_{\lambda_{t_i}}^{\lambda_{t_{i+1}}}
e^{\lambda}\Bigl(\sum_{j=0}^{p-1}\ell_{j}(\lambda)\Bigr)\,\mathrm{d}\lambda
=
\int_{\lambda_{t_i}}^{\lambda_{t_{i+1}}} e^{\lambda}\,\mathrm{d}\lambda.
\]

Consequently, the Adams method satisfies the coefficient summation condition.

\section{Lagrange Interpolation-based Approach}
\label{sec:Lagrange Interpolation-based Approach}

In this section, we aim to understand UniPC as a Lagrange interpolation-based approach and prove its equivalence to SA-Solver under certain conditions.

UniPC~\citep{zhao2023unipc} adopts Lagrange interpolation for model prediction, but differs from other Lagrange interpolation-based approaches~\citep{Zhang2023fast, xue2024sa} in that it computes the coefficients directly, rather than explicitly using the Lagrange basis (Appendix~\ref{sec:UniPC}). 
UniPC also introduces several key modifications. 
Specifically, UniPC incorporates the $B(h)$ function and applies a Taylor approximation to adjust the coefficient when only a single coefficient appears in the second-order formulation(Appendix~\ref{sec:UniPC 2nd Order Coefficients}).

As a variance-controlled diffusion SDE solver, SA-Solver~\citep{xue2024sa} employs stochastic Adams method within a predictor-corrector scheme which predicts $\bm{\hat x}_\theta$ using Adams-Bashforth and corrects it by substituting into Adams-Moulton.
In the zero-variance setting, SA-Solver is equivalent to UniPC with $B(h)=h$ (Appendix~\ref{sec:Equivalence between UniPC and SA-Solver}).

\subsection{UniPC: Interpreting UniPC as a Lagrange Interpolation-Based Approach}
\label{sec:UniPC}
In this section, we revisit UniPC method in \cite{zhao2023unipc} and reformulate the sampling method in the form of a Lagrange interpolation-based approach.

Recall the sampling equation of the Lagrange interpolation-based Adams method in Appendix~\ref{sec:Lagrange-Solver}:
\begin{equation*}
\tilde{\bm x}_{t_{i+1}} =
\frac{\sigma_{t_{i+1}}}{\sigma_{t_{i}}}
\tilde{\bm x}_{t_{i}}
+
\sigma_{t_{i+1}}
(\bm c_{t_{i}}^{(p)})^{\!\top}\,\bm X_{t_{i}}^{(p)}
= \frac{\sigma_{t_{i+1}}}{\sigma_{t_{i}}}
\tilde{\bm x}_{t_{i}}
+ \sigma_{t_{i+1}}\sum_{k=0}^{p-1} 
\big[\bm c_{t_{i}}^{(p)}\big]_k \, \bm{\hat x}_\theta(\bm{\hat x}_{\lambda_{t_{i-k}}},\lambda_{t_{i-k}}).
\end{equation*}
To compare with 
the coefficient vector $\bm{a}_p=(\tilde{a}_1,\dots,\tilde{a}_{p-1})^\top$ for the $p$-th order data prediction model UniP-$p$ in {\bf Appendix A.2}
of UniPC~\citep{zhao2023unipc}, we define $a_k$ 
to reformulate the coefficients $\big[c_{t_{i}}^{(p)}\big]_k$ 
as 
\[h_{t_i} a_k := \frac{\alpha_{t_{i+1}}}{\sigma_{t_{i+1}}}\,\big[c_{t_{i}}^{(p)}\big]_k, \quad k=0,\dots,p-1. \]
Then the sampling equation of the Lagrange interpolation-based method can be represented as
\begin{align} \label{eq:sampling_UniPC}
\bm{x}_{t_{i+1}} 
&= \frac{\sigma_{t_{i+1}}}{\sigma_{t_i}}\bm{x}_{t_i} + 
\alpha_{t_{i+1}} \sum_{k=0}^{p-1} h_{t_i} a_k\bm{\hat x}_\theta(\bm{\hat x}_{\lambda_{t_{i-k}}},\lambda_{t_{i-k}}).
\end{align}
Let
\[ r_j = \frac{\lambda_{t_{i-j}} - \lambda_{t_i}}{h_{t_{i}}}, \quad j=0,\dots,p-1 \]
and define the functions $\psi_k(h)$ as
\[ 
\psi_k(h) := \int_0^1 e^{(r - 1)h} \frac{r^{k-1}}{(k-1)!} \, \mathrm{d}r, 
\qquad \varphi_0(h) = e^{-h}.
\]
From the construction of the coefficient 
$\bm c_{t_{i}}^{(p)}$ of Adams method
in Appendix~\ref{sec:Lagrange-Solver},
the coefficients
\( \bm{a}:=\bigl(a_0, \dots, a_{p-1}\bigr)^{\top}\in\mathbb{R}^{p} \)
satisfy
the following \( p\times p \) linear system:
\begin{equation} \label{eq:Lagrange Reproducing}
    \begin{pmatrix} 
        1 & 1 & \dots & 1 \\ 
        r_0 & r_1 & \dots & r_{p-1} \\ 
        \vdots & \vdots & \ddots & \vdots \\ 
        r_0^{p-1} & r_1^{p-1} & \dots & r_{p-1}^{p-1} 
    \end{pmatrix}
    \begin{pmatrix}
        a_0 \\
        a_1 \\ 
        \vdots \\ 
        a_{p-1}
    \end{pmatrix} =
    \begin{pmatrix} 
       \psi_1(h_{t_i}) \\ 
       \psi_2(h_{t_i}) \\ 
        \vdots \\ 
       (p-1)!\psi_{p}(h_{t_i})
    \end{pmatrix}.
\end{equation}

To derive the coefficient formulas as in UniPC, we define auxiliary functions $\phi_k$ by 
\[ \phi_k(h):=h^k k! \psi_{k+1}(h), \quad k=0,1,\dots,p-1 \]
in the same manner as the auxiliary functions defined in \cite{zhao2023unipc} for the noise prediction model. By using this $\phi_k$, we reformulate the linear system Eq.~\eqref{eq:Lagrange Reproducing} as
\begin{equation} \label{eq:Lagrange Reproducing phi}
    \begin{pmatrix} 
        1 & 1 & \dots & 1 \\ 
        r_0h_{t_i} & r_1h_{t_i} & \dots & r_ph_{t_i} \\ 
        \vdots & \vdots & \ddots & \vdots \\ 
        (r_0h_{t_i})^{p-1} & (r_1h_{t_i})^{p-1} & \dots & (r_{p-1}h_{t_i})^{p-1} 
    \end{pmatrix}
    \begin{pmatrix}
        a_0 \\
        a_1 \\ 
        \vdots \\ 
        a_{p-1}
    \end{pmatrix} =
    \begin{pmatrix} 
       \phi_0(h_{t_i}) \\ 
       \phi_1(h_{t_i})\\ 
        \vdots \\ 
       \phi_{p-1}(h_{t_i})
    \end{pmatrix}.
\end{equation}
Using the fact \(r_0=0 \),
the linear system can be reduced to
\begin{equation*} 
    \begin{pmatrix} 
        1 & 1 & \dots & 1 \\ 
        0 & r_1h_{t_i} & \dots & r_ph_{t_i} \\ 
        \vdots & \vdots & \ddots & \vdots \\ 
        0 & (r_1h_{t_i})^{p-1} & \dots & (r_{p-1}h_{t_i})^{p-1} 
    \end{pmatrix}
    \begin{pmatrix}
        a_0 \\
        a_1 \\ 
        \vdots \\ 
        a_{p-1}
    \end{pmatrix} =
    \begin{pmatrix} 
       \phi_0(h_{t_i}) \\ 
       \phi_1(h_{t_i})\\ 
        \vdots \\ 
       \phi_{p-1}(h_{t_i})
    \end{pmatrix}.
\end{equation*}
Since \(a_1, \dots, a_{p-1}\) do not depend on 
\(a_0\), we first solve the \((p-1) \times (p-1)\) linear system
to find \(a_1, \dots, a_{p-1}\):
\begin{equation*}
    \begin{pmatrix}  
        r_1h_{t_i} & \dots & r_ph_{t_i} \\ 
        \vdots & \ddots & \vdots \\ 
        (r_1h_{t_i})^{p-1} & \dots & (r_{p-1}h_{t_i})^{p-1} 
    \end{pmatrix}
    \begin{pmatrix}
        a_1 \\ 
        \vdots \\ 
        a_{p-1}
    \end{pmatrix} =
    \begin{pmatrix} 
       \phi_1(h_{t_i})\\ 
        \vdots \\ 
       \phi_{p-1}(h_{t_i})
    \end{pmatrix},
\end{equation*}
equivalently,
\begin{equation} \label{eq:UniPC LS}
    \begin{pmatrix}  
        1 & \dots & 1 \\ 
        \vdots & \ddots & \vdots \\ 
        (r_1h_{t_i})^{p-2} & \dots & (r_{p-1}h_{t_i})^{p-2} 
    \end{pmatrix}
    \begin{pmatrix}
        r_1h_{t_i} a_1 \\ 
        \vdots \\ 
        r_{p-1}h_{t_i} a_{p-1}
    \end{pmatrix} =
    \begin{pmatrix} 
       \phi_1(h_{t_i})\\ 
        \vdots \\ 
       \phi_{p-1}(h_{t_i})
    \end{pmatrix}.
\end{equation}
Next, we find $a_0$ by using
\(    \sum_{m=0}^{p-1} a_m = \phi_0(h_{t_i}) \) as
\[ a_0= \phi_0(h_{t_i}) -\sum_{m=1}^{p-1} a_m.\]

Finally, the Lagrange interpolation-based sampling equation Eq.~\eqref{eq:sampling_UniPC} can be represented as:
\begin{align*}
\bm{x}_{t_{i+1}} 
&= \frac{\sigma_{t_{i+1}}}{\sigma_{t_i}}\bm{x}_{t_i} + 
\alpha_{t_{i+1}} \sum_{k=0}^{p-1} h_{t_i} \, a_k \,\bm{\hat x}_\theta(\bm{\hat x}_{\lambda_{t_{i-k}}},\lambda_{t_{i-k}})\\
&= \frac{\sigma_{t_{i+1}}}{\sigma_{t_i}}\bm{x}_{t_i} + 
\alpha_{t_{i+1}} h_{t_i} a_0 \bm{\hat x}_\theta(\bm{\hat x}_{\lambda_{t_i}},\lambda_{t_i})+ 
\alpha_{t_{i+1}}
\sum_{k=1}^{p-1} h_{t_i} \, a_k \,\bm{\hat x}_\theta(\bm{\hat x}_{\lambda_{t_{i-k}}},\lambda_{t_{i-k}}) \\
&= \frac{\sigma_{t_{i+1}}}{\sigma_{t_i}}\bm{x}_{t_i} + 
\alpha_{t_{i+1}} h_{t_i} \Bigg( \phi_0 -\sum_{k=1}^{p-1} a_k \Bigg) \bm{\hat x}_\theta(\bm{\hat x}_{\lambda_{t_i}},\lambda_{t_i})+ 
\alpha_{t_{i+1}}
\sum_{k=1}^{p-1} h_{t_i} \, a_k \,\bm{\hat x}_\theta(\bm{\hat x}_{\lambda_{t_{i-k}}},\lambda_{t_{i-k}}) \\
&= \frac{\sigma_{t_{i+1}}}{\sigma_{t_i}}\bm{x}_{t_i} + 
\alpha_{t_{i+1}} h_{t_i} \phi_0  \bm{\hat x}_\theta(\bm{\hat x}_{\lambda_{t_i}},\lambda_{t_i})+ 
\alpha_{t_{i+1}}
\sum_{k=1}^{p-1} h_{t_i}  a_k \Bigg(\bm{\hat x}_\theta(\bm{\hat x}_{\lambda_{t_{i-k}}},\lambda_{t_{i-k}}) - \bm{\hat x}_\theta(\bm{\hat x}_{\lambda_{t_i}},\lambda_{t_i}) \Bigg) \\
&= \frac{\sigma_{t_{i+1}}}{\sigma_{t_i}}\bm{x}_{t_i} + 
\alpha_{t_{i+1}} (1-e^{-h_{t_i}}) \,\bm{\hat x}_\theta(\bm{\hat x}_{\lambda_{t_i}},\lambda_{t_i})
+ 
\alpha_{t_{i+1}} h_{t_i} 
\sum_{k=1}^{p-1} a_k \,D_k  \\
&= \frac{\sigma_{t_{i+1}}}{\sigma_{t_i}}\bm{x}_{t_i} + 
\alpha_{t_{i+1}}
(1-e^{-h_{t_i}}) \,\bm{\hat x}_\theta(\bm{\hat x}_{\lambda_{t_i}},\lambda_{t_i})
+ 
\alpha_{t_{i+1}} h_{t_i} 
\sum_{k=1}^{p-1} \frac{\tilde{a}_k}{r_k} \,D_k,
\end{align*}
where \( D_k:=\bm{\hat x}_\theta(\bm{\hat x}_{\lambda_{t_{i-k}}},\lambda_{t_{i-k}}) - \bm{\hat x}_\theta(\bm{\hat x}_{\lambda_{t_{i}}},\lambda_{t_{i}}) \) and 
\( \tilde{a}_k := r_k \, a_k \).
The vector $(\tilde{a}_1, \dots,\tilde{a}_{p-1})$ in the above representation is identical to the coefficient vector $\bm{a}_p$ of the UniP-
$p$ data prediction in UniPC~\citep{zhao2023unipc}, 
since the linear system in Eq.~\eqref{eq:UniPC LS} coincides with the matrix equation presented in {\bf Appendix A.2} of UniPC~\citep{zhao2023unipc}.

\subsection{Proof of Equivalence between UniPC and SA-Solver} 
\label{sec:Equivalence between UniPC and SA-Solver}
Let's prove that UniPC~\citep{zhao2023unipc} with $B(h) = h$ is equivalent to SA-Solver~\citep{xue2024sa} in the zero-variance setting.
The SA-Predictor update equation is given by
\begin{equation}
\label{eq:SA-Predictor}
\begin{aligned}
\bm{x}_{t_{i + 1}} &= 
\frac{\sigma_{t_{i + 1}}}{\sigma_{t_{i}}} e^{-\int_{\lambda_{t_i}}^{\lambda_{t_{i+1}}} \tau^2(\lambda) \mathrm{d} \lambda} \bm{x}_{t_i} + 
\sum_{j=0}^{p-1} b_{i-j} \bm{\hat x}_{\theta}(\bm{\hat x}_{\lambda_{t_{i-j}}},\lambda_{t_{i-j}}) 
+ \Tilde{\sigma}_i \boldsymbol{\xi}, \hspace{4mm} \boldsymbol{\xi} \sim \mathcal{N}(\mathbf{0}, \boldsymbol{I}), \\
b_{i-j} &= \sigma_{t_{i + 1}}   \int_{\lambda_{t_i}}^{\lambda_{t_{i+1}}} e^{-\int_{\lambda}^{\lambda_{t_{i + 1}}} \tau^2(\lambda_v) \mathrm{d} \lambda_v }\left(1+\tau^2\left(\lambda\right)\right) e^{\lambda} \ell_{j}(\lambda) \mathrm{d} \lambda, 
\quad \ell_{j}(\lambda) = \prod_{\substack{k=0 \\ k \neq j}}^{p-1}
\frac{\lambda - \lambda_{t_{i-k}}}
{\lambda_{t_{i-j}} - \lambda_{t_{i-k}}},\\
\Tilde{\sigma}_i &= \sigma_{t_{i + 1}}\sqrt{1 - e^{-2\int_{\lambda_{t_i}}^{\lambda_{t_{i + 1}}} \tau^2(\lambda) \mathrm{d} \lambda}}.
\end{aligned}
\end{equation}
If $\tau(\lambda_u)=0$, we obtain the following simplified equation
\begin{equation}
\label{eq:SA-Predictor_wo_tau}
\begin{aligned}
\bm{x}_{t_{i + 1}} &= \frac{\sigma_{t_{i + 1}}}{\sigma_{t_{i}}} \bm{x}_{t_i} + 
\sigma_{t_{i + 1}}  \int_{\lambda_{t_i}}^{\lambda_{t_{i+1}}} 
e^{\lambda} 
\Bigg[\sum_{j=0}^{p-1} \ell_{j}(\lambda)
\bm{\hat x}_{\theta}(\bm{\hat x}_{\lambda_{t_{i-j}}},\lambda_{t_{i-j}}) \Bigg]
\mathrm{d} \lambda.
\end{aligned}
\end{equation}
The Lagrange basis $\ell_{j}(\lambda)$ can be transformed to $\ell_{j}(r)$ as follows
\begin{equation*}
\begin{aligned}
\ell_{j}(\lambda) &= \prod_{\substack{k=0 \\ k \neq j}}^{p-1}
\frac{\lambda - \lambda_{t_{i-k}}}
{\lambda_{t_{i-j}} - \lambda_{t_{i-k}}}
= \prod_{\substack{k=0 \\ k \neq j}}^{p-1} \frac{r h_{t_{i}} - r_k h_{t_{i}}}{r_j h_{t_{i}} - r_k h_{t_{i}}}
= \prod_{\substack{k=0 \\ k \neq j}}^{p-1} \frac{r - r_k}{r_j - r_k}
= \ell_{j}(r),\\
\lambda &= \lambda_{t_i} + r h_{t_{i}}, 
\quad r_j = \frac{\lambda_{t_{i-j}} - \lambda_{t_i}}{h_{t_{i}}}, 
\quad r_k = \frac{\lambda_{t_{i-k}} - \lambda_{t_i}}{h_{t_{i}}}.
\end{aligned}
\end{equation*}
Accordingly, Eq.~\ref{eq:SA-Predictor_wo_tau} can be rewritten as
\begin{equation}
\begin{aligned}
\bm{x}_{t_{i + 1}} &= \frac{\sigma_{t_{i + 1}}}{\sigma_{t_{i}}} \bm{x}_{t_i} + 
\sigma_{t_{i + 1}} 
\int_{\lambda_{t_i}}^{\lambda_{t_{i+1}}} 
e^{\lambda} 
\Bigg[\sum_{j=0}^{p-1} \ell_{j}(r) 
\bm{\hat x}_{\theta}(\bm{\hat x}_{\lambda_{t_{i-j}}},\lambda_{t_{i-j}})\Bigg]
\, \mathrm{d} \lambda \\
\end{aligned}
\end{equation}
If we define the Lagrange basis vector 
\( \bm{\ell}(r):=\bigl(\ell_{0}(r), \ell_{1}(r), \dots, \ell_{p-1}(r) \bigr)^{\top}\in\mathbb{R}^{p} \), it satisfies the polynomial reproduction property as follows:
\begin{equation*}
\bm{M} \cdot \bm{\ell}(r) = \bm{R}, \quad\quad
    \bm{R} := \begin{bmatrix} 
       1 \\ 
       r \\ 
        \vdots \\ 
       r^{p-1}
    \end{bmatrix}, \quad
    \bm{M} := \begin{bmatrix} 
        1 & 1 & \dots & 1 \\ 
        r_0 & r_1 & \dots & r_{p-1} \\ 
        \vdots & \vdots & \ddots & \vdots \\ 
        r_0^{p-1} & r_1^{p-1} & \dots & r_{p-1}^{p-1} 
    \end{bmatrix}.
\end{equation*}
By applying a change of variables to transform the integral from  $\lambda$-space to  $r$-space, and applying $\bm{\ell}(r) = \bm{M}^{-1} \bm{R}$, we obtain the following result:
\begin{equation}
\begin{aligned}
\bm{x}_{t_{i + 1}} 
&= \frac{\sigma_{t_{i + 1}}}{\sigma_{t_{i}}} \bm{x}_{t_i} + \alpha_{t+1} \int_0^1 e^{(r-1)h_{t_i}} 
\Bigg[\sum_{j=0}^p \ell_j(r)
\bm{\hat x}_{\theta}(\bm{\hat x}_{\lambda_{t_{i-j}}},\lambda_{t_{i-j}})\Bigg]
\, h_{t_i} \mathrm{d}r \\
&= \frac{\sigma_{t_{i + 1}}}{\sigma_{t_{i}}} \bm{x}_{t_i} + \alpha_{t+1} 
\sum_{j=0}^p h_{t_i} 
\bm{\hat x}_{\theta}(\bm{\hat x}_{\lambda_{t_{i-j}}},\lambda_{t_{i-j}})
\int_0^1 e^{(r-1)h} \ell_j(r) \, \mathrm{d}r \\
&= \frac{\sigma_{t_{i + 1}}}{\sigma_{t_{i}}} \bm{x}_{t_i} + \alpha_{t+1} 
\sum_{j=0}^p h_{t_i}   
\bm{\hat x}_{\theta}(\bm{\hat x}_{\lambda_{t_{i-j}}},\lambda_{t_{i-j}})
\int_0^1 e^{(r-1)h}  
\left(\sum_{k=0}^{p-1} (\bm{M}^{-1})_{jk} \bm{R}_k\right) 
\, \mathrm{d}r\\
&= \frac{\sigma_{t_{i + 1}}}{\sigma_{t_{i}}} \bm{x}_{t_i} + \alpha_{t+1} 
\sum_{j=0}^p h_{t_i}   
\bm{\hat x}_{\theta}(\bm{\hat x}_{\lambda_{t_{i-j}}},\lambda_{t_{i-j}})
\left(\sum_{k=0}^{p-1} (\bm{M}^{-1})_{jk} \int_0^1 e^{(1-r)h}  \bm{R}_k \, \mathrm{d}r \right)\\
&= \frac{\sigma_{t_{i + 1}}}{\sigma_{t_{i}}} \bm{x}_{t_i} + \alpha_{t+1} 
\sum_{j=0}^p h_{t_i}   
\bm{\hat x}_{\theta}(\bm{\hat x}_{\lambda_{t_{i-j}}},\lambda_{t_{i-j}})
\left(\sum_{j=0}^{p-1} (\bm{M}^{-1})_{jk} \tilde{\bm{R}}_k \right) \\
&= \frac{\sigma_{t_{i + 1}}}{\sigma_{t_{i}}} \bm{x}_{t_i} + \alpha_{t+1} 
\sum_{k=0}^p h_{t_i} \bm{a}_k 
\bm{\hat x}_{\theta}(\bm{\hat x}_{\lambda_{t_{i-j}}},\lambda_{t_{i-j}}),\\
\end{aligned}
\end{equation}
where 
\begin{equation*}
    \boldsymbol{a} = \bm{M}^{-1} \tilde{\bm{R}}, \quad
    \tilde{\bm{R}} := \begin{bmatrix} 
       \int_0^1 e^{(1-r)h}  \, \mathrm{d}r\\ 
       \int_0^1 e^{(1-r)h} r \, \mathrm{d}r\\ 
        \vdots \\ 
       \int_0^1 e^{(1-r)h} r^{p-1}\, \mathrm{d}r
    \end{bmatrix}.
\end{equation*}

The derivation from this equation to the final UniPC formulation follows the same process starting from Eq.~\eqref{eq:sampling_UniPC} as described in Appendix \ref{sec:UniPC}.

\subsection{2nd Order Coefficients}
\label{sec:UniPC 2nd Order Coefficients}

The second-order method in UniPC and SA-Solver can degenerate into a simple equation in which one or two coefficients are unknown. To prevent such degeneration, they determine the coefficients using a Taylor approximation with an accuracy of $\mathcal{O}(h^2)$.

\textbf{UniPC} The conditions for UniP-2 and UniC-1 degenerate into a simple equation where only a single $a_1$ is unknown.
For $B_1(h)=h$, $a_1=1/2$ satisfies the following condition:
\begin{equation}
    a_1B(h)-h\psi_2(h) = \frac{1}{2}h-\frac{1}{2}h + \mathcal{O}(h^2) = \mathcal{O}(h^2),
\end{equation}
where
\begin{equation}
    \psi_2(h) = \frac{h - 1 + e^{-h}}{h^2}.
\end{equation}
For $B_2(h)=e^h-1=h+\mathcal{O}(h^2)$, the derivation is similar. Therefore, UniPC directly set $a_1=1/2$ for UniP-2 and UniC-1 without solving the equation. 

\textbf{SA-Solver}
Similar to UniPC, the coefficients for the 2-step SA-Predictor and 1-step SA-Corrector degenerate into a simple case.
From Appendix D of the SA-Solver, the coefficients of the 2-step SA-Predictor and the 1-step SA-Corrector in Eq. (106) become simplified if we set $\tau = 0$ and $h=\lambda_{t_i} - \lambda_{t_{i-1}}$.
\begin{equation}
\label{eq:appendix_2nd_coeff}
\begin{aligned}
b_i + b_{i-1} &= \alpha_{t_{i+1}}(1 - e^{-h(1+\tau^2)})
\quad {\longrightarrow} \quad  
b_i + b_{i-1} = \alpha_{t_{i+1}} h \psi_1(h),\\
b_{i-1} &= \alpha_{t_{i+1}} \frac{e^{-(1+\tau^2)h} + (1+\tau^2)h - 1}{(1+\tau^2)(\lambda_{t_i} - \lambda_{t_{i-1}})} 
\quad {\longrightarrow} \quad  
b_{i-1} = \alpha_{t_{i+1}} h \psi_2(h),
\end{aligned}
\end{equation}
where
\begin{equation}
\begin{aligned}
\psi_1(h) = \frac{1 - e^{-h}}{h}, \quad \psi_2(h) = \frac{h - 1 + e^{-h}}{h^2}.
\end{aligned}
\end{equation}
Using a Taylor expansion of $e^{-h}$, $b_{i-1}$ can be approximated as follows:
\begin{equation}
\label{eq:2nd_coeff_b_i_1}
\begin{aligned}
\tilde{b}_{i-1} &= \alpha_{t_{i+1}} h  \frac{h - 1 + e^{-h}}{h^2}
= \alpha_{t_{i+1}} \frac{h - 1 + (1 - h + \frac{1}{2}h^2 + \mathcal{O}(h^3) )}{h}
= \alpha_{t_{i+1}} \frac{1}{2}h + \mathcal{O}(h^2).
\end{aligned}
\end{equation}
${b}_i$ can be approximated as:
\begin{equation}
\label{eq:2nd_coeff_b_i}
\begin{aligned}
\tilde{b}_i &= (b_i + b_{i-1}) - \tilde{b}_{i-1}
= \alpha_{t_{i+1}} (1 - e^{-h}) - \frac{1}{2} \alpha_{t_{i+1}} h + \mathcal{O}(h^2).
\end{aligned}
\end{equation}

Applying $\tilde{b}_i$ and $\tilde{b}_{i-1}$ to the update equation, the equation becomes identical to the case where $B_1(h) = h$ is used in UniP-2 and UniC-1:
\begin{equation}
\begin{aligned}
\bm{x}_{t_{i+1}} &= \frac{\sigma_{t_{i+1}}}{\sigma_{t_i}} x_{t_i} 
+ b_i \bm{\hat x}_{\theta}(\bm{\hat x}_{\lambda_{t_{i}}},\lambda_{t_{i}}) 
+ b_{i-1} \bm{\hat x}_{\theta}(\bm{\hat x}_{\lambda_{t_{i-1}}},\lambda_{t_{i-1}})  \\
&= \frac{\sigma_{t_{i+1}}}{\sigma_{t_i}} x_{t_i} + \alpha_{t_{i+1}} (1 - e^{-h}) \bm{\hat x}_{\theta}(\bm{\hat x}_{\lambda_{t_{i}}},\lambda_{t_{i}}) 
+ \frac{1}{2} \alpha_{t_{i+1}} h \left(\bm{\hat x}_{\theta}(\bm{\hat x}_{\lambda_{t_{i-1}}}) - \bm{\hat x}_{\theta}(\bm{\hat x}_{\lambda_{t_{i}}},\lambda_{t_{i}}) \right) + \mathcal{O}(h^2)
\end{aligned}
\end{equation}
\section{Analysis of Shape Parameter}
\label{app:Analysis of Shape Parameter}

Building on the ImageNet 128$\times$128~\cite{deng2009imagenet} experiments presented in Appendix~\ref{app:imagenet128}, this section clarifies how the predictor’s shape parameter (Section~\ref{sec:Learning of Shape Parameter}) relates to both the coefficients in Eq.~\eqref{eq:RBF-Solver Equation} and the sampling trajectory.

\paragraph{Shape Parameter and Coefficient Magnitude Ratio (CMR)}

Figure~\ref{fig:optimized shape parameters} plots the predictor and corrector shape parameters,
$\gamma_{t_i}^{\mathrm{pred}}$ and $\gamma_{t_i}^{\mathrm{corr}}$, optimized by the procedure in
Section~\ref{sec:Learning of Shape Parameter}; both curves are shown on a logarithmic scale,
i.e., as $\log\gamma_{t_i}$.
To quantify the relative importance of each coefficient, we introduce the
\emph{coefficient magnitude ratio} (CMR).
For a given time index~$t_i$ and coefficient position~$j$, the CMR is defined as the
absolute value of that coefficient divided by the $\ell_1$-sum of all coefficients, i.e.,

\begin{equation}
\label{eq:CMR-definition}
\begin{aligned}
\bigl[\mathrm{CMR}^{\mathrm{pred}}_{t_i}\bigr]_j
  &= \frac{\bigl|[\mathbf{c}^{(p)}_{t_i}]_j\bigr|}
         {\displaystyle\sum_{k=0}^{p-1}\bigl|[\mathbf{c}^{(p)}_{t_i}]_k\bigr|}
  , \quad 0 \le j < p, \\[6pt]
\bigl[\mathrm{CMR}^{\mathrm{corr}}_{t_i}\bigr]_j
  &= \frac{\bigl|[\mathbf{c}^{(p+1)}_{t_{i+1}}]_j\bigr|}
         {\displaystyle\sum_{k=0}^{p}\bigl|[\mathbf{c}^{(p+1)}_{t_{i+1}}]_k\bigr|}
  , \quad 0 \le j < p+1.
\end{aligned}
\end{equation}


Throughout the figures, CMRs are displayed for guidance scales
\(\{2.0, 4.0, 6.0, 8.0\}\) and numbers of function evaluations (NFE)
\(\{5, 10, 15\}\).

\paragraph{Shape Parameter and Sampling Trajectory}
Figure~\ref{fig:mean_deltas} compares four quantities
(the shape parameter and the zeroth-index CMR correspond to the $\mathrm{NFE}=15$ results in
Figure~\ref{fig:optimized shape parameters}):
\begin{itemize}
    \item[(a)] The predictor‐side shape parameter
               $\gamma_{t_i}^{\mathrm{pred}}$ (Figure~\ref{fig:shape_param});
    \item[(b)] The predictor‐side zeroth‐index CMR $\bigl[\mathrm{CMR}^{\mathrm{pred}}_{t_i}\bigr]_0$ (Figure~\ref{fig:coeff0});
    \item[(c)] The absolute change in successive model evaluations
               $\lvert\Delta\bm{x}_{\theta}(\bm{x}_{t_i}, t_i)\rvert$
               (Figure~\ref{fig:mean_x0});
    \item[(d)] The absolute change in successive intermediate samples
               $\lvert\Delta\bm{x}_{t_i}\rvert$
               (Figure~\ref{fig:mean_xt}).
\end{itemize}

These differences, obtained from the UniPC~\cite{zhao2023unipc} run with $\mathrm{NFE}=200$, are defined as

\begin{equation*}
\begin{aligned}
\left|\Delta\bm{x}_{\theta}\bigl(\bm{x}_{t_i},\, t_i\bigr)\right|
  &=\bigl|
      \bm{x}_{\theta}\bigl(\bm{x}_{t_i},\, t_i\bigr)
      -\bm{x}_{\theta}\bigl(\bm{x}_{t_{i-1}},\, t_{i-1}\bigr)
    \bigr|,\\[4pt]
\left|\Delta\bm{x}_{t_i}\right|
  &=\bigl|\bm{x}_{t_i}-\bm{x}_{t_{i-1}}\bigr|,
  \qquad 0<i< M=200.
\end{aligned}
\end{equation*}

As Figure~\ref{fig:mean_xt} illustrates, the variation in the sampling trajectory grows as
the guidance scale increases from~2.0 to~8.0. Correspondingly,
Figure~\ref{fig:mean_x0} shows that the model evaluations change more abruptly at larger
scales. This tendency is especially pronounced during the early steps, where the sampler
captures coarse structural changes rather than fine textures~\cite{ho2020denoising}. To interpolate such rapidly
diverging model evaluations, the shape parameters are learned
(Figure~\ref{fig:shape_param}), which in turn adjust the CMRs
(Figure~\ref{fig:coeff0}). Notably, the zeroth-index CMR rises, indicating a
stronger dependency on the most recent information.

\begin{figure*}[t]
  \centering
  \begin{minipage}[t]{0.49\linewidth}
    \centering
    \includegraphics[width=\linewidth]{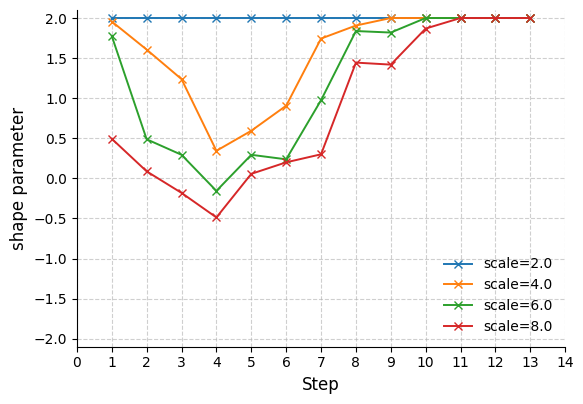}
    \subcaption{Predictor shape parameter $\gamma_{t_i}^{\mathrm{pred}}$}
    \label{fig:shape_param}
  \end{minipage}\hfill
  \begin{minipage}[t]{0.49\linewidth}
    \centering
    \includegraphics[width=\linewidth]{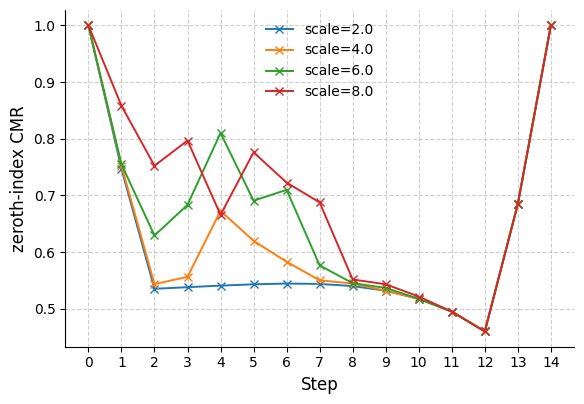}
    \subcaption{Predictor zeroth-index CMR $\bigl[\mathrm{CMR}^{\mathrm{pred}}_{t_i}\bigr]_0$}
    \label{fig:coeff0}
  \end{minipage}

  \vspace{0.8em} 

  \begin{minipage}[t]{0.49\linewidth}
    \centering
    \includegraphics[width=\linewidth]{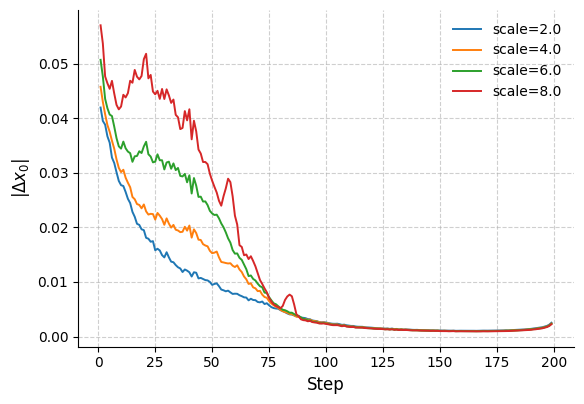}
    \subcaption{Difference of model evaluations $\lvert \Delta\bm{x}_{\theta}(\bm{x}_{t_i}, t_i)\rvert$}
    \label{fig:mean_x0}
  \end{minipage}\hfill
  \begin{minipage}[t]{0.49\linewidth}
    \centering
    \includegraphics[width=\linewidth]{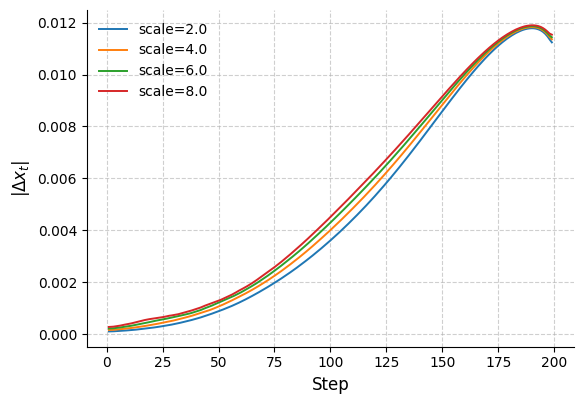}
    \subcaption{Difference of intermediate samples $\lvert \Delta\bm{x}_{t_i}\rvert$}
    \label{fig:mean_xt}
  \end{minipage}

  \caption{Evolution of key quantities across guidance scales 2.0, 4.0, 6.0, and 8.0.
Variations in the predictor’s shape parameter modulate the magnitude of its zeroth-index coefficient, which in turn helps interpolate the rapidly changing trajectory of the model evaluations.}
  \label{fig:mean_deltas}
\end{figure*}

\begin{figure*}[h]
  \centering
  \setlength{\tabcolsep}{0pt}
  \renewcommand{\arraystretch}{0}

  \begin{tabular}{c}
    {\scriptsize\textbf{Optimized shape parameters and coefficient weights (guidance scale = 2.0)}}\\[4pt]
    \includegraphics[width=\textwidth]{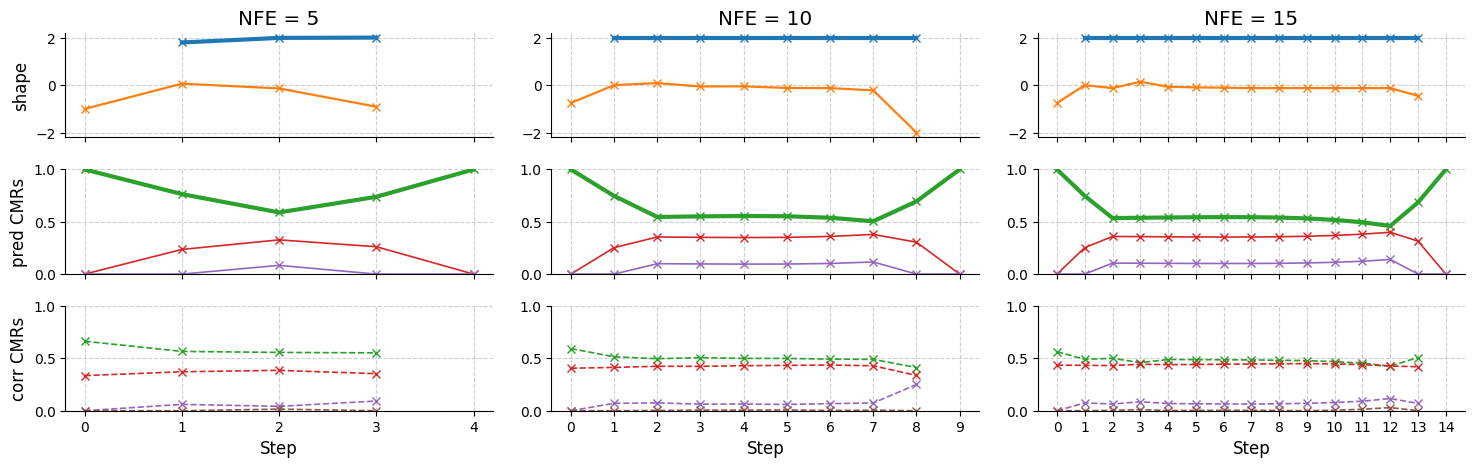}\\[6pt]

    {\scriptsize\textbf{Optimized shape parameters and coefficient weights (guidance scale = 4.0)}}\\[4pt]
    \includegraphics[width=\textwidth]{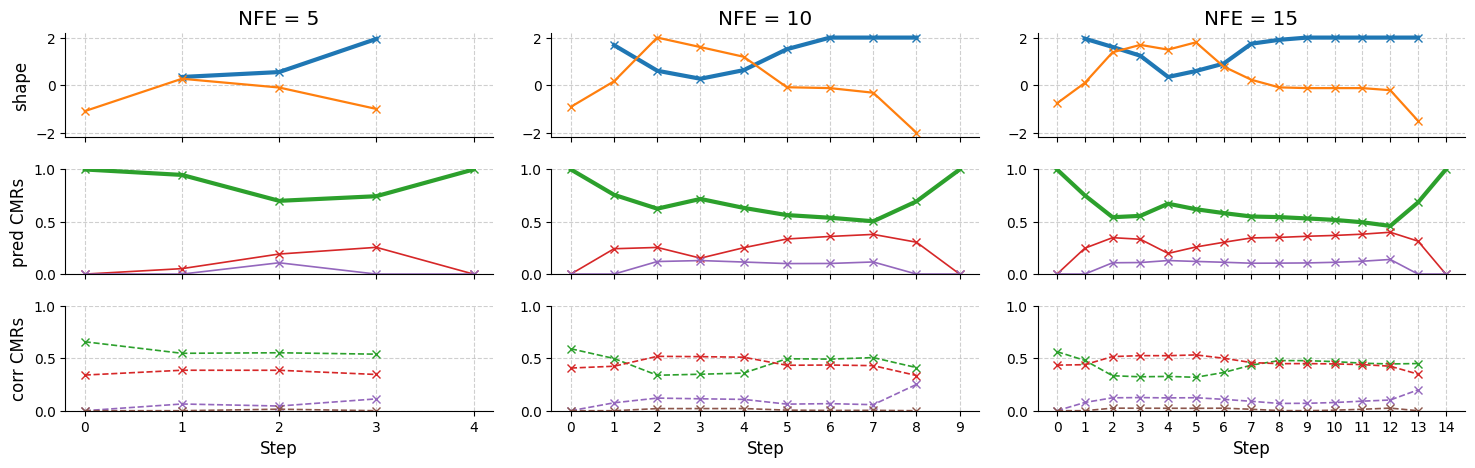}\\[6pt]

    {\scriptsize\textbf{Optimized shape parameters and coefficient weights (guidance scale = 6.0)}}\\[4pt]
    \includegraphics[width=\textwidth]{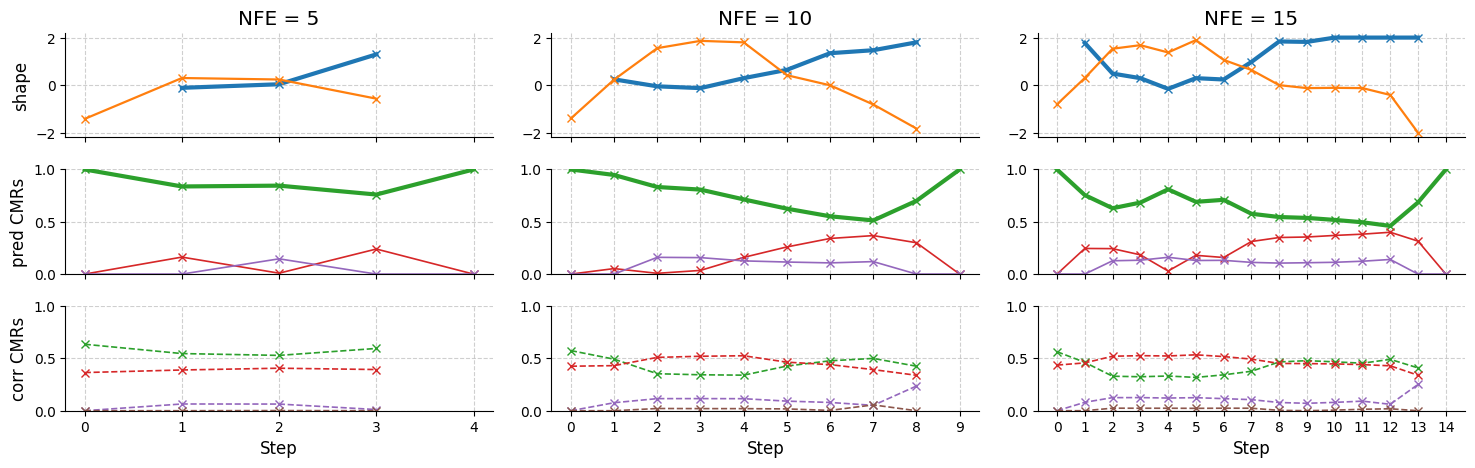}\\[6pt]

    {\scriptsize\textbf{Optimized shape parameters and coefficient weights (guidance scale = 8.0)}}\\[4pt]
    \includegraphics[width=\textwidth]{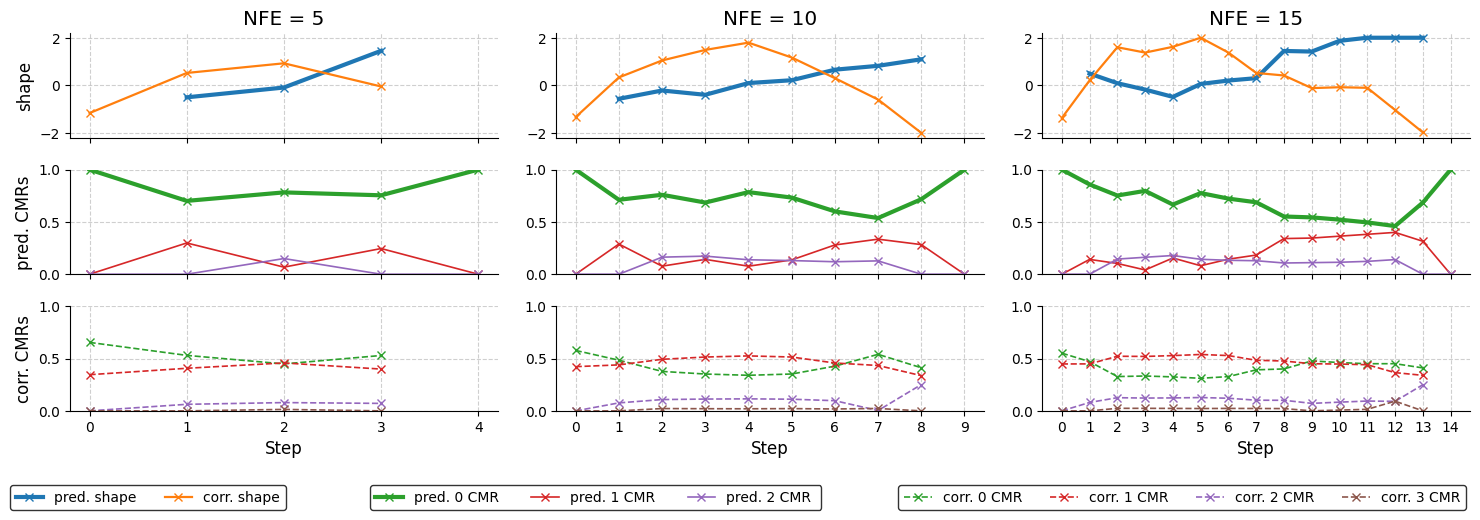}
  \end{tabular}

  \caption{Optimized predictor- and corrector-side shape parameters and their corresponding coefficient-magnitude ratios (CMRs).
Rows represent guidance scales 2.0, 4.0, 6.0, and 8.0 (top → bottom); columns represent $\mathrm{NFE}=$ 5, 10, 15.
Shape parameters are plotted on a logarithmic scale (blue = predictor, yellow = corrector).
CMRs are color-coded as follows: green (0-index), red (1-index), purple (2-index), and brown (3-index).}
  \label{fig:optimized shape parameters}
\end{figure*}

From this figure, we observe that the shape parameters—and hence the resulting coefficients—adapt to variations in the data. These results illustrate the Gaussian’s flexibility and the associated performance improvements. Similar outperforming experiments demonstrating this flexibility have been conducted in \cite{Dyn2007GaussianSubdivision, Yoon2022WENO_RBF} for subdivision schemes and hyperbolic conservation laws, respectively.

\clearpage
\section{Additional Ablation Study}
\label{app:additional_ablation_study}
This section discusses ablation studies performed in addition to those presented in the main text.
\subsection{Model Trajectory vs. Target Trajectory}
\label{app:intermediate_target_ablation}
We conduct an ablation study on how to choose the intermediate target
\(\bm{x}_{t_i}^{\mathrm{target}}\) when optimizing the shape parameter.
Following prior work~\cite{dpm_solver_v3,dc_solver}, one option is to run an
existing sampler—such as DDIM~\cite{song2020denoising} or
DPM-Solver++~\cite{lu2022dpmpp}—with a high number of function evaluations
(\(\text{NFE}=200\) or \(1000\)) and record the resulting trajectory from
model evaluations.  
Because this trajectory is produced entirely by model evaluations, we refer
to it as the \emph{model trajectory}. Alternatively, starting from the same endpoints—the noise
\(\bm{x}_{T}^{\mathrm{target}}\) and the image \(\bm{x}_{0}^{\mathrm{target}}\)—we can obtain the intermediate
target by running the forward diffusion process:
\begin{equation*}
    \bm{x}_{t_i}^{\mathrm{target}}
    = \alpha_{t_i}\,\bm{x}_{0}^{\mathrm{target}}
    + \sigma_{t_i}\,\bm{x}_{T}^{\mathrm{target}}.
    \label{eq:target_trajectory}
\end{equation*}
Since this trajectory is computed directly from the target noise and image, we call it the \emph{target trajectory}. To decide which approach yields better learning of the shape parameter, we conduct an experiment. Using DPM-Solver++ with \(\text{NFE}=200\), we obtain both the model trajectory and the target trajectory for 128 image--noise pairs and take each in turn as the intermediate target.
Table~\ref{tab:fid_model_vs_target} reports the FID scores obtained when the
shape parameter is trained using each type of intermediate target.  Across
almost all NFE settings, the \textbf{target trajectory} yields
consistently and substantially better FID than the \textbf{model
trajectory}.

\begin{table}[h]
    \centering
    \caption{Comparison of model and target trajectory. Sample quality is measured by FID $\downarrow$ on the CIFAR-10 32×32 dataset (Score-SDE), evaluated on 50k samples.}

    \label{tab:fid_model_vs_target}
    \renewcommand{\arraystretch}{1.2}
    \setlength{\tabcolsep}{5.0pt}
    \begin{tabular}{lccccccccccc}
        \toprule
        \multirow{2}{*}{\textbf{Series}} & \multicolumn{11}{c}{\textbf{NFE}} \\
        \cmidrule(lr){2-12}
        & \textbf{5} & \textbf{6} & \textbf{8} & \textbf{10} & \textbf{12} & \textbf{15} & \textbf{20} & \textbf{25} & \textbf{30} & \textbf{35} & \textbf{40} \\
        \midrule
        Model Trajectory  & 27.37 & \textbf{11.39} & 5.97 & 4.68 & 4.76 & 3.54 & 2.91 & 2.73 & 2.65 & 2.61 & 2.61 \\
        \addlinespace[0.5ex]
        \rowcolor{gray!20}
        Target Trajectory & \textbf{26.65} & 13.05 & \textbf{5.32} & \textbf{4.13} & \textbf{3.84} & \textbf{3.02} & \textbf{2.66} & \textbf{2.56} & \textbf{2.51} & \textbf{2.50} & \textbf{2.49} \\
        \bottomrule
    \end{tabular}
\end{table}

\subsection{Shape Parameter Range}
\label{app:shape_parameter_range}
As described in Section~\ref{sec:Learning of Shape Parameter}, we learn the shape
parameter $\gamma$ in log space. Empirically, the optimization frequently drives
$\gamma \!\to\! \infty$. When $\gamma$ diverges, the solver converges to the
Adams method, as proven in Section~\ref{sec:Lagrange Convergence}. For numerical
stability, we therefore switch to the Adams method whenever
$\log\gamma$ exceeds a chosen threshold.

To identify the best range, we perform an ablation study.
Tables~\ref{tab:fid_shape_range_cifar10} and~\ref{tab:fid_shape_range_imagenet}
evaluate three search intervals for $\log\gamma$: $[-1,1]$, $[-2,2]$,
and $[-3,3]$.  In each case, we fall back to the Adams method if
$\log\gamma$ rises above the interval’s upper bound (i.e., $1$, $2$, or
$3$, respectively).  On CIFAR-10 (Score-SDE), the early switch with
$[-1,1]$ gives the best results, whereas on ImageNet $64\times64$ the
wider intervals $[-2,2]$ and $[-3,3]$ are optimal at different NFEs.
Since the best threshold varies slightly across datasets, we
adopt the balanced range $[-2,2]$ as our default setting.

\begin{table}[h]
  \centering
    \caption{Shape parameter range ablation study. Sample quality is measured by FID $\downarrow$ on CIFAR-10 32$\times$32 dataset (Score-SDE), evaluated on 50k samples.}
  \label{tab:fid_shape_range_cifar10}
  \renewcommand{\arraystretch}{1.2}
  \setlength{\tabcolsep}{5pt}
  \begin{tabularx}{\textwidth}{lYYYYYYYYYYY}
    \toprule
    \multirow{2}{*}{\textbf{Series}}
      & \multicolumn{11}{c}{\textbf{NFE}} \\ \cmidrule(lr){2-12}
      & \textbf{5}&\textbf{6}&\textbf{8}&\textbf{10}&\textbf{12}&\textbf{15}
      &\textbf{20}&\textbf{25}&\textbf{30}&\textbf{35}&\textbf{40}\\
    \midrule
    $[-1,\,1]$ &
      \textbf{26.34}&16.34&\textbf{5.28}&\textbf{4.11}&\textbf{3.80}&\textbf{3.01}&%
      \textbf{2.66}&\textbf{2.56}&\textbf{2.51}&\textbf{2.50}&\textbf{2.49}\\
    \rowcolor{gray!20}
    $[-2,\,2]$ &
      26.65&13.05&5.32&4.13&3.84&3.02&%
      \textbf{2.66}&\textbf{2.56}&\textbf{2.51}&\textbf{2.50}&\textbf{2.49}\\
    $[-3,\,3]$ &
      26.64&\textbf{12.45}&5.36&4.16&3.92&3.02&%
      \textbf{2.66}&\textbf{2.56}&\textbf{2.51}&\textbf{2.50}&2.50\\
    \bottomrule
  \end{tabularx}
\end{table}

\begin{table}[h]
  \centering
    \caption{Shape parameter range ablation study. Sample quality is measured by FID $\downarrow$ on ImageNet 64$\times$64 dataset (Improved-Diffusion), evaluated on 50k samples.}
  \label{tab:fid_shape_range_imagenet}
  \renewcommand{\arraystretch}{1.2}
  \setlength{\tabcolsep}{5pt}
  \begin{tabularx}{\textwidth}{lYYYYYYYYYYY}
    \toprule
    \multirow{2}{*}{\textbf{Series}}
      & \multicolumn{11}{c}{\textbf{NFE}} \\ \cmidrule(lr){2-12}
      & \textbf{5}&\textbf{6}&\textbf{8}&\textbf{10}&\textbf{12}&\textbf{15}
      &\textbf{20}&\textbf{25}&\textbf{30}&\textbf{35}&\textbf{40}\\
    \midrule
    $[-1,\,1]$ &
      116.69&\textbf{70.58}&35.68&26.57&22.37&20.89&%
      19.30&18.73&18.33&18.18&18.09\\
    \rowcolor{gray!20}
    $[-2,\,2]$ &
      \textbf{116.55}&70.68&\textbf{34.15}&\textbf{26.47}&22.72&20.89&%
      19.30&\textbf{18.72}&18.31&18.17&\textbf{18.07}\\
    $[-3,\,3]$ &
      116.75&70.60&35.13&26.64&\textbf{22.36}&\textbf{20.88}&%
      \textbf{19.28}&18.73&\textbf{18.30}&\textbf{18.15}&\textbf{18.07}\\
    \bottomrule
  \end{tabularx}
\end{table}

\section{Runtime Analysis of Shape Parameter Optimization}
\label{app:elapsed_time}

In this section, we measure the wall-clock time required to optimize the shape
parameters for sampling on
CIFAR-10~\cite{chrabaszcz2017downsampled},
ImageNet 64$\times$64~\cite{chrabaszcz2017downsampled},
and ImageNet 128$\times$128 and 256$\times$256~\cite{deng2009imagenet}.

The overall runtime consists of two stages:
\begin{enumerate}[nosep,label=(\arabic*)]
    \item \textbf{Target-set preparation}: generation of \emph{(noise, data)} target
          pairs with an existing sampler;
    \item \textbf{Per-NFE shape parameter optimization} with the proposed
          RBF-Solver.
\end{enumerate}

The target set is created with UniPC~\cite{zhao2023unipc}
using \(\text{NFE}=200\).
For each NFE, we perform a grid search over
\(1{,}089\;(\gamma^{\mathrm{pred}},\gamma^{\mathrm{corr}})\) pairs obtained by
uniformly dividing the range \([-2,2]\) into 33 values for each parameter.

Three variables control the amount of data processed during optimization:
\begin{itemize}
    \item \(N\): number of target pairs in the target set,
    \item \(M\): batch size used during optimization,
    \item \(K\): number of independent optimization runs.
\end{itemize}

We fix $N=128$ for all datasets. 
For CIFAR-10 and ImageNet~$64\times64$, we use $(M,K)=(128,1)$.  
For ImageNet~$128\times128$ and $256\times256$, we set $(M,K)=(16,20)$ owing to GPU-memory limitations.  
For each chosen NFE, the sampler requires a vector of
\(2(\text{NFE}-1)\) shape parameters: 
\(\gamma^{\mathrm{pred}}_{t_i}\) for \(i=1,\dots,\text{NFE}-1\)
(the initial predictor parameter \(\gamma^{\mathrm{pred}}_{t_0}\) is unused)
and \(\gamma^{\mathrm{corr}}_{t_i}\) for \(i=0,\dots,\text{NFE}-2\)
(the final corrector parameter at \(t_{\text{NFE}}\) is not needed).
We optimize this vector \(K\) times, obtain \(K\) independent estimates,
and then average them element-wise to obtain the final
shape-parameter vector.
The same \(N\) and \(M\) settings apply when constructing the target set.
To obtain \(N = 128\) target pairs, CIFAR-10 and ImageNet~$64\times64$ require
one pass with \(M = 128\),
whereas ImageNet~$128\times128$ and $256\times256$ need eight passes with
\(M = 16\) each \((16\times 8 = 128)\).
All timings are measured on a single NVIDIA RTX 4090 GPU and an Intel Core i7-12700H CPU.

\begin{table}[H]
  \centering
  \caption{Runtime per NFE for shape-parameter optimization. 
Each per-NFE value is the time required to optimize one batch of size $M$. 
The \textbf{Total} row sums all per-NFE runtimes, multiplies the result by $K$, and then adds the target-set runtime.
All values are reported as mean$\;\pm\;$std over 10 runs.
}
  \label{tab:runtime_multi_dataset}
  \setlength{\tabcolsep}{6pt}
  \renewcommand{\arraystretch}{1.15}
  \resizebox{\linewidth}{!}{%
  \begin{tabular}{lcccc}
    \toprule
    \textbf{NFE} &
    \textbf{CIFAR-10} &
    \textbf{ImageNet 64$\times$64} &
    \textbf{ImageNet 128$\times$128} &
    \textbf{ImageNet 256$\times$256} \\
    \midrule
     5  & $1.916\pm0.032$ & $4.846\pm0.006$ & $4.356\pm0.046$  & $7.432\pm0.024$ \\
     6  & $2.361\pm0.032$ & $6.002\pm0.026$ & $5.346\pm0.031$  & $9.101\pm0.025$ \\
     8  & $3.333\pm0.028$ & $8.333\pm0.019$ & $7.384\pm0.041$  & $12.491\pm0.024$ \\
    10  & $4.207\pm0.029$ & $10.692\pm0.012$ & $9.461\pm0.031$ & $15.889\pm0.027$ \\
    12  & $5.205\pm0.031$ & $12.968\pm0.009$ & $11.465\pm0.065$ & $19.275\pm0.033$ \\
    15  & $6.524\pm0.033$ & $16.504\pm0.009$ & $14.566\pm0.067$ & $24.349\pm0.041$ \\
    20  & $8.902\pm0.033$ & $22.236\pm0.008$ & $19.712\pm0.067$ & $32.810\pm0.069$ \\
    \midrule
    Target Set & $26.365\pm0.028$ & $57.905\pm0.019$ & $312.700\pm0.714$ & $1139.204\pm1.290$ \\
    \textbf{Total} &
      \textbf{58.8~s} &
      \textbf{2~min~19.5~s} &
      \textbf{29~min~18.5~s} &
      \textbf{59~min~26.1~s} \\
    \bottomrule
  \end{tabular}}
\end{table}

\clearpage
\section{License}
\label{app:license}

We list the used datasets and code licenses in Table~\ref{tab:license}.


\begin{table}[h]                
  \centering
  \footnotesize                    
  \setlength{\tabcolsep}{9pt}
  \renewcommand{\arraystretch}{1.15}
  \caption{Datasets, codebases, and their respective licenses.}
  \label{tab:license}
  \vspace{0.3em}

  \begin{adjustbox}{width=\textwidth}  
  \begin{tabular}{@{}llll@{}}
    \toprule
    \textbf{Name} & \textbf{URL} & \textbf{Citation} & \textbf{License}\\
    \midrule
    \multicolumn{4}{l}{\textit{Datasets}}\\
    \cmidrule(lr){1-4}
    CIFAR-10         & \url{https://www.cs.toronto.edu/~kriz/cifar.html} & \cite{Krizhevsky09learningmultiple} & N/A\\
    ImageNet         & \url{https://www.image-net.org}                  & \cite{deng2009imagenet}             & ImageNet Terms\\
    MS-COCO~2014     & \url{https://cocodataset.org}                    & \cite{lin2014microsoft}             & CC~BY~4.0\\
    \midrule
    \multicolumn{4}{l}{\textit{Code repositories}}\\
    \cmidrule(lr){1-4}
    Score-SDE        & \url{https://github.com/yang-song/score_sde}      & \cite{song2021score}        & Apache 2.0\\
    EDM              & \url{https://github.com/NVlabs/edm}               & \cite{karras2022elucidating} & CC BY-NC-SA 4.0\\
    Guided-Diffusion & \url{https://github.com/openai/guided-diffusion}  & \cite{dhariwal2021diffusion} & MIT\\
    Latent-Diffusion & \url{https://github.com/CompVis/latent-diffusion} & \cite{rombach2022high}      & MIT\\
    Improved-Diffusion & \url{https://github.com/openai/improved-diffusion} & \cite{nichol2021improved} & MIT\\
    Stable-Diffusion & \url{https://github.com/CompVis/stable-diffusion} & \cite{rombach2022high}      & CreativeML Open RAIL-M\\
    DPM-Solver++     & \url{https://github.com/LuChengTHU/dpm-solver}    & \cite{lu2022dpmpp}          & MIT\\
    UniPC            & \url{https://github.com/wl-zhao/UniPC}            & \cite{zhao2023unipc}        & MIT\\
    DPM-Solver-v3    & \url{https://github.com/thu-ml/DPM-Solver-v3}     & \cite{dpm_solver_v3}        & MIT\\
    DC-Solver        & \url{https://github.com/wl-zhao/DC-Solver}        & \cite{dc_solver}            & Apache 2.0\\
    AMED-Solver      & \url{https://github.com/zju-pi/diff-sampler}      & \cite{amed_solver}          & Apache 2.0\\
    \bottomrule
  \end{tabular}
  \end{adjustbox}
\end{table}

\section{Additional Main Results}
\label{app:Additional_Main_Results}
In this section, we provide additional information on the main experiments described in Section \ref{sec:Experiments}, including the experimental environments, the models used, the sampler configurations, the exact numerical results obtained, and detailed analyses of those results. We report the GPU models and card counts because identical seeds can yield different noise across GPU models and in multi-GPU runs.

\paragraph{Common Sampler Settings} Across all experiments, we employ data-prediction mode with multistep samplers: unconditional runs use third-order updates, conditional runs use second-order updates, and the \texttt{lower\_order\_final} flag automatically switches to a lower order when only a few steps remain.

\begin{table}[h]
  \setlength{\tabcolsep}{9pt}
  \renewcommand{\arraystretch}{1.15}
  \centering
  \begin{tabular}{llcl}
    \toprule
    Dataset & Model & \# Eval. samples & Section \\
    \midrule
    \multicolumn{4}{l}{\textbf{Unconditional}} \\[2pt]
    CIFAR-10 32$\times$32~\cite{Krizhevsky09learningmultiple} & Score-SDE~\cite{song2021score}          & 50k & Appendix~\ref{app:cifar10_scoresdes} \\
    CIFAR-10 32$\times$32~\cite{Krizhevsky09learningmultiple} & EDM~\cite{karras2022elucidating}        & 50k & Appendix~\ref{app:cifar10_edm}        \\
    ImageNet 64$\times$64~\cite{chrabaszcz2017downsampled}    & Improved-Diffusion~\cite{nichol2021improved} & 50k & Appendix~\ref{app:imagenet64_improved} \\
    \midrule
    \multicolumn{4}{l}{\textbf{Conditional}} \\[2pt]
    ImageNet 128$\times$128~\cite{deng2009imagenet} & Guided-Diffusion~\cite{dhariwal2021diffusion} & 50k & Appendix~\ref{app:imagenet128} \\
    ImageNet 256$\times$256~\cite{deng2009imagenet} & Guided-Diffusion~\cite{dhariwal2021diffusion} & 10k & Appendix~\ref{app:imagenet256} \\
    LAION-5B~\cite{schuhmann2022laion}              & Stable-Diffusion v1.4~\cite{rombach2022high} & 10k & Appendix~\ref{app:stable_diffusion} \\
    \bottomrule
  \end{tabular}
  \caption{Summary of experimental settings used in Additional Main Results}
\end{table}

\subsection{Unconditional Models}
\label{app:Additional Experiment Results - Unconditional Models}

\subsubsection{CIFAR-10 (Score-SDE) Experiment Results}
\label{app:cifar10_scoresdes}

\paragraph{Environment Setup}
\begin{itemize}
    \item \textbf{GPU}: Single NVIDIA RTX~4090
    \item \textbf{Model}: Score-SDE~\cite{song2021score}, \url{https://github.com/yang-song/score_sde}
    \item \textbf{Checkpoint}: \texttt{vp/cifar10\_ddpmpp\_deep\_continuous}
\end{itemize}

\paragraph{Sampler Settings}
\begin{itemize}
    \item \textbf{UniPC}: employs $B_1(h)$, which yields the best results in unconditional generation, as reported in \cite{zhao2023unipc}.
    \item \textbf{RBF-Solver}: The shape parameters are optimized using a target set of 128 samples generated by running UniPC with NFE$=$200. The solver is also evaluated with order$=$4.
\end{itemize}

\paragraph{Experiment Results}
As shown in Table~\ref{tab:fid_cifar10_score_sde}, RBF-Solver begins to outperform the competing methods once the number of function evaluations~(NFE) exceeds~10, and it maintains a consistent FID margin up to NFE~40.  
For RBF-Solver with order$=$4, the performance is the best among all samplers at NFE$=$5, and at NFE$=$6 it still surpasses the order$=$3 variant. This trend is consistently observed on the CIFAR-10 (EDM) experiment~\ref{app:cifar10_edm}.

\begin{table}[h]
    \centering
    \caption{Sample quality measured by FID $\downarrow$ on CIFAR-10 32$\times$32 dataset (Score-SDE), evaluated on 50k samples. Numbers in parentheses indicate the solver order.}
    \label{tab:fid_cifar10_score_sde}
    \renewcommand{\arraystretch}{1.2}
    \setlength{\tabcolsep}{5.0pt}
    \begin{tabular}{lccccccccccc}
        \toprule
        \multirow{2}{*}{\textbf{Model}} & \multicolumn{11}{c}{\textbf{NFE}}\\
        \cmidrule(lr){2-12}
        & \textbf{5} & \textbf{6} & \textbf{8} & \textbf{10} & \textbf{12} & \textbf{15} & \textbf{20} & \textbf{25} & \textbf{30} & \textbf{35} & \textbf{40} \\
        \midrule
        DPM-Solver++ (3)      & 28.50 & 13.46 & 5.32 & 4.01 & 4.04 & 3.32 & 2.89 & 2.75 & 2.68 & 2.65 & 2.62 \\
        UniPC--$B_1(h)$ (3)   & 23.69 & \textbf{10.38} & \textbf{5.15} & \textbf{3.96} & 3.93 & 3.05 & 2.73 & 2.64 & 2.60 & 2.59 & 2.57 \\
        \addlinespace[0.5ex]
        \rowcolor{gray!20}
        RBF-Solver (3)        & 26.65 & 13.05 & 5.32 & 4.13 & \textbf{3.84} & 3.02 & 2.66 & 2.56 & 2.51 & 2.50 & 2.49 \\
        \rowcolor{gray!20}
        RBF-Solver (4)        & \textbf{20.17} & 12.45 & 5.71 & 4.26 & 3.90 & \textbf{2.87} & \textbf{2.60} & \textbf{2.54} & \textbf{2.49} & \textbf{2.49} & \textbf{2.48} \\
        \bottomrule
    \end{tabular}
\end{table}

\begin{table*}[p]
  \centering

  \caption{Qualitative comparison of different samplers on CIFAR-10 (Score-SDE). Each image is randomly sampled and solver order 3. Columns correspond to samplers; rows indicate NFE values of 5, 10, and 15.}
  
  \label{tab:scoresde_samples}
  \renewcommand{\arraystretch}{1.02}
  \setlength{\tabcolsep}{2pt}

  \begin{adjustbox}{max width=\textwidth}
  \begin{tabular}{@{}C C C@{}}
    \toprule
    \textbf{DPM-Solver++} & \textbf{UniPC} & \textbf{RBF-Solver} \\
    \midrule
    \multicolumn{3}{c}{\textbf{NFE = 5}} \\[2pt]
    \includegraphics[width=\linewidth]{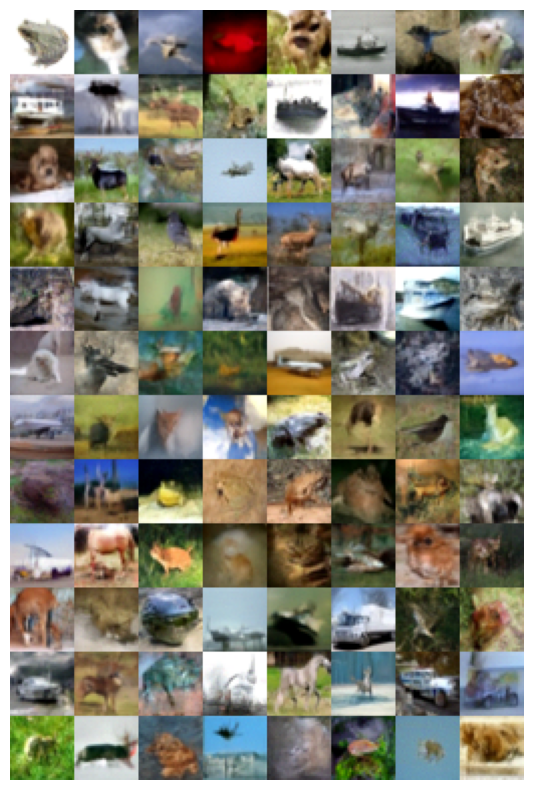} &
    \includegraphics[width=\linewidth]{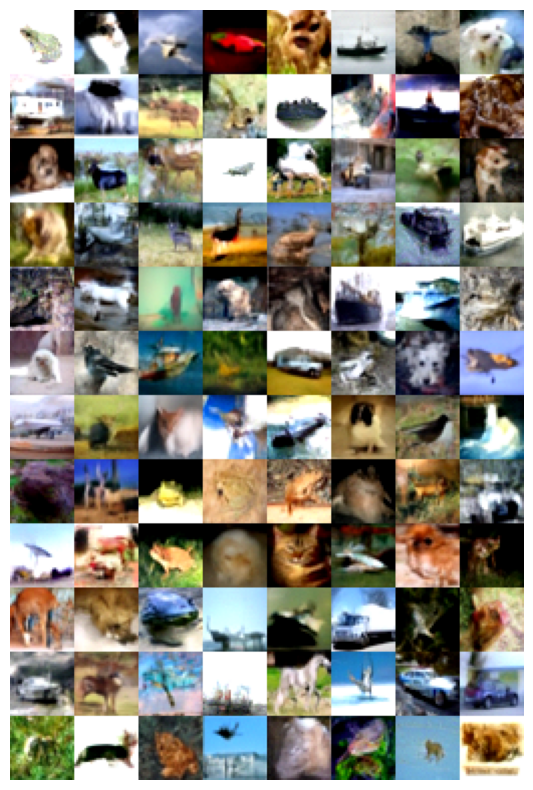} &
    \includegraphics[width=\linewidth]{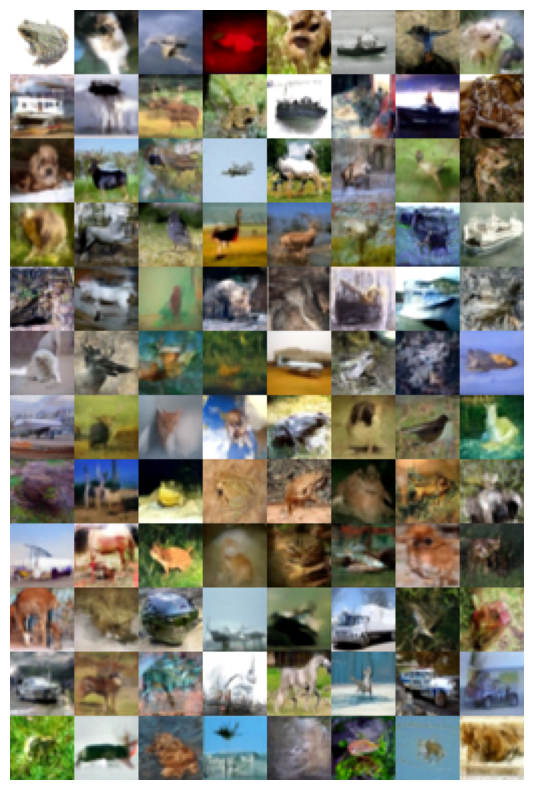} \\
    \midrule
    \multicolumn{3}{c}{\textbf{NFE = 10}} \\[2pt]
    \includegraphics[width=\linewidth]{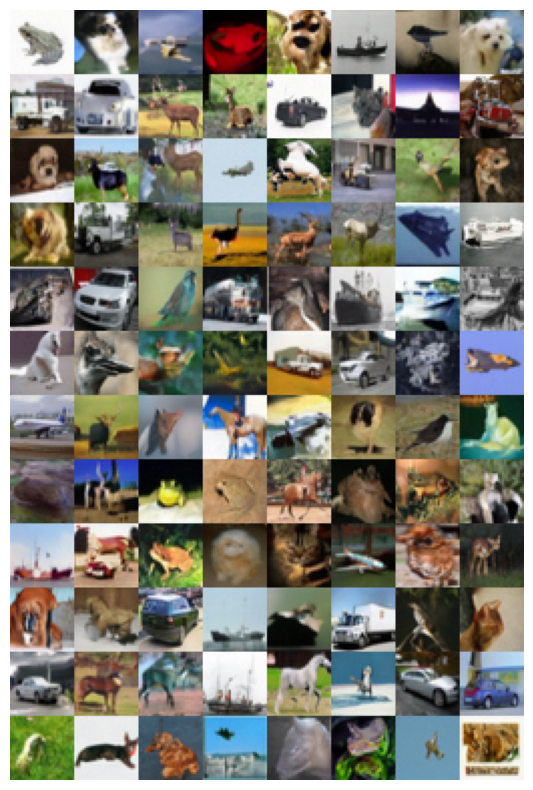} &
    \includegraphics[width=\linewidth]{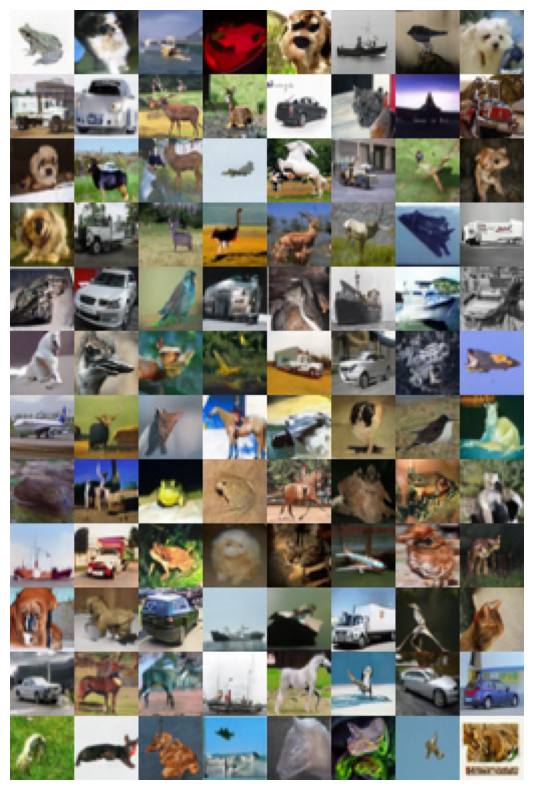} &
    \includegraphics[width=\linewidth]{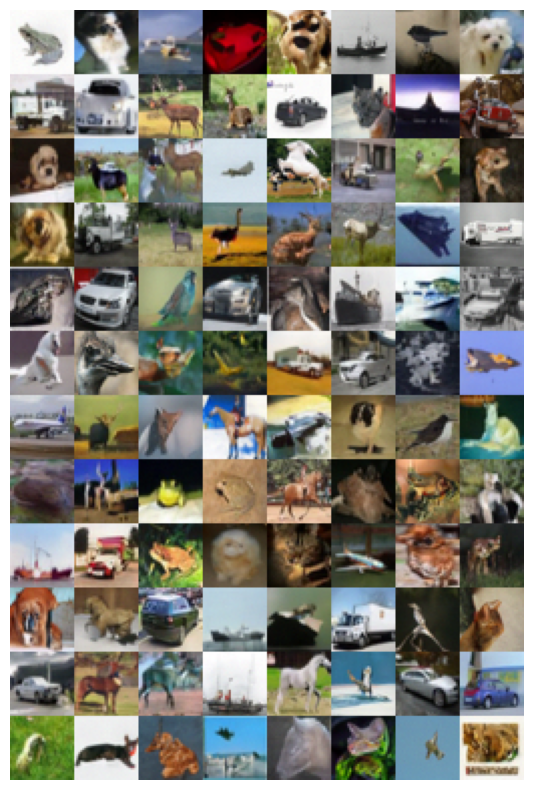} \\
    \midrule
    \multicolumn{3}{c}{\textbf{NFE = 15}} \\[2pt]
    \includegraphics[width=\linewidth]{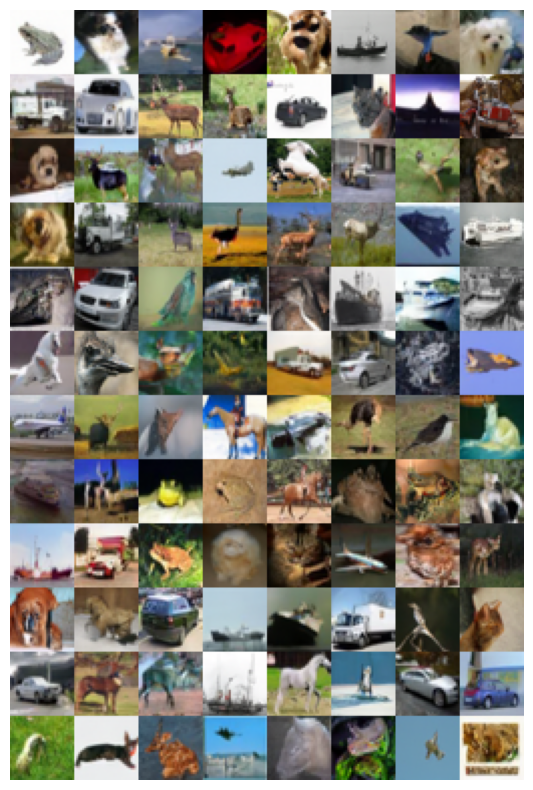} &
    \includegraphics[width=\linewidth]{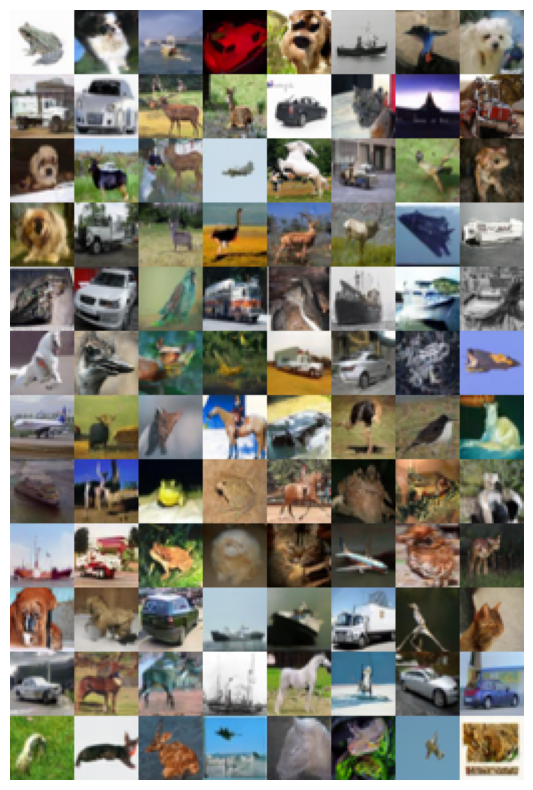} &
    \includegraphics[width=\linewidth]{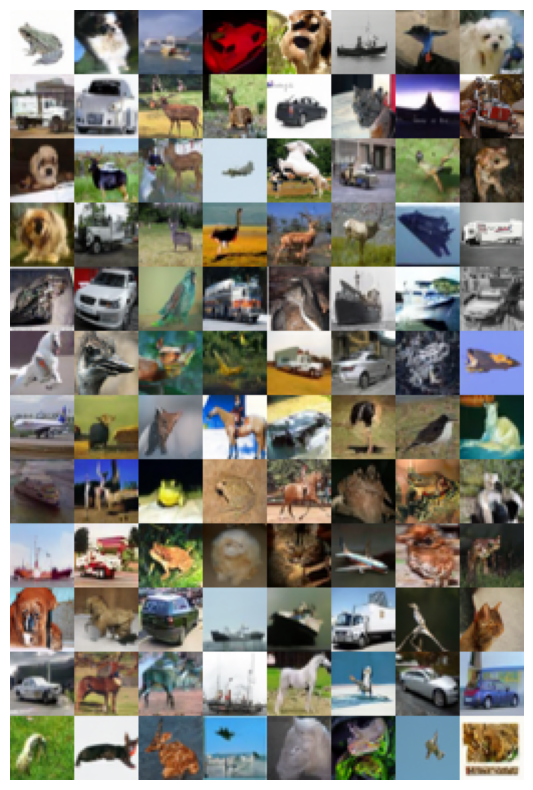} \\
    \bottomrule
  \end{tabular}
  \end{adjustbox}
\end{table*}

\subsubsection{CIFAR-10 (EDM) Experiment Results}
\label{app:cifar10_edm}

\paragraph{Environment Setup}
\begin{itemize}
    \item \textbf{GPU}: Single NVIDIA RTX~4090
    \item \textbf{Model}: EDM~\cite{karras2022elucidating}, \url{https://github.com/NVlabs/edm}
    \item \textbf{Checkpoint}: \texttt{cifar10‑32x32‑uncond‑vp}
\end{itemize}

\paragraph{Sampler Settings}
\begin{itemize}
    \item \textbf{UniPC}: uses $B_1(h)$ and $B_2(h)$.
    \item \textbf{RBF-Solver}: The shape parameters are optimized using a target set of 128 samples generated by running UniPC with NFE$=$200. The solver is also evaluated with order$=$4.
\end{itemize}

\paragraph{Experiment Results} As shown in Table~\ref{tab:fid_cifar10_edm}, RBF‐Solver with order$=$3 achieves marginally better FID scores than UniPC for all NFEs except 10 and 12. The order$=$4 variant provides a sizeable performance gain over order$=$3 at NFE 5 and 6, but offers no clear advantage at higher NFEs.

\begin{table}[h]
    \centering
    \caption{Sample quality measured by FID $\downarrow$ on CIFAR-10 32$\times$32 dataset (EDM), evaluated on 50k samples. Numbers in parentheses indicate the solver order.}
    \label{tab:fid_cifar10_edm}
    \renewcommand{\arraystretch}{1.2}
    \setlength{\tabcolsep}{5.0pt}
    \begin{tabular}{lccccccccccc}
        \toprule
        \multirow{2}{*}{\textbf{Model}} & \multicolumn{11}{c}{\textbf{NFE}}\\
        \cmidrule(lr){2-12}
        & \textbf{5} & \textbf{6} & \textbf{8} & \textbf{10} & \textbf{12} & \textbf{15} & \textbf{20} & \textbf{25} & \textbf{30} & \textbf{35} & \textbf{40} \\
        \midrule
        DPM-Solver++ (3)   & 24.54 & 11.87 & 4.36 & 2.91 & 2.45 & 2.17 & 2.05 & 2.02 & 2.01 & 2.00 & 2.00 \\
        UniPC-$B_1(h)$ (3)    & 23.75 & 12.74 & 4.22 & 3.01 & 2.41 & 2.08 & 2.01 & 2.00 & 2.00 & 2.00 & 2.00 \\
        UniPC-$B_2(h)$ (3)    & 23.52 & 11.11 & 3.86 & \textbf{2.85} & \textbf{2.37} & 2.07 & 2.01 & 2.00 & 1.99 & 1.99 & 1.99 \\
        \addlinespace[0.5ex]
        \rowcolor{gray!20}
        RBF-Solver (3)     & 22.25 & 10.78 & \textbf{3.77} & 2.88 & \textbf{2.37} & 2.06 & \textbf{1.99} & \textbf{1.98} & \textbf{1.98} & \textbf{1.98} & \textbf{1.98} \\
        \rowcolor{gray!20}
        RBF-Solver (4)     & \textbf{14.91} & \textbf{7.67} & 4.90 & 3.71 & 2.50 & \textbf{2.05} & 2.00 & 1.99 & 1.99 & 1.99 & \textbf{1.98} \\
        \bottomrule
    \end{tabular}
\end{table}

\subsubsection{ImageNet 64$\times$64 Experiment Results}
\label{app:imagenet64_improved}

\paragraph{Environment Setup}
\begin{itemize}
    \item \textbf{GPUs}: 8$\times$ NVIDIA RTX~4090
    \item \textbf{Model}: Improved-Diffusion~\cite{nichol2021improved},\\
    \url{https://github.com/openai/improved-diffusion}
    \item \textbf{Checkpoint}: \texttt{imagenet64\_uncond\_100M\_1500K.pt}
\end{itemize}

\paragraph{Sampler Settings}
\begin{itemize}
    \item \textbf{UniPC}: uses $B_1(h)$.
    \item \textbf{RBF-Solver}: The shape parameters are optimized using a target set of 128 samples generated by running UniPC with NFE$=$200. The solver is also evaluated with order$=$4.
\end{itemize}

\paragraph{Experiment Results} As shown in Table~\ref{tab:fid_imagenet64_improved}, both RBF-Solver variants (order$=$3 and order$=$4) outperform the competing samplers at every NFE except 5 and 6.


\begin{table}[h]
    \centering
    \caption{Sample quality measured by FID $\downarrow$ on ImageNet 64$\times$64 dataset (Improved-Diffusion), evaluated on 50k samples. Numbers in parentheses indicate the solver order.}
    \label{tab:fid_imagenet64_improved}
    \renewcommand{\arraystretch}{1.2}
    \setlength{\tabcolsep}{3pt}
    \begin{tabular}{lccccccccccc}
        \toprule
        \multirow{2}{*}{\textbf{Model}} & \multicolumn{11}{c}{\textbf{NFE}}\\
        \cmidrule(lr){2-12}
        & \textbf{5} & \textbf{6} & \textbf{8} & \textbf{10} & \textbf{12} & \textbf{15} & \textbf{20} & \textbf{25} & \textbf{30} & \textbf{35} & \textbf{40} \\
        \midrule
        DPM-Solver++ (3)  & 110.23 & 66.00 & 38.48 & 28.95 & 24.59 & 21.62 & 19.94 & 19.36 & 18.91 & 18.70 & 18.57 \\
        UniPC-$B_1(h)$ (3)  & \textbf{109.81} & \textbf{65.06} & 35.11 & 27.75 & 24.30 & 21.15 & 19.51 & 18.98 & 18.60 & 18.42 & 18.34 \\
        \rowcolor{gray!20}
        RBF-Solver (3) & 116.55 & 70.68 & 34.15 & 26.47 & 22.72 & 20.89 & 19.30 & 18.72 & 18.31 & 18.17 & 18.07 \\
        \rowcolor{gray!20}
        RBF-Solver (4) & 126.26 & 76.92 & \textbf{31.71} & \textbf{24.72} & \textbf{21.70} & \textbf{19.61} & \textbf{18.67} & \textbf{18.30} & \textbf{18.08} & \textbf{17.99} & \textbf{17.95} \\
        \bottomrule
    \end{tabular}
\end{table}

\begin{table*}[p]
  \centering

  \caption{Qualitative comparison of different samplers on ImageNet 64$\times$64. Each image is randomly sampled and solver order$=$3. Columns correspond to samplers; rows indicate NFE values of 5, 10, and 15.}
  
  \label{tab:imagenet64_samples}
  \renewcommand{\arraystretch}{1.02}
  \setlength{\tabcolsep}{2pt}

  \begin{adjustbox}{max width=\textwidth}
  \begin{tabular}{@{}C C C@{}}
    \toprule
    \textbf{DPM-Solver++} & \textbf{UniPC} & \textbf{RBF-Solver} \\
    \midrule
    \multicolumn{3}{c}{\textbf{NFE = 5}} \\[2pt]
    \includegraphics[width=\linewidth]{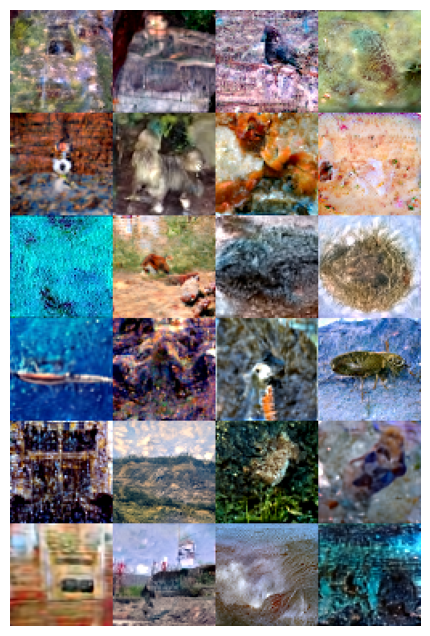} &
    \includegraphics[width=\linewidth]{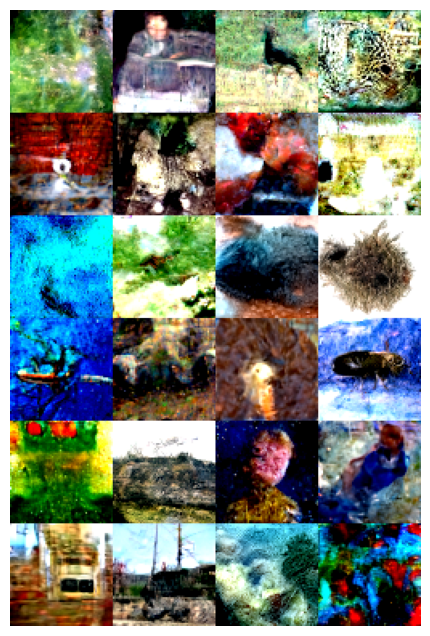} &
    \includegraphics[width=\linewidth]{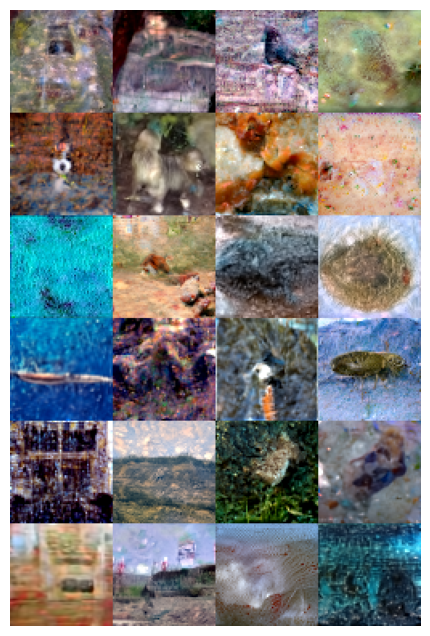} \\
    \midrule
    \multicolumn{3}{c}{\textbf{NFE = 10}} \\[2pt]
    \includegraphics[width=\linewidth]{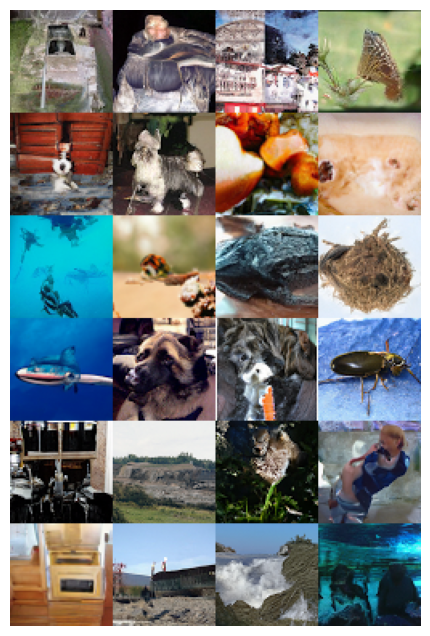} &
    \includegraphics[width=\linewidth]{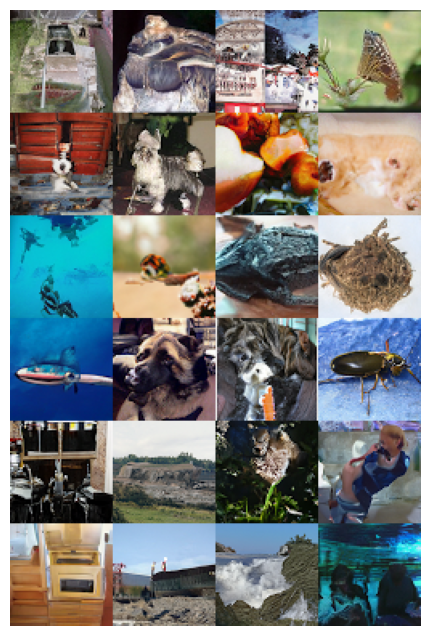} &
    \includegraphics[width=\linewidth]{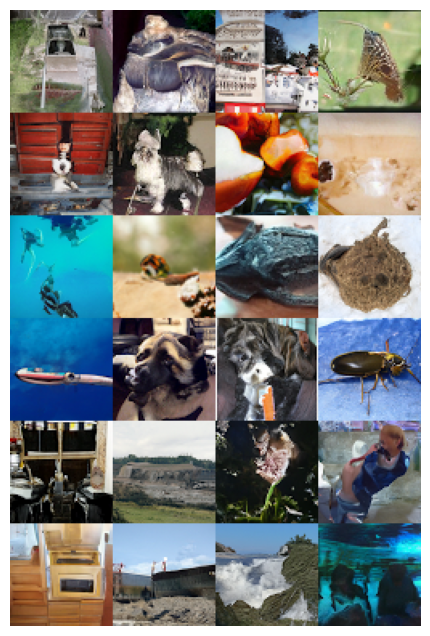} \\
    \midrule
    \multicolumn{3}{c}{\textbf{NFE = 15}} \\[2pt]
    \includegraphics[width=\linewidth]{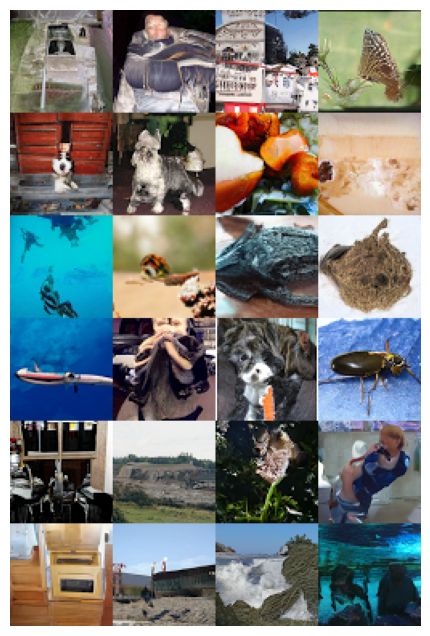} &
    \includegraphics[width=\linewidth]{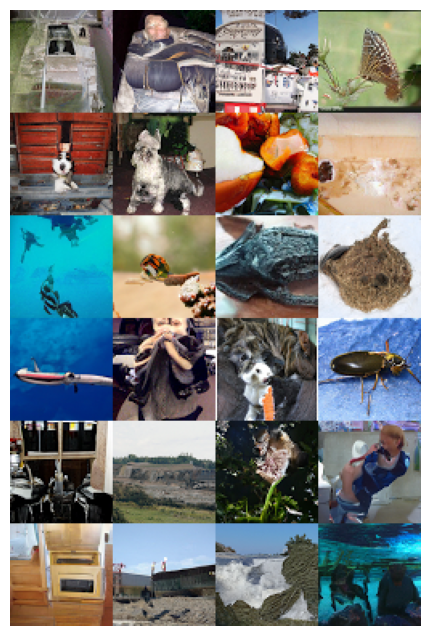} &
    \includegraphics[width=\linewidth]{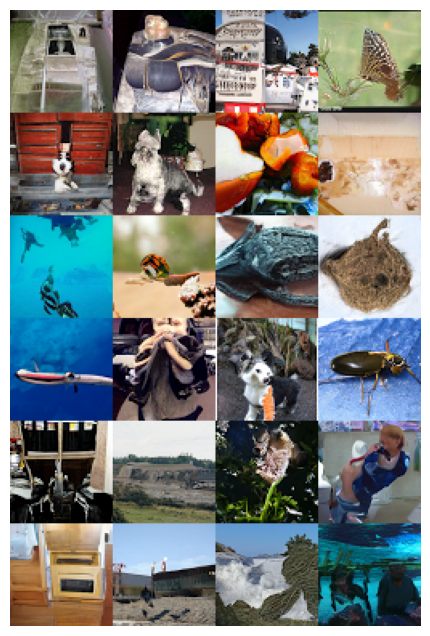} \\
    \bottomrule
  \end{tabular}
  \end{adjustbox}
\end{table*}

\subsection{Conditional Models}
\label{app:Additional Experiment Results - Conditional Models}

\subsubsection{ImageNet 128$\times$128 Experiment Results}
\label{app:imagenet128}

\paragraph{Environment Setup}
\begin{itemize}
    \item \textbf{GPUs}: 8-way NVIDIA RTX~4090
    \item \textbf{Model}: Guided-Diffusion~\cite{dhariwal2021diffusion},
    \url{https://github.com/openai/guided-diffusion}
    \item \textbf{Checkpoint}: \texttt{128x128\_classifier.pt, 128x128\_diffusion.pt}
\end{itemize}

\paragraph{Sampler Settings}
\begin{itemize}
    \item \textbf{UniPC}: employs $B_2(h)$, which yields the best results in conditional generation, as reported in \cite{zhao2023unipc}.
    \item \textbf{RBF-Solver}: The shape parameters are optimized by averaging the results of 20 runs, each performed on a batch of 16 images randomly drawn from a 128-image target set generated with UniPC NFE$=$200. The solver is also evaluated with order$=$3.
\end{itemize}

\paragraph{Experiment Results}
As shown in Table~\ref{tab:fid_main_imagenet128}, RBF-Solver with order$=$2 and order$=$3 consistently outperforms the competing solvers at every guidance scale. This advantage becomes especially pronounced at the higher scales of 6.0 and 8.0 and when the NFE is low. Furthermore, the order$=$3 variant yields slightly better performance than the order$=$2 variant overall.


\begin{table}[h!]
    \centering
    \caption{Sample quality measured by FID $\downarrow$ on ImageNet 128$\times$128 dataset (Guided-Diffusion), evaluated on 50k samples. Numbers in parentheses indicate the solver order.} 
    \label{tab:fid_main_imagenet128}
    \renewcommand{\arraystretch}{1.2}
    \setlength{\tabcolsep}{6pt}
    \begin{tabular}{llcccccccc}
        \toprule
        & & \multicolumn{8}{c}{\textbf{NFE}} \\
        \cmidrule(lr){3-10}
        & & \textbf{5} & \textbf{6} & \textbf{8} & \textbf{10} & \textbf{12} & \textbf{15} & \textbf{20} & \textbf{25} \\
        \midrule

        \multicolumn{10}{l}{\textbf{Guidance Scale = 2.0}} \\
        \addlinespace[0.5ex]
        & DPM-Solver++ (2)              & 15.32 & 10.60 &  7.16 &  6.01 &  5.41 &  4.93 &  4.50 &  4.30 \\
        & UniPC-$B_2(h)$ (2)                 & 13.43 &  9.19 &  6.50 &  5.69 &  5.23 &  4.85 & \textbf{4.48} &  4.29 \\
        \rowcolor{gray!20}
        & RBF-Solver (2)          & 12.83 &  8.80 &  6.34 &  5.61 &  5.22 &  4.85 &  4.49 & \textbf{4.28} \\
        \rowcolor{gray!20}
        & RBF-Solver (3)          & \textbf{12.19} & \textbf{8.15} & \textbf{5.98} & \textbf{5.51} & \textbf{5.16} & \textbf{4.84} &  4.51 &  4.30 \\
        \midrule

        \multicolumn{10}{l}{\textbf{Guidance Scale = 4.0}} \\
        \addlinespace[0.5ex]
        & DPM-Solver++ (2)            & 15.56 & 10.86 &  8.31 &  7.67 &  7.40 &  7.12 & \textbf{6.82} & \textbf{6.65} \\
        & UniPC-$B_2(h)$ (2)               & 15.10 &  9.98 &  7.81 &  7.46 &  7.28 &  7.13 &  6.89 &  6.72 \\
        \rowcolor{gray!20}
        & RBF-Solver (2)          & 14.24 &  9.66 &  7.71 &  7.40 &  7.20 &  7.11 &  6.86 &  6.72 \\
        \rowcolor{gray!20}
        & RBF-Solver (3)          & \textbf{13.77} & \textbf{9.52} & \textbf{7.42} & \textbf{7.11} & \textbf{7.01} & \textbf{6.96} &  6.83 &  6.70 \\
        \midrule

        \multicolumn{10}{l}{\textbf{Guidance Scale = 6.0}} \\
        \addlinespace[0.5ex]
        & DPM-Solver++ (2)            & 25.67 & 15.61 & 10.21 &  9.09 &  8.69 &  8.63 &  8.48 &  8.27 \\
        & UniPC-$B_2(h)$ (2)               & 29.12 & 17.39 & 10.69 &  9.14 &  8.59 &  8.51 &  8.51 &  8.39 \\
        \rowcolor{gray!20}
        & RBF-Solver (2)          & 19.93 & \textbf{12.28} & \textbf{9.36} &  8.68 &  8.58 &  8.48 &  8.36 &  8.29 \\
        \rowcolor{gray!20}
        & RBF-Solver (3)          & \textbf{19.46} & 12.61 &  9.75 & \textbf{8.60} & \textbf{8.33} & \textbf{8.27} & \textbf{8.21} & \textbf{8.18} \\
        \midrule

        \multicolumn{10}{l}{\textbf{Guidance Scale = 8.0}} \\
        \addlinespace[0.5ex]
        & DPM-Solver++ (2)            & 43.76 & 28.75 & 15.96 & 11.68 & 10.11 &  9.60 &  9.51 &  9.49 \\
        & UniPC-$B_2(h)$ (2)               & 50.63 & 34.56 & 19.94 & 14.11 & 11.10 &  9.72 &  9.46 &  9.46 \\
        \rowcolor{gray!20}
        & RBF-Solver (2)         & 30.55 & \textbf{17.87} & \textbf{11.70} & 10.35 &  9.64 &  9.43 &  9.37 &  9.40 \\
        \rowcolor{gray!20}
        & RBF-Solver (3)          & \textbf{30.24} & 18.41 & 13.11 & \textbf{10.22} & \textbf{9.44} & \textbf{9.20} & \textbf{9.22} & \textbf{9.29} \\
        \bottomrule
    \end{tabular}
\end{table}


\begin{table*}[p]
  \centering
  \caption{Qualitative comparison of various samplers on ImageNet 128$\times$128. Each image is randomly sampled with guidance scale 6.0 and solver order$=$2. Columns correspond to samplers; rows indicate NFE values of 5, 10, and 15.}
  \label{tab:imagenet128_scale4_samples}
  \renewcommand{\arraystretch}{1.02}
  \setlength{\tabcolsep}{2pt}

  \begin{adjustbox}{max width=\textwidth}
  \begin{tabular}{@{}C C C@{}}
    \toprule
    \textbf{DPM-Solver++} & \textbf{UniPC} & \textbf{RBF-Solver} \\
    \midrule
    \multicolumn{3}{c}{\textbf{NFE = 5}} \\[2pt]
    \includegraphics[width=\linewidth]{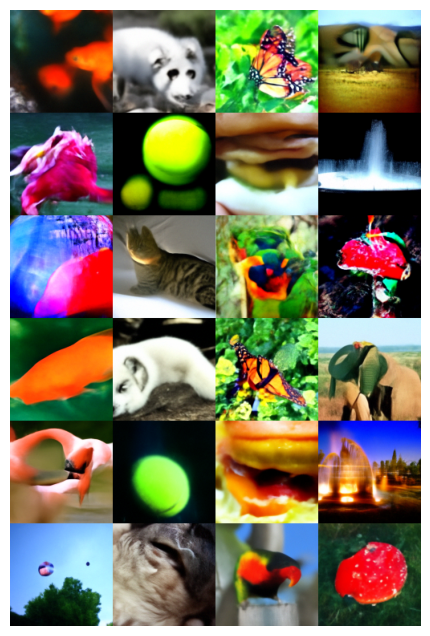} &
    \includegraphics[width=\linewidth]{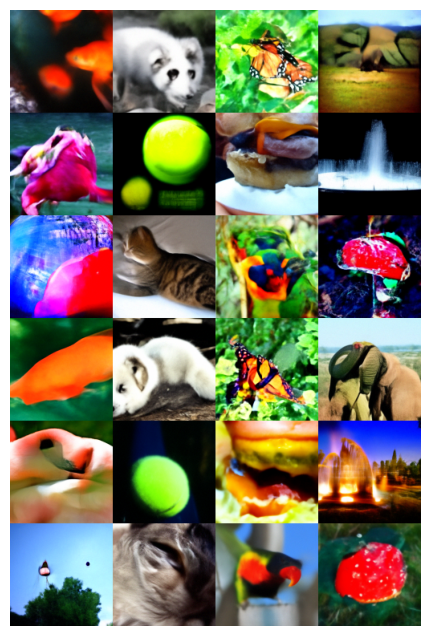} &
    \includegraphics[width=\linewidth]{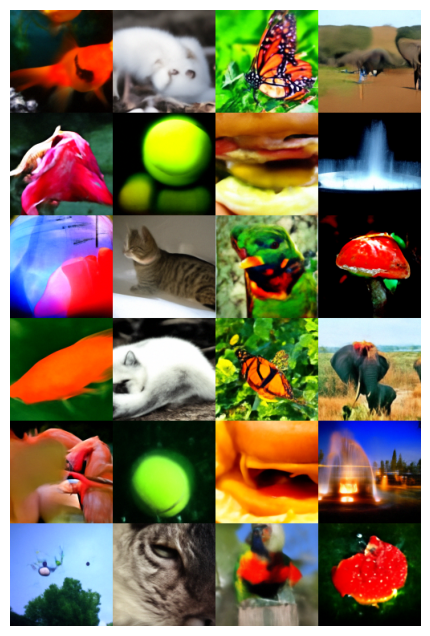} \\
    \midrule
    \multicolumn{3}{c}{\textbf{NFE = 10}} \\[2pt]
    \includegraphics[width=\linewidth]{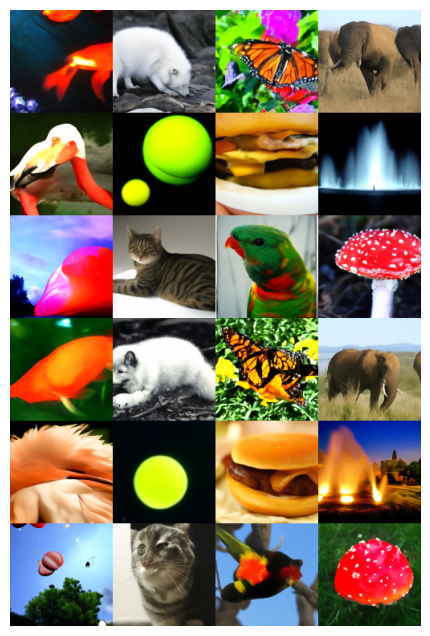} &
    \includegraphics[width=\linewidth]{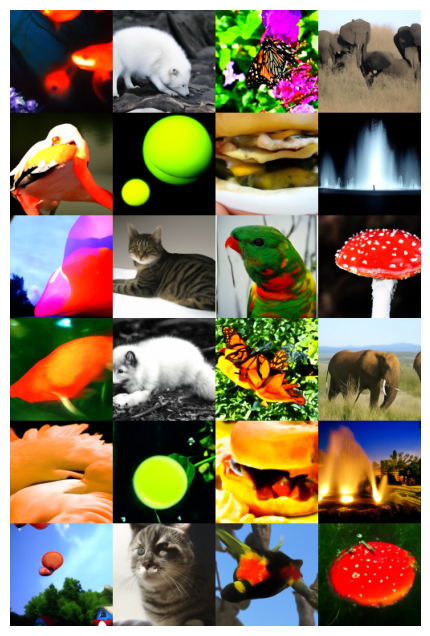} &
    \includegraphics[width=\linewidth]{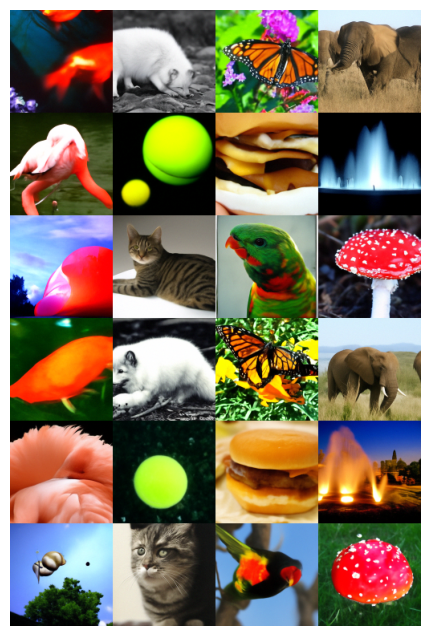} \\
    \midrule
    \multicolumn{3}{c}{\textbf{NFE = 15}} \\[2pt]
    \includegraphics[width=\linewidth]{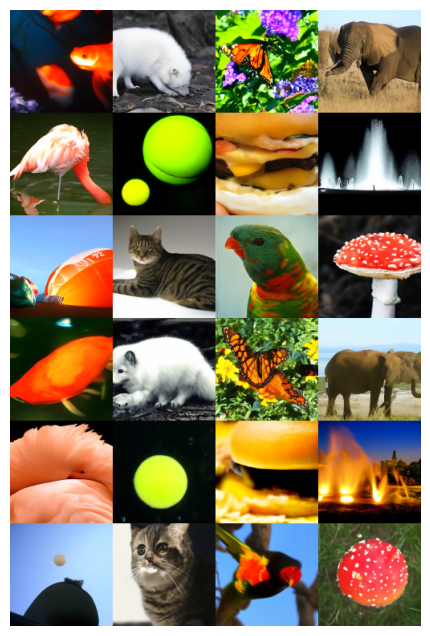} &
    \includegraphics[width=\linewidth]{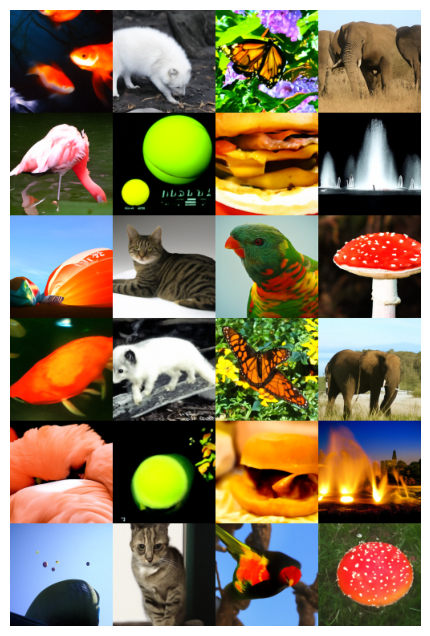} &
    \includegraphics[width=\linewidth]{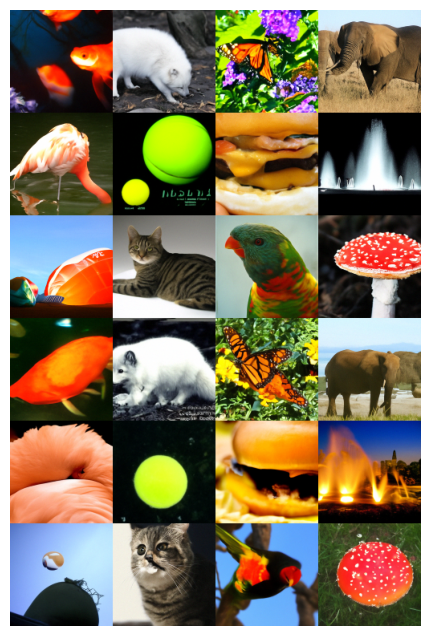} \\
    \bottomrule
  \end{tabular}
  \end{adjustbox}
\end{table*}

\subsubsection{ImageNet 256$\times$256 Experiment Results}
\label{app:imagenet256}


\paragraph{Environment Setup}
\begin{itemize}
    \item \textbf{GPUs}: 8-way NVIDIA RTX~4090
    \item \textbf{Model}: Guided-Diffusion~\cite{dhariwal2021diffusion},
    \url{https://github.com/openai/guided-diffusion}
    \item \textbf{Checkpoint}: \texttt{256x256\_classifier.pt, 256x256\_diffusion.pt}
\end{itemize}

\paragraph{Sampler Settings}
\begin{itemize}
    \item \textbf{UniPC}: employs $B_2(h)$.
    \item \textbf{RBF-Solver}: The shape parameters are optimized by averaging the results of 20 runs, each performed on a batch of 16 images randomly drawn from a 128 image-target set generated with UniPC NFE$=$200. The solver is also evaluated with order$=$3.
\end{itemize}

\paragraph{Experiment Results}
As shown in Table~\ref{tab:fid_main_imagenet256}, RBF-Solver attains performance similar to the competing samplers at guidance scales 2.0 and 4.0, while clearly outperforming them at the higher scales 6.0 and 8.0.  
This advantage becomes more pronounced as the NFE decreases.  
Overall, the order$=$4 variant of RBF-Solver yields better results than its order$=$3 variant.


\begin{table}[h]
    \centering
    \caption{Sample quality measured by FID $\downarrow$ on ImageNet 256$\times$256 dataset (Guided-Diffusion), evaluated on 10k samples. Numbers in parentheses indicate the solver order.} 
    \label{tab:fid_main_imagenet256}
    \renewcommand{\arraystretch}{1.2}
    \setlength{\tabcolsep}{6pt}
    \begin{tabular}{llccccccccc}
        \toprule
        & & \multicolumn{9}{c}{\textbf{NFE}} \\
        \cmidrule(lr){3-11}
        & & \textbf{5} & \textbf{6} & \textbf{8} & \textbf{10} & \textbf{12} & \textbf{15} & \textbf{20} & \textbf{25} & \textbf{30} \\
        \midrule

        \multicolumn{11}{l}{\textbf{Guidance Scale = 2.0}} \\
        \addlinespace[0.5ex]
        & DPM-Solver++ (2)      & 16.27 & 12.59 &  9.62 &  8.54 &  8.11 &  7.68 &  7.44 &  7.33 &  7.27 \\
        & UniPC-$B_2(h)$ (2)     & 15.00 & 11.53 &  9.05 & \textbf{8.20} & \textbf{7.83} & \textbf{7.62} & \textbf{7.36} & \textbf{7.23} & \textbf{7.23} \\
        \rowcolor{gray!20}
        & RBF-Solver (2)     & 15.05 & 11.45 &  9.11 &  8.28 &  7.87 &  7.71 &  7.43 &  7.31 &  7.28 \\
        \rowcolor{gray!20}
        & RBF-Solver (3)     & \textbf{14.77} & \textbf{11.11} & \textbf{9.00} &  8.40 &  8.05 &  7.78 &  7.45 &  7.31 &  7.28 \\
        \midrule

        \multicolumn{11}{l}{\textbf{Guidance Scale = 4.0}} \\
        \addlinespace[0.5ex]
        & DPM-Solver++ (2)      & 18.73 & 13.28 &  9.92 &  9.02 &  8.57 &  8.35 &  8.10 & \textbf{8.01} & \textbf{7.95} \\
        & UniPC-$B_2(h)$ (2)     & 18.41 & 12.92 &  9.63 &  8.85 & \textbf{8.50} &  8.34 & \textbf{8.04} &  8.03 &  7.97 \\
        \rowcolor{gray!20}
        & RBF-Solver (2)    & \textbf{17.76} & 12.47 &  9.55 &  8.95 &  8.58 &  8.35 &  8.08 & \textbf{8.01} &  7.99 \\
        \rowcolor{gray!20}
        & RBF-Solver (3)    & 18.33 & \textbf{12.03} & \textbf{9.22} & \textbf{8.83} &  8.52 & \textbf{8.26} &  8.11 &  8.02 &  7.97 \\
        \midrule

        \multicolumn{11}{l}{\textbf{Guidance Scale = 6.0}} \\
        \addlinespace[0.5ex]
        & DPM-Solver++ (2)      & 31.46 & 20.86 & 12.36 & 10.07 &  9.34 &  9.04 &  8.95 &  8.89 &  8.82 \\
        & UniPC-$B_2(h)$ (2)    & 34.05 & 22.93 & 13.71 & 10.53 &  9.74 &  9.12 &  9.02 &  8.85 &  8.82 \\
        \rowcolor{gray!20}
        & RBF-Solver (2)     & 26.78 & 16.21 & 11.00 &  9.74 &  9.39 &  8.97 &  8.95 & \textbf{8.79} &  8.79 \\
        \rowcolor{gray!20}
        & RBF-Solver (3)     & \textbf{24.91} & \textbf{15.48} & \textbf{10.72} & \textbf{9.65} & \textbf{9.31} & \textbf{8.91} & \textbf{8.79} &  8.80 & \textbf{8.75} \\
        \midrule

        \multicolumn{11}{l}{\textbf{Guidance Scale = 8.0}} \\
        \addlinespace[0.5ex]
        & DPM-Solver++ (2)      & 50.80 & 37.41 & 20.50 & 13.44 & 10.98 &  9.85 &  9.61 &  9.54 &  9.48 \\
        & UniPC-$B_2(h)$ (2)    & 55.80 & 43.31 & 26.09 & 17.20 & 12.71 & 10.35 &  9.57 &  9.53 &  9.53 \\
        \rowcolor{gray!20}
        & RBF-Solver (2)     & 42.61 & \textbf{24.79} & \textbf{13.79} & 11.25 & 10.30 &  9.69 &  9.46 & \textbf{9.42} &  9.48 \\
        \rowcolor{gray!20}
        & RBF-Solver (3)     & \textbf{38.81} & \textbf{24.79} & 14.04 & \textbf{10.98} & \textbf{10.21} & \textbf{9.63} & \textbf{9.44} &  9.47 & \textbf{9.41} \\
        \bottomrule
    \end{tabular}
\end{table}


\begin{table*}[p]
  \centering
  \caption{Qualitative comparison of different samplers on ImageNet 256$\times$256. Each image is randomly sampled with guidance scale set to 8.0 and solver order$=$2. Columns correspond to samplers; rows indicate NFE values of 5, 10, and 15.}
  \label{tab:imagenet256_scale4_samples}
  \renewcommand{\arraystretch}{1.02}
  \setlength{\tabcolsep}{2pt}

  \begin{adjustbox}{max width=\textwidth}
  \begin{tabular}{@{}C C C@{}}
    \toprule
    \textbf{DPM-Solver++} & \textbf{UniPC} & \textbf{RBF-Solver} \\
    \midrule
    \multicolumn{3}{c}{\textbf{NFE = 5}} \\[2pt]
    \includegraphics[width=\linewidth]{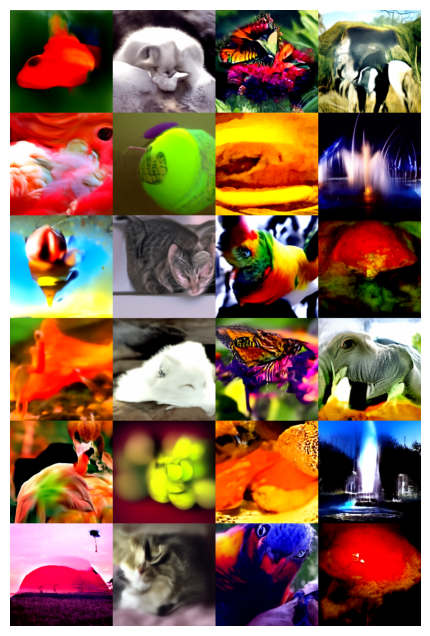} &
    \includegraphics[width=\linewidth]{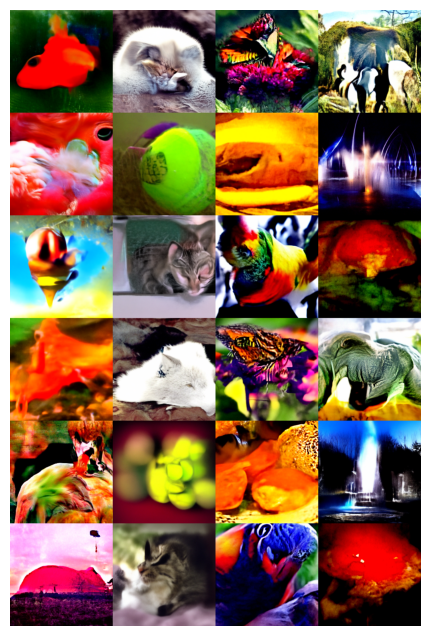} &
    \includegraphics[width=\linewidth]{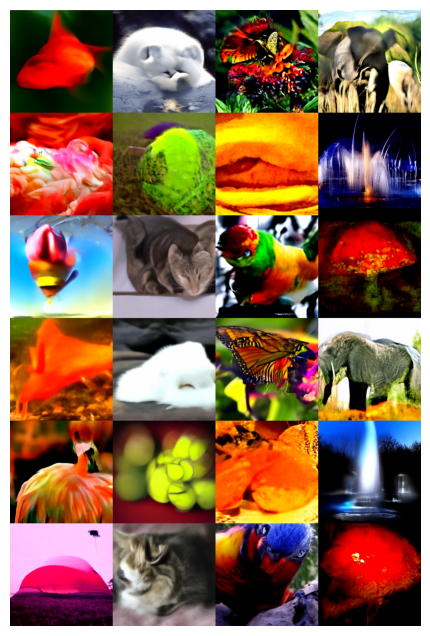} \\
    \midrule
    \multicolumn{3}{c}{\textbf{NFE = 10}} \\[2pt]
    \includegraphics[width=\linewidth]{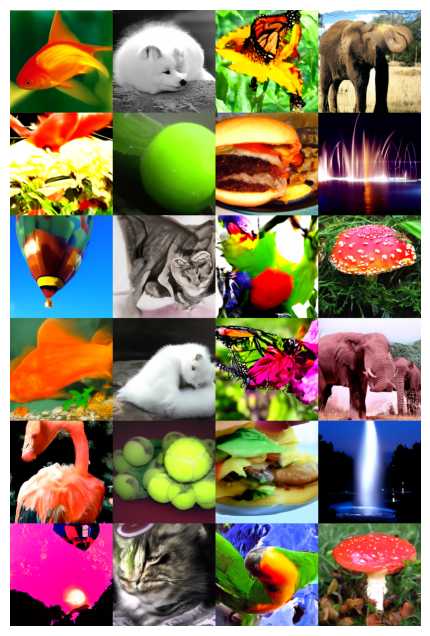} &
    \includegraphics[width=\linewidth]{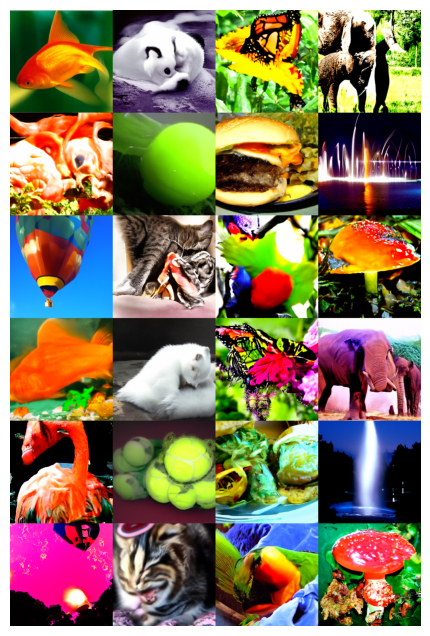} &
    \includegraphics[width=\linewidth]{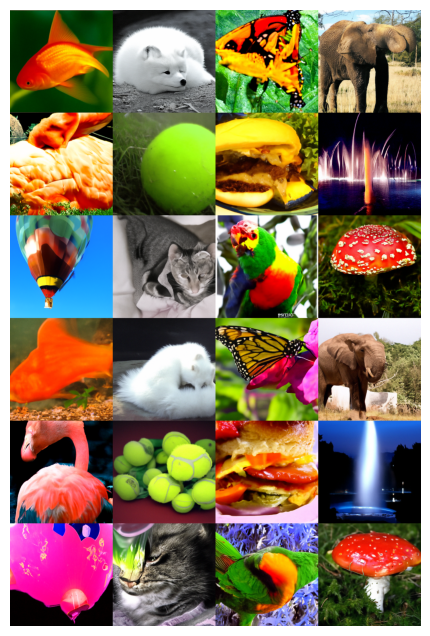} \\
    \midrule
    \multicolumn{3}{c}{\textbf{NFE = 15}} \\[2pt]
    \includegraphics[width=\linewidth]{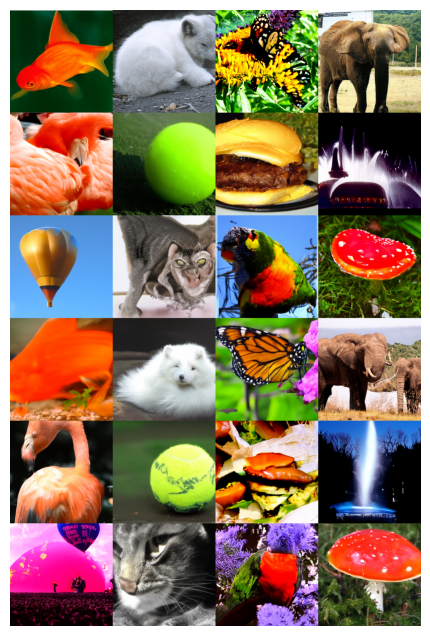} &
    \includegraphics[width=\linewidth]{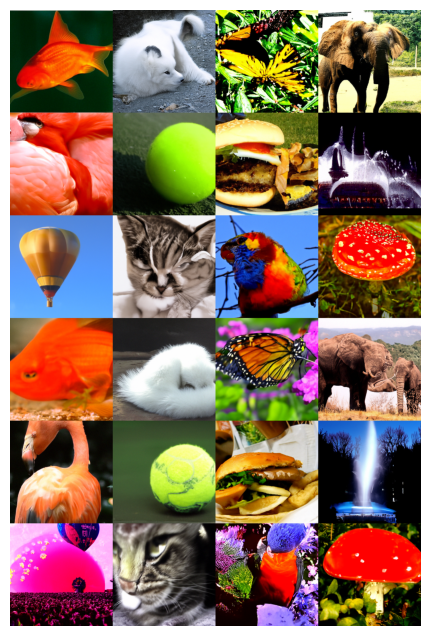} &
    \includegraphics[width=\linewidth]{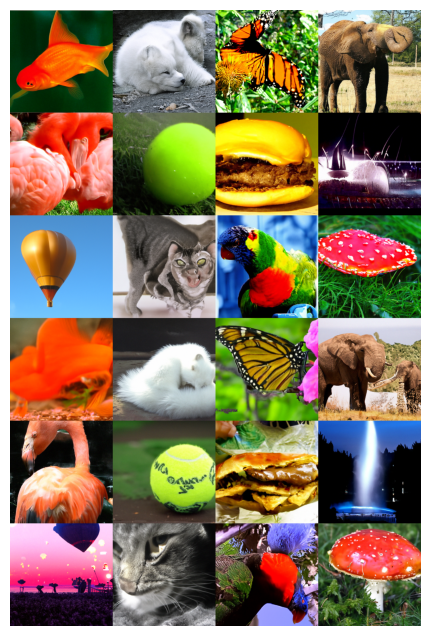} \\
    \bottomrule
  \end{tabular}
  \end{adjustbox}
\end{table*}

\subsubsection{Stable-Diffusion v1.4 Experiment Results}
\label{app:stable_diffusion}

\paragraph{Environment Setup}
\begin{itemize}
    \item \textbf{GPU}: Single NVIDIA RTX~4090
    \item \textbf{Model}: Stable-Diffusion~\cite{rombach2022high},
    \url{https://github.com/CompVis/stable-diffusion}
    \item \textbf{Checkpoint}: \texttt{sd-v1-4.ckpt}
\end{itemize}

\paragraph{Sampler Settings}
\begin{itemize}
    \item \textbf{UniPC}: employs $B_2(h)$.
    \item \textbf{RBF-Solver}: The shape parameters are optimized by averaging the results of 20 runs, each performed on a batch of 6 images randomly drawn from a 128 image target set generated with UniPC at NFE$=$200. The solver is also evaluated with order=3.
\end{itemize}

\paragraph{Experiment Results}
Table \ref{tab:main_stable} reports both the RMSE loss and the CLIP~\cite{radford2021learning} cosine similarity.
For RMSE, we regard the images sampled with UniPC at NFE=200 as references and compute the root-mean-square error against them.
For the CLIP metric, we input the reference images and the generated samples into CLIP ViT-B/32 and take the cosine of their image-embedding vectors.
RBF-Solver outperforms the competing samplers in RMSE at guidance scales 1.5 and 9.5, and performs comparably to them at the intermediate scales.
With respect to CLIP cosine similarity, however, RBF-Solver delivers consistently superior results across all guidance scales.

\begin{table*}[t]
    \centering
    \caption{CLIP embedding cosine similarity (↑) and RMSE loss (↓) with Stable Diffusion v1.4 for guidance scales 1.5, 3.5, 5.5, 7.5 and 9.5.
             Bold numbers mark the best value per NFE in each metric.}
    \label{tab:main_stable}
    \renewcommand{\arraystretch}{0.92}
    \setlength{\tabcolsep}{6pt}

  \begin{adjustbox}{max height=\textheight}
      \begin{minipage}{\textwidth}
      \centering

    \ContinuedFloat
        \caption*{(a) CLIP cosine similarity (↑).}

        \begin{tabular}{llccccccc}
            \toprule
            & & \multicolumn{7}{c}{\textbf{NFE}} \\
            \cmidrule(lr){3-9}
            & & \textbf{5} & \textbf{6} & \textbf{8} & \textbf{10} & \textbf{12} & \textbf{15} & \textbf{20} \\
            \midrule
            \multicolumn{9}{l}{\textbf{Guidance Scale = 1.5}}\\
            \addlinespace[0.4ex]
            & DPM-Solver++        & 0.8579 & 0.8906 & 0.9263 & 0.9453 & 0.9570 & 0.9688 & 0.9800 \\
            & UniPC-$B_2(h)$      & 0.8789 & 0.9092 & 0.9404 & 0.9575 & 0.9688 & 0.9790 & 0.9878 \\
            \rowcolor{gray!20}
            & RBF-Solver $p{=}2$  & 0.8887 & 0.9165 & 0.9453 & 0.9614 & 0.9717 & 0.9810 & 0.9888 \\
            \rowcolor{gray!20}
            & RBF-Solver $p{=}3$  & \textbf{0.8926} & \textbf{0.9199} & \textbf{0.9468} & \textbf{0.9629} & \textbf{0.9731} & \textbf{0.9814} & \textbf{0.9893} \\
            \midrule
            \multicolumn{9}{l}{\textbf{Guidance Scale = 3.5}}\\
            \addlinespace[0.4ex]
            & DPM-Solver++        & 0.8774 & 0.8989 & 0.9233 & 0.9380 & 0.9478 & 0.9585 & 0.9702 \\
            & UniPC-$B_2(h)$      & 0.8877 & 0.9067 & 0.9287 & 0.9434 & 0.9536 & 0.9648 & 0.9771 \\
            \rowcolor{gray!20}
            & RBF-Solver $p{=}2$  & 0.8916 & 0.9097 & \textbf{0.9302} & \textbf{0.9448} & \textbf{0.9551} & \textbf{0.9663} & \textbf{0.9785} \\
            \rowcolor{gray!20}
            & RBF-Solver $p{=}3$  & \textbf{0.8931} & \textbf{0.9106} & \textbf{0.9302} & 0.9443 & 0.9546 & \textbf{0.9663} & 0.9780 \\
            \midrule
            \multicolumn{9}{l}{\textbf{Guidance Scale = 5.5}}\\
            \addlinespace[0.4ex]
            & DPM-Solver++        & 0.8794 & 0.8960 & 0.9170 & 0.9297 & 0.9399 & 0.9502 & 0.9619 \\
            & UniPC-$B_2(h)$      & 0.8843 & 0.8994 & 0.9189 & 0.9316 & 0.9419 & 0.9526 & 0.9658 \\
            \rowcolor{gray!20}
            & RBF-Solver $p{=}2$  & \textbf{0.8853} & \textbf{0.9004} & \textbf{0.9194} & \textbf{0.9321} & \textbf{0.9424} & \textbf{0.9531} & \textbf{0.9663} \\
            \rowcolor{gray!20}
            & RBF-Solver $p{=}3$  & 0.8843 & 0.8999 & 0.9189 & 0.9316 & 0.9409 & 0.9521 & 0.9648 \\
            \midrule
            \multicolumn{9}{l}{\textbf{Guidance Scale = 7.5}}\\
            \addlinespace[0.4ex]
            & DPM-Solver++        & 0.8735 & 0.8901 & 0.9097 & 0.9224 & 0.9316 & 0.9424 & 0.9546 \\
            & UniPC-$B_2(h)$      & 0.8730 & 0.8896 & 0.9082 & 0.9214 & 0.9312 & 0.9424 & \textbf{0.9561} \\
            \rowcolor{gray!20}
            & RBF-Solver $p{=}2$  & \textbf{0.8760} & \textbf{0.8931} & \textbf{0.9106} & \textbf{0.9233} & \textbf{0.9321} & \textbf{0.9429} & \textbf{0.9561} \\
            \rowcolor{gray!20}
            & RBF-Solver $p{=}3$  & 0.8735 & 0.8906 & 0.9097 & 0.9224 & 0.9316 & 0.9424 & 0.9551 \\
            \midrule
            \multicolumn{9}{l}{\textbf{Guidance Scale = 9.5}}\\
            \addlinespace[0.4ex]
            & DPM-Solver++        & 0.8594 & 0.8789 & 0.9014 & 0.9146 & 0.9233 & \textbf{0.9346} & \textbf{0.9468} \\
            & UniPC-$B_2(h)$      & 0.8521 & 0.8721 & 0.8965 & 0.9111 & 0.9209 & 0.9326 & 0.9463 \\
            \rowcolor{gray!20}
            & RBF-Solver $p{=}2$  & 0.8613 & \textbf{0.8828} & \textbf{0.9033} & \textbf{0.9150} & \textbf{0.9238} & \textbf{0.9346} & \textbf{0.9468} \\
            \rowcolor{gray!20}
            & RBF-Solver $p{=}3$  & \textbf{0.8618} & 0.8809 & 0.9019 & 0.9146 & \textbf{0.9238} & 0.9341 & \textbf{0.9468} \\
            \bottomrule
        \end{tabular}

        \vspace{1em} 

      
       \ContinuedFloat
        \caption*{(b) RMSE loss (↓).}
    
        \begin{tabular}{llccccccc}
            \toprule
            & & \multicolumn{7}{c}{\textbf{NFE}} \\
            \cmidrule(lr){3-9}
            & & \textbf{5} & \textbf{6} & \textbf{8} & \textbf{10} & \textbf{12} & \textbf{15} & \textbf{20} \\
            \midrule
            \multicolumn{9}{l}{\textbf{Guidance Scale = 1.5}}\\
            \addlinespace[0.4ex]
            & DPM-Solver++        & 0.2605 & 0.2252 & 0.1801 & 0.1505 & 0.1285 & 0.1044 & 0.0780 \\
            & UniPC-$B_2(h)$      & 0.2363 & 0.2030 & 0.1589 & 0.1278 & 0.1045 & 0.0799 & 0.0549 \\
            \rowcolor{gray!20}
            & RBF-Solver $p{=}2$  & 0.2240 & 0.1926 & 0.1506 & 0.1203 & 0.0976 & 0.0741 & 0.0521 \\
            \rowcolor{gray!20}
            & RBF-Solver $p{=}3$  & \textbf{0.2202} & \textbf{0.1905} & \textbf{0.1500} & \textbf{0.1174} & \textbf{0.0932} & \textbf{0.0720} & \textbf{0.0503} \\
            \midrule
            \multicolumn{9}{l}{\textbf{Guidance Scale = 3.5}}\\
            \addlinespace[0.4ex]
            & DPM-Solver++        & 0.3551 & 0.3287 & 0.2887 & 0.2567 & 0.2304 & 0.1986 & 0.1594 \\
            & UniPC-$B_2(h)$      & 0.3490 & 0.3241 & \textbf{0.2823} & \textbf{0.2460} & 0.2153 & 0.1780 & 0.1327 \\
            \rowcolor{gray!20}
            & RBF-Solver $p{=}2$  & \textbf{0.3474} & \textbf{0.3239} & 0.2829 & 0.2462 & \textbf{0.2145} & \textbf{0.1756} & \textbf{0.1284} \\
            \rowcolor{gray!20}
            & RBF-Solver $p{=}3$  & 0.3493 & 0.3276 & 0.2863 & 0.2488 & 0.2167 & 0.1769 & 0.1309 \\
            \midrule
            \multicolumn{9}{l}{\textbf{Guidance Scale = 5.5}}\\
            \addlinespace[0.4ex]
            & DPM-Solver++        & \textbf{0.4608} & \textbf{0.4363} & \textbf{0.3949} & \textbf{0.3600} & 0.3300 & 0.2931 & 0.2436 \\
            & UniPC-$B_2(h)$      & 0.4637 & 0.4405 & 0.3978 & 0.3601 & \textbf{0.3267} & \textbf{0.2847} & \textbf{0.2259} \\
            \rowcolor{gray!20}
            & RBF-Solver $p{=}2$  & 0.4659 & 0.4437 & 0.4018 & 0.3639 & 0.3302 & 0.2871 & 0.2267 \\
            \rowcolor{gray!20}
            & RBF-Solver $p{=}3$  & 0.4647 & 0.4411 & 0.3989 & 0.3631 & 0.3315 & 0.2896 & 0.2333 \\
            \midrule
            \multicolumn{9}{l}{\textbf{Guidance Scale = 7.5}}\\
            \addlinespace[0.4ex]
            & DPM-Solver++        & 0.5686 & 0.5410 & 0.4963 & \textbf{0.4576} & \textbf{0.4255} & \textbf{0.3841} & 0.3276 \\
            & UniPC-$B_2(h)$      & 0.5783 & 0.5517 & 0.5073 & 0.4671 & 0.4321 & 0.3874 & \textbf{0.3230} \\
            \rowcolor{gray!20}
            & RBF-Solver $p{=}2$  & 0.5664 & 0.5407 & 0.5002 & 0.4624 & 0.4310 & 0.3878 & 0.3257 \\
            \rowcolor{gray!20}
            & RBF-Solver $p{=}3$  & \textbf{0.5655} & \textbf{0.5373} & \textbf{0.4956} & 0.4592 & 0.4265 & 0.3869 & 0.3273 \\
            \midrule
            \multicolumn{9}{l}{\textbf{Guidance Scale = 9.5}}\\
            \addlinespace[0.4ex]
            & DPM-Solver++        & 0.6793 & 0.6471 & 0.5950 & 0.5525 & 0.5169 & \textbf{0.4717} & \textbf{0.4108} \\
            & UniPC-$B_2(h)$      & 0.6955 & 0.6659 & 0.6160 & 0.5715 & 0.5337 & 0.4850 & 0.4183 \\
            \rowcolor{gray!20}
            & RBF-Solver $p{=}2$  & 0.6666 & 0.6347 & 0.5859 & 0.5487 & 0.5193 & 0.4757 & 0.4152 \\
            \rowcolor{gray!20}
            & RBF-Solver $p{=}3$  & \textbf{0.6661} & \textbf{0.6327} & \textbf{0.5842} & \textbf{0.5462} & \textbf{0.5147} & 0.4724 & 0.4142 \\
            \bottomrule
        \end{tabular}

     \end{minipage}
    \end{adjustbox}

\end{table*}

\begin{table*}[t]
  \centering
  \caption{Qualitative comparison of different samplers on four shared generation scripts extracted from MS-COCO 2014. Each cell displays a 512×512 image generated using Stable Diffusion v1.4, with classifier-free guidance scale set to 7.5 and NFE$=$5.}
  \label{tab:qualitative_comparison}

  \adjustbox{max width=\textwidth}{
  \renewcommand{\arraystretch}{0.97} 
  \setlength{\tabcolsep}{2pt}

  \begin{tabular}{ccc}
    \toprule
    \textbf{DPM\textminus Solver++} &
    \textbf{UniPC} &
    \textbf{RBF\textminus Solver} \\
    \midrule
    \multicolumn{3}{c}{\textbf{“A baby laying on its back holding a toothbrush.”}}\\[2pt] 
    \includegraphics[width=.31\textwidth]{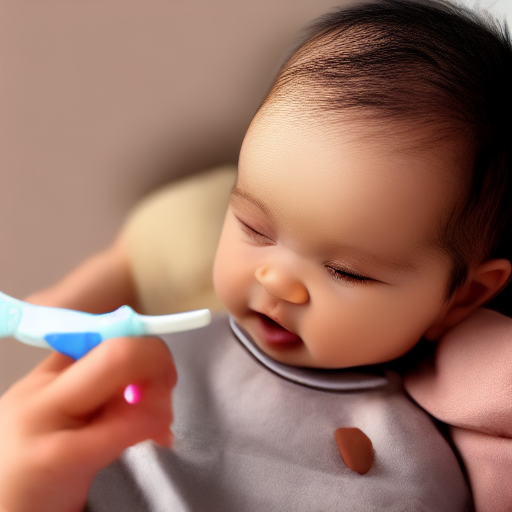} &
    \includegraphics[width=.31\textwidth]{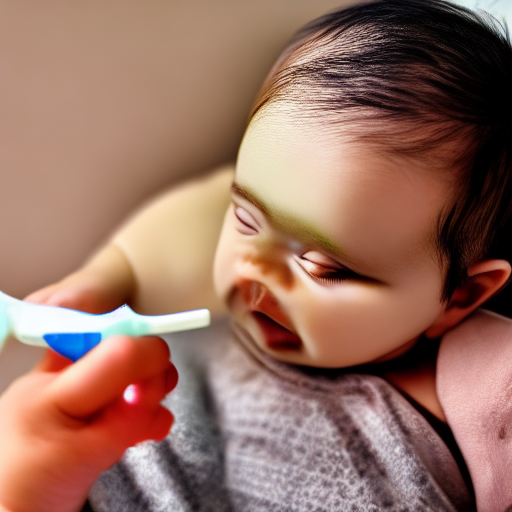} &
    \includegraphics[width=.31\textwidth]{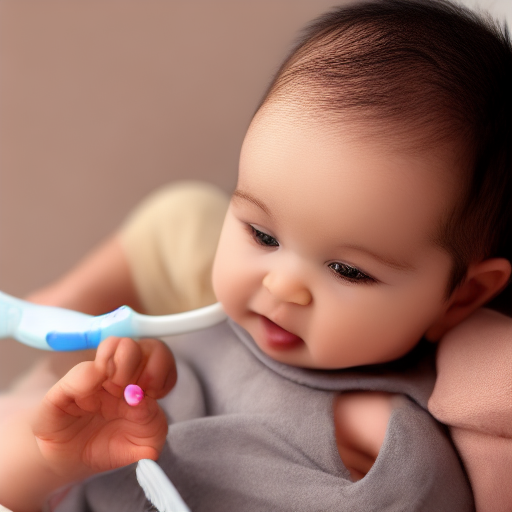} \\[3pt]  

    \multicolumn{3}{c}{\textbf{“A stuffed bear that is under the blanket.”}}\\[2pt]
    \includegraphics[width=.31\textwidth]{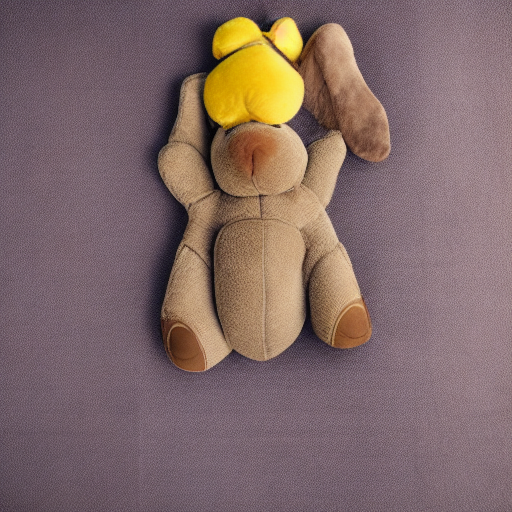} &
    \includegraphics[width=.31\textwidth]{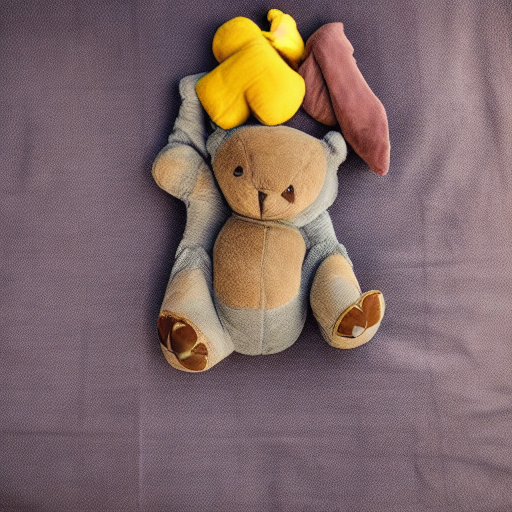} &
    \includegraphics[width=.31\textwidth]{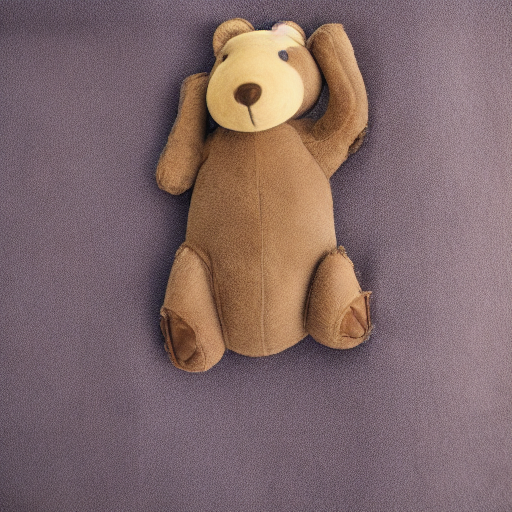} \\[3pt]

    \multicolumn{3}{c}{\textbf{“A small gold and black clock tower in a village.”}}\\[2pt]
    \includegraphics[width=.31\textwidth]{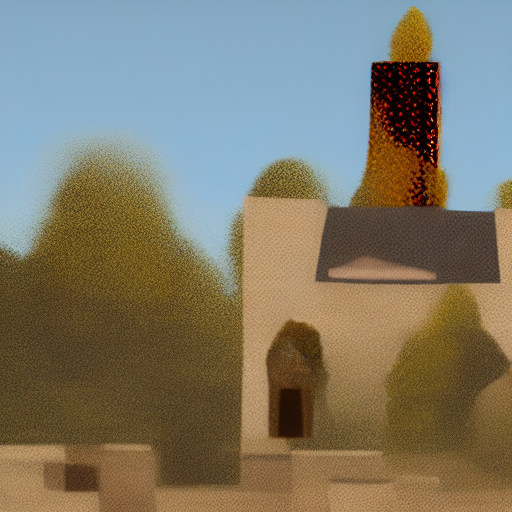} &
    \includegraphics[width=.31\textwidth]{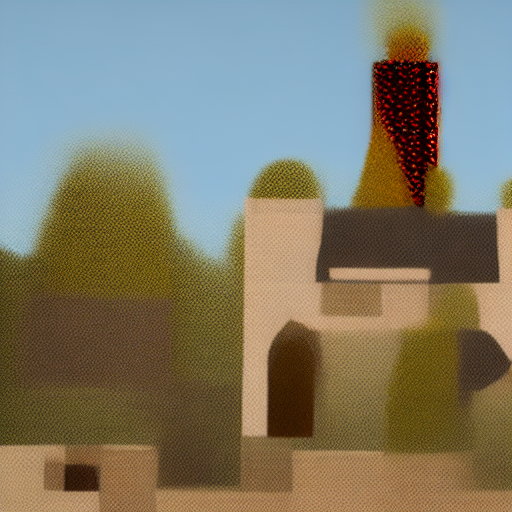} &
    \includegraphics[width=.31\textwidth]{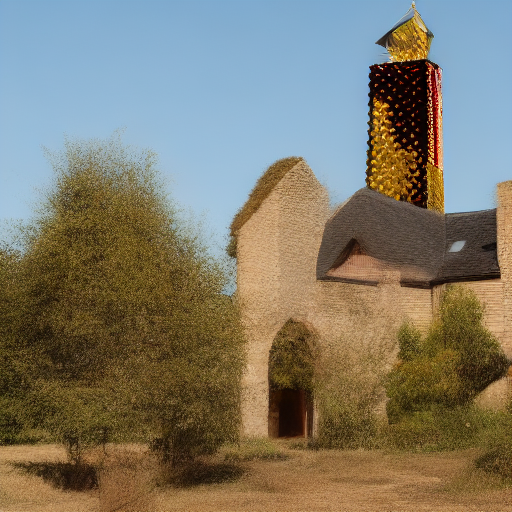} \\[3pt]

    \multicolumn{3}{c}{\textbf{“this is an image of a yorkie in a small bag.”}}\\[2pt]
    \includegraphics[width=.31\textwidth]{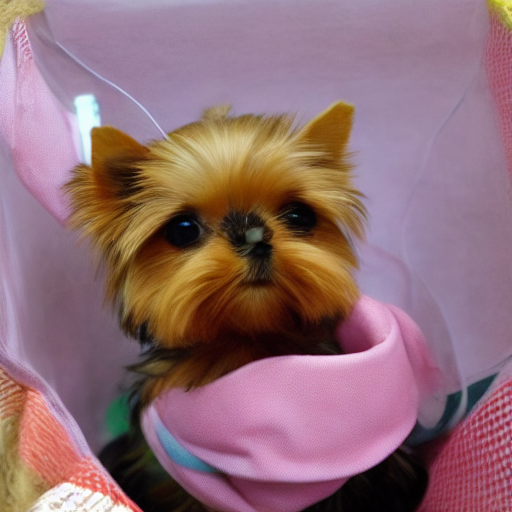} &
    \includegraphics[width=.31\textwidth]{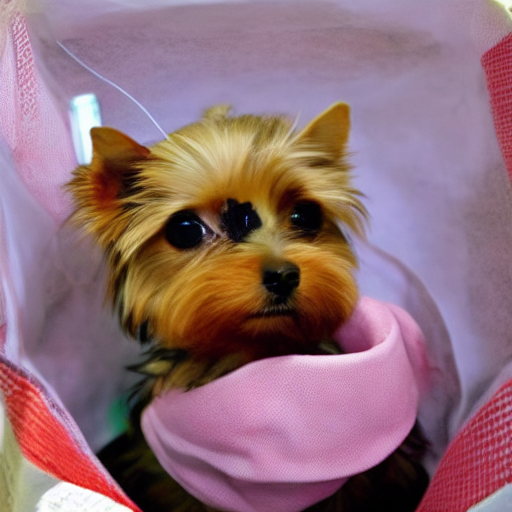} &
    \includegraphics[width=.31\textwidth]{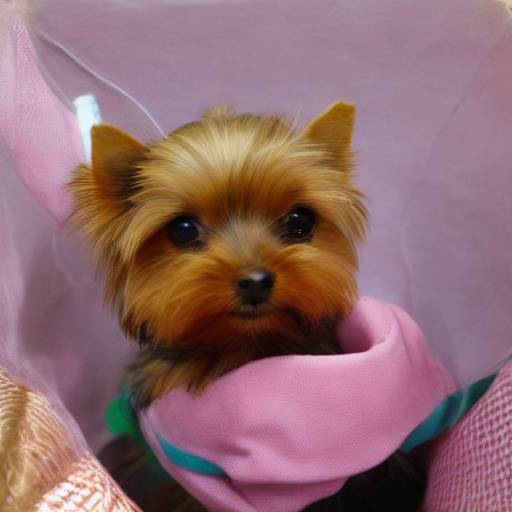} \\
    \bottomrule
  \end{tabular}}
\end{table*}

\clearpage
\section{Experiment Results for Auxiliary-Tuning Samplers}
\label{app:Experiment Results for Auxiliary-Tuning Samplers}
This section defines the class of auxiliary-tuning samplers—including our proposed RBF-Solver—compares these samplers with each other, and shows the experimental overview in Table~\ref{tab:summary_aux-tuning}.

\begin{table}[h]
  \setlength{\tabcolsep}{3pt}   
  \renewcommand{\arraystretch}{1.15}
  \centering
  \resizebox{\linewidth}{!}{%
  \begin{tabular}{llll}
    \toprule
    Dataset & Model & Compared Samplers & Section \\
    \midrule
    \multicolumn{4}{l}{\textbf{Unconditional}} \\[2pt]
    CIFAR-10 32$\times$32~\cite{Krizhevsky09learningmultiple} & Score-SDE~\cite{song2021score} & DPM-Solver-v3~\cite{dpm_solver_v3}, DC-Solver~\cite{dc_solver} & ~\ref{app:cifar10_scoresde_cali} \\
    CIFAR-10 32$\times$32~\cite{Krizhevsky09learningmultiple} & EDM~\cite{karras2022elucidating} & DPM-Solver-v3~\cite{dpm_solver_v3}, DC-Solver~\cite{dc_solver} & ~\ref{app:cifar10_edm_cali} \\
    CIFAR-10 32$\times$32~\cite{Krizhevsky09learningmultiple} & EDM~\cite{karras2022elucidating} & AMED-Solver, AMED-Plugin~\cite{amed_solver} & ~\ref{app:cifar10_edm_amed} \\
    FFHQ 64$\times$64~\cite{karras2019style} & EDM~\cite{karras2022elucidating} & AMED-Solver, AMED-Plugin~\cite{amed_solver} & ~\ref{app:ffhq_edm_amed} \\
    \midrule
    \multicolumn{4}{l}{\textbf{Conditional}} \\[2pt]
    ImageNet 256$\times$256~\cite{deng2009imagenet} & Guided-Diffusion~\cite{dhariwal2021diffusion} & DPM-Solver-v3~\cite{dpm_solver_v3}, DC-Solver~\cite{dc_solver} & ~\ref{app:imagenet256_cali} \\
    LAION-5B~\cite{schuhmann2022laion} & Stable-Diffusion v1.4~\cite{rombach2022high} & DPM-Solver-v3~\cite{dpm_solver_v3}, DC-Solver~\cite{dc_solver} & ~\ref{app:stable_cali} \\
    \bottomrule
  \end{tabular}
  }
  \caption{Summary of experimental settings for auxiliary-tuning samplers.}
  \label{tab:summary_aux-tuning}
\end{table}

\subsection{Comparison of Auxiliary-Tuning Samplers}
\label{app:Comparison of Auxiliary-Tuning Samplers}

Training-Free samplers~\citep{song2021score, lu2022dpm, lu2022dpmpp, zhao2023unipc, xue2024sa} treat diffusion sampling as solving the reverse diffusion ODE/SDE and therefore do not need any re-training of the neural network. Solvers based on numerical methods solve the probability–flow ODE~\cite{song2021score} with analytically derived high-order formulas and have global hyper-parameters, such as
the number of function evaluations~(NFE) and the time–step schedule.
The EDM~\cite{karras2022elucidating} further standardized the sampling framework and showed, through extensive hyper-parameter tuning,
that even these solvers with relatively low NFE can rival long-run stochastic samplers~\cite{ho2020denoising, song2021score} when their
global parameters are carefully tuned.

\begin{table}[b]
  \centering
  \setlength{\tabcolsep}{3pt} 
  \caption{Comparison of solvers that employ auxiliary-parameter tuning.
         \emph{Param.\ Count} is the total number of auxiliary-tuning parameters stored after the offline step; \emph{Dataset Size} is the amount of data each sampler uses for tuning; \emph{Additional Network} indicates whether a separate network is trained.}
  \label{tab:aux_tuning_sampler_comparison}
  \renewcommand{\arraystretch}{1.23}
  \begin{tabular}{lccccc}
    \toprule
    \textbf{Method}
      & \makecell{\textbf{Parameter} \\ \textbf{Type}}
      & \makecell{\textbf{Param.} \\ \textbf{Count}} 
      & \makecell{\textbf{Dataset} \\ \textbf{Size}} 
      & \makecell{\textbf{Offline} \\ \textbf{Tuning Time\textsuperscript{*}}}
      & \makecell{\textbf{Additional} \\ \textbf{Network}} \\ \midrule
    \makecell[l]{\textbf{DPM‐Solver-v3} \\ \cite{dpm_solver_v3}}
      & Empirical Model Stats.
      & 1.1M–295M
      & 1k–4k
      & $<$\,11 h / 8 GPUs
      & \xmark \\

    \makecell[l]{\textbf{AMED‐Solver} \\ \cite{amed_solver}}
      & Neural-Net Weights
      & $\approx$9k
      & 10k
      & $<$\,3 h / 1 GPU
      & \cmark \\

    \makecell[l]{\textbf{DC‐Solver} \\ \cite{dc_solver}}
      & Compensation Ratios
      & $\mathrm{NFE}$
      & 10 
      & $<$\,5 min / 1 GPU
      & \xmark \\
    \rowcolor{gray!20}
    \makecell[l]
    {\textbf{RBF‐Solver  \,\,\,\,\,\,\,\,}\\ \textbf{(Ours)}}
      & Shape Parameters
      & $2\,(\mathrm{NFE}\!-\!1)$ 
      & 128 
      & 1 min - 1 h / 1 GPU\textsuperscript{**}
      & \xmark \\
      \bottomrule
  \end{tabular}
  \vspace{0.4ex}
  \raggedright
  \textsuperscript{*}\;Tuning time varies with model and dataset; we report
  the longest case cited in each paper.
  See Appendix~D of DPM-Solver-v3~\cite{dpm_solver_v3}, Appendix~B of AMED-Solver~\cite{amed_solver},
  and Section 4.1 of DC-Solver~\cite{dc_solver} for details.\\
  \textsuperscript{**}\; These figures are obtained using CIFAR-10, ImageNet~$64\times64$, ImageNet~$128\times128$, and ImageNet~$256\times256$; see Appendix \ref{app:elapsed_time} for details.
\end{table}

Recent research has moved beyond purely plug-and-play methods and now
incorporates model-specific tuning to improve sample quality. 
We call any solver that introduces its own set of \emph{auxiliary
parameters} and performs an offline tuning procedure to determine them an
\emph{auxiliary-tuning sampler}.
During this tuning procedure, the auxiliary parameters are optimized for each pretrained model before sampling. Because the tuning procedure directly minimizes the local discretization error of the underlying ODE/SDE, these samplers preserve image quality even at low NFEs. Empirically, the tuning procedure requires only a handful of trajectories and thus consumes orders of magnitude less computation and data than distillation-based methods~\citep{Luhman2021knowledge, salimansprogressive, meng2023distillation, song2023consistency},
still offering substantial speed-quality gains. Moreover, this lightweight tuning preserves the training-free nature of solver methods, making deployment on existing checkpoints trivial. 

Among existing approaches, the state-of-the-art auxiliary-tuning samplers are
DPM-Solver-v3~\citep{dpm_solver_v3}, AMED-Solver~\citep{amed_solver}, and DC-Solver~\citep{dc_solver}. 
DPM-Solver-v3 unifies the noise and data-prediction models in a single closed-form and introduces high-dimensional tunable parameters to improve sample quality at small NFE.
AMED-Solver~\citep{amed_solver} learns a tiny network that predicts per-step geometric coefficients to lower truncation error by sampling toward the approximate mean direction. DC-Solver~\cite{dc_solver} introduces a compensation ratio $\rho_n$ per step to mitigate the misalignment of predictor-corrector diffusion sampler proposed by UniPC~\cite{zhao2023unipc}.

\paragraph{Parameter type \& parameter count}
We first clarify what is actually stored by each auxiliary-tuning sampler and how many parameters must be tuned before sampling.
DPM-Solver-v3 precomputes the empirical model statistics
(three $D$-dimensional vectors) at $N\!=\!120\!-\!1,\!200$ discrete noise levels, which translates to roughly 1.1 M to 295 M stored parameters (depending on model resolution). 
AMED-Solver employs auxiliary neural networks to predict two geometric coefficients at each sampling step. Apart from the pre-trained feature extractor and time embedder, the only learnable component is a two-layer fully connected module containing 8,950 parameters.
DC-Solver optimizes a single compensation ratio
\(\rho_n\) per step, so the parameter count equals \(\mathrm{NFE}\).  
Finally, RBF-Solver fits two shape parameters per step—one for the predictor and one for the corrector—while skipping the first-step prediction and the final-step correction, yielding $2\times(\mathrm{NFE}-1)$ radial-basis parameters in total. 

\paragraph{Additional network}
Among the four samplers, only AMED-Solver introduces an auxiliary network
(a 2-layer MLP with approximately 9k weights) and performs a forward pass through this network at every sampling step; these weights are trained during the tuning procedure.
In contrast, DPM-Solver-v3, DC-Solver, and RBF-Solver remain
network-free: their auxiliary-tuning parameters are stored in a
static look-up table and incur no extra neural network computation.

\paragraph{Dataset size \& offline tuning time}
Below we quote the number of data points and wall-clock times each sampler takes for parameter tuning from each original paper. 
According to Appendix D of \cite{dpm_solver_v3}, DPM-Solver-v3 uses 4,096 data points for unconditional CIFAR10~\cite{chrabaszcz2017downsampled} and 1,024 for all other settings. With this dataset size, EMS computation finishes in 12 min on 120 step LSUN-Bedroom~\cite{yu2015lsun} and up to \(\sim\)11 h on the
250-step Stable Diffusion v1.4 checkpoint~\cite{rombach2022high}, on eight NVIDIA A40 GPUs.
According to Appendix B of \cite{amed_solver}, AMED-Solver trains its 8,950-parameter MLP on 10k teacher trajectories; the procedure takes 2–8 min on CIFAR-10 and 1–3 h on LSUN-Bedroom, both on a single NVIDIA A100 GPU.
According to Section 4.1 of \cite{dc_solver}, DC-Solver fits one compensation ratio per step on just 10 ground-truth trajectories; the tuning procedure finishes in \(<\!5\) min on a single GPU excluding the time required to prepare those trajectories. RBF-Solver optimizes the shape parameters using 128 reference targets and completes tuning in 1 min. – 1 h on a single NVIDIA RTX 4090 GPU, including the time spent building the targets(see Appendix \ref{app:elapsed_time}).



\subsection{Unconditional Experiment Results for Auxiliary-Tuning Samplers}

\subsubsection{CIFAR-10 (Score-SDE) Results for Auxiliary-Tuning Samplers}
\label{app:cifar10_scoresde_cali}

\paragraph{Environment Setup}
\begin{itemize}
    \item \textbf{GPU}: Single NVIDIA RTX~4090
    \item \textbf{Model}: Score-SDE~\cite{song2021score}, \url{https://github.com/yang-song/score_sde}
    \item \textbf{Checkpoint}: \texttt{vp/cifar10\_ddpmpp\_deep\_continuous}
\end{itemize}

\paragraph{Sampler Settings}
\begin{itemize}
    \item \textbf{DPM-Solver-v3}: The EMS statistics files are obtained from the official repository
          (\url{https://github.com/thu-ml/DPM-Solver-v3/tree/main/codebases/score_sde}).
          All other hyperparameters—\texttt{eps}, \texttt{p\_pseudo}, \texttt{lower\_order\_final},
          \texttt{use\_corrector}, etc.—are kept exactly as specified in
          \url{https://github.com/thu-ml/DPM-Solver-v3/blob/main/codebases/score_sde/sample.sh}.
    \item \textbf{DC-Solver}: As proposed in \cite{dc_solver}, we search the dynamic-compensation ratios using 10 images sampled with DDIM~\cite{song2020denoising} at NFE$=$200. We adopted the \(B_1(h)\) variant recommended by \cite{zhao2023unipc} for unconditional generation.
    \item \textbf{RBF-Solver}: The shape parameters are optimized using a target set of 128 samples generated by running UniPC with NFE$=$200. The solver is also evaluated with order$=$4.
\end{itemize}

\paragraph{Experiment Results}
As shown in Table \ref{tab:fid_cifar10_score_sde_cali}, RBF-Solver achieves the best FID score on CIFAR-10 in the order$=$4 variant when the NFE $\geq$ 15.

\begin{table}[h]
    \centering
    \caption{Sample quality measured by FID $\downarrow$ on CIFAR-10 32$\times$32 dataset (Score-SDE), evaluated on 50k samples. Numbers in parentheses indicate the solver order.}
    \label{tab:fid_cifar10_score_sde_cali}
    \renewcommand{\arraystretch}{1.2}
    \setlength{\tabcolsep}{5pt}
    \begin{tabular}{lccccccccccc}
        \toprule
        \multirow{2}{*}{\textbf{Model}} & \multicolumn{11}{c}{\textbf{NFE}}\\
        \cmidrule(lr){2-12}
        & \textbf{5} & \textbf{6} & \textbf{8} & \textbf{10} & \textbf{12} & \textbf{15} & \textbf{20} & \textbf{25} & \textbf{30} & \textbf{35} & \textbf{40} \\
        \midrule
        DPM-Solver-v3 (3) & \textbf{12.76} & \textbf{7.39} & \textbf{3.93} & \textbf{3.39} & \textbf{3.23} & 2.89 & 2.69 & 2.62 & 2.58 & 2.58 & 2.57 \\
        DC-Solver (3)     & 48.27 & 22.65 & 6.46 & 4.64 & 5.51 & 4.39 & 3.86 & 2.68 & 2.59 & 2.52 & \textbf{2.48} \\
        \rowcolor{gray!20}
        RBF-Solver (3)    & 26.65 & 13.05 & 5.32 & 4.13 & 3.84 & 3.02 & 2.66 & 2.56 & 2.51 & 2.50 & 2.49 \\
        \rowcolor{gray!20}
        RBF-Solver (4)    & 20.17 & 12.45 & 5.71 & 4.26 & 3.90 & \textbf{2.87} & \textbf{2.60} & \textbf{2.54} & \textbf{2.49} & \textbf{2.49} & \textbf{2.48} \\
        \bottomrule
    \end{tabular}
\end{table}
\subsubsection{CIFAR-10 (EDM) Results for Auxiliary-Tuning Samplers}
\label{app:cifar10_edm_cali}

\paragraph{Environment Setup}
\begin{itemize}
    \item \textbf{GPU}: Single NVIDIA RTX~4090
    \item \textbf{Model}: EDM~\cite{karras2022elucidating}, \url{https://github.com/NVlabs/edm}
    \item \textbf{Checkpoint}: \texttt{cifar10‑32x32‑uncond‑vp}
\end{itemize}

\paragraph{Sampler Settings}
\begin{itemize}
    \item \textbf{DPM-Solver-v3}: The EMS statistics files are obtained from the official repository (\url{https://github.com/thu-ml/DPM-Solver-v3/tree/main/codebases/edm}). Following the reference implementation, we use \texttt{0.002\_80.0\_1200\_1024} when \(\text{NFE} < 10\); otherwise, we use \texttt{0.002\_80.0\_120\_4096}.
    \item \textbf{DC-Solver}: As proposed in \cite{dc_solver}, we search the dynamic-compensation ratios using 10 images sampled with DDIM~\cite{song2020denoising} at NFE$=$10. We adopt the \(B_1(h)\) variant recommended by \cite{zhao2023unipc} for unconditional generation.
    \item \textbf{RBF-Solver}: The shape parameters are optimized using a target set of 128 samples generated by running UniPC with NFE$=$200. The solver is also evaluated with order$=$4.
\end{itemize}

\paragraph{Experiment Results}
As shown in Table \ref{tab:fid_cifar10_edm_cali}, both the order$=$3 and order$=$4 variants of RBF-Solver consistently outperform the competing samplers for NFE ≥ 15, albeit by a small margin.

\begin{table}[h]
    \centering
    \caption{Sample quality measured by FID $\downarrow$ on CIFAR-10 32$\times$32 dataset (EDM), evaluated on 50k samples. Numbers in parentheses indicate the solver order.}
    \label{tab:fid_cifar10_edm_cali}
    \renewcommand{\arraystretch}{1.2}
    \setlength{\tabcolsep}{5pt}
    \begin{tabular}{lccccccccccc}
        \toprule
        \multirow{2}{*}{\textbf{Model}} & \multicolumn{11}{c}{\textbf{NFE}}\\
        \cmidrule(lr){2-12}
        & \textbf{5} & \textbf{6} & \textbf{8} & \textbf{10} & \textbf{12} & \textbf{15} & \textbf{20} & \textbf{25} & \textbf{30} & \textbf{35} & \textbf{40} \\
        \midrule
        DPM-Solver-v3 (3) & \textbf{12.19} & 8.56 & \textbf{3.50} & \textbf{2.51} & \textbf{2.23} & 2.10 & 2.02 & 2.00 & 2.00 & 1.99 & 1.99 \\
        DC-Solver (3)       & 29.69 & 12.52 & 5.21 & 2.88 & 3.15 & 2.37 & 2.07 & 2.16 & 2.15 & 2.11 & 2.17 \\
        \rowcolor{gray!20}
        RBF-Solver (3)    & 22.25 & 10.78 & 3.77 & 2.88 & 2.37 & 2.06 & \textbf{1.99} & \textbf{1.98} & \textbf{1.98} & \textbf{1.98} & \textbf{1.98} \\
        \rowcolor{gray!20}
        RBF-Solver (4)    & 14.91 & \textbf{7.67} & 4.88 & 3.71 & 2.50 & \textbf{2.05} & 2.00 & 1.99 & 1.99 & \textbf{1.98} & \textbf{1.98} \\
        \bottomrule
    \end{tabular}
\end{table}

\begin{table*}[p]
  \centering
  
  \caption{Qualitative comparison of different samplers on CIFAR-10 (EDM). Each image is randomly sampled and solver order$=$3. Columns correspond to samplers; rows indicate NFE values of 5, 10, and 15.}
  
  \label{tab:cifar10_edm_cali}
  \renewcommand{\arraystretch}{1.02}
  \setlength{\tabcolsep}{2pt}

  \begin{adjustbox}{max width=\textwidth}
  \begin{tabular}{@{}C C C@{}}
    \toprule
    \textbf{DPM-Solver-v3} & \textbf{DC-Solver} & \textbf{RBF-Solver} \\
    \midrule
    \multicolumn{3}{c}{\textbf{NFE = 5}} \\[2pt]
    \includegraphics[width=\linewidth]{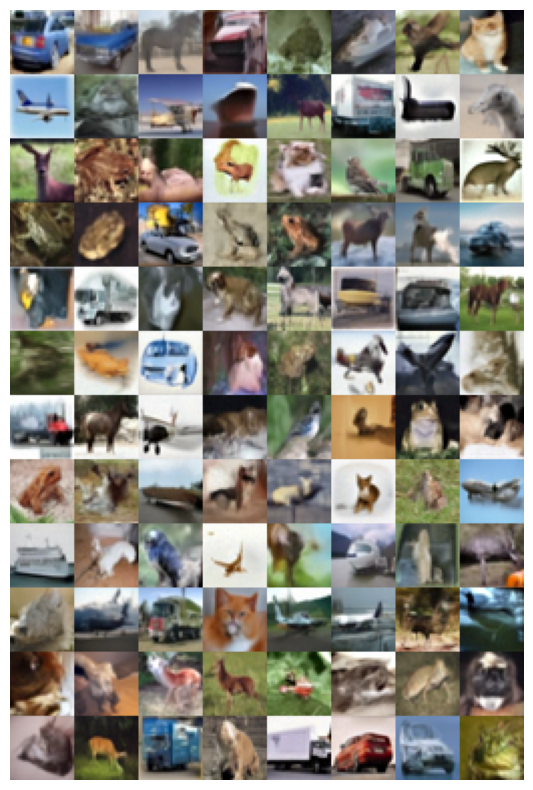} &
    \includegraphics[width=\linewidth]{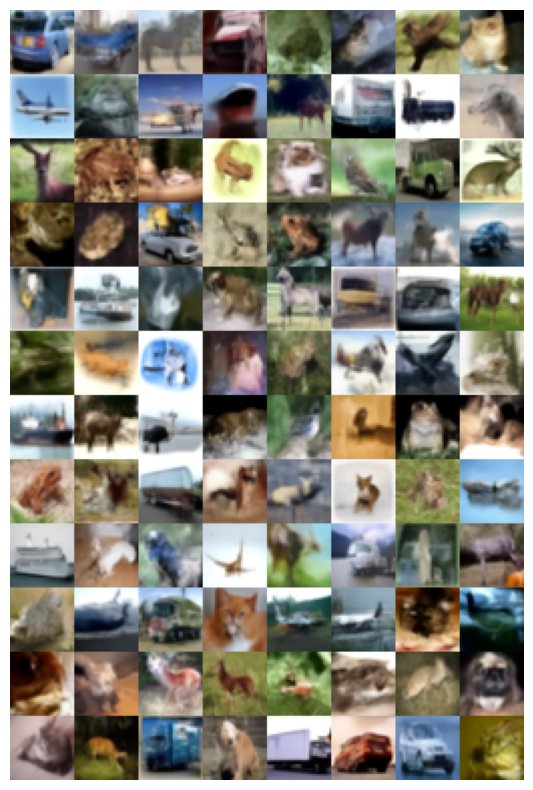} &
    \includegraphics[width=\linewidth]{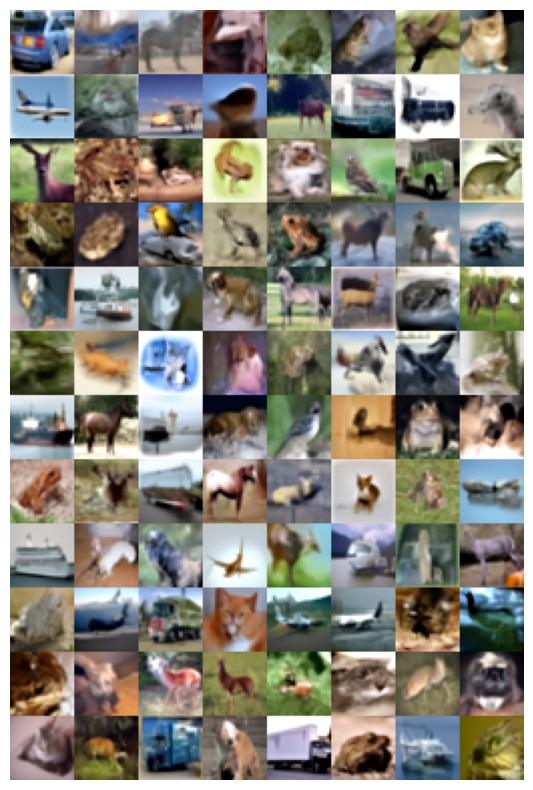} \\
    \midrule
    \multicolumn{3}{c}{\textbf{NFE = 10}} \\[2pt]
    \includegraphics[width=\linewidth]{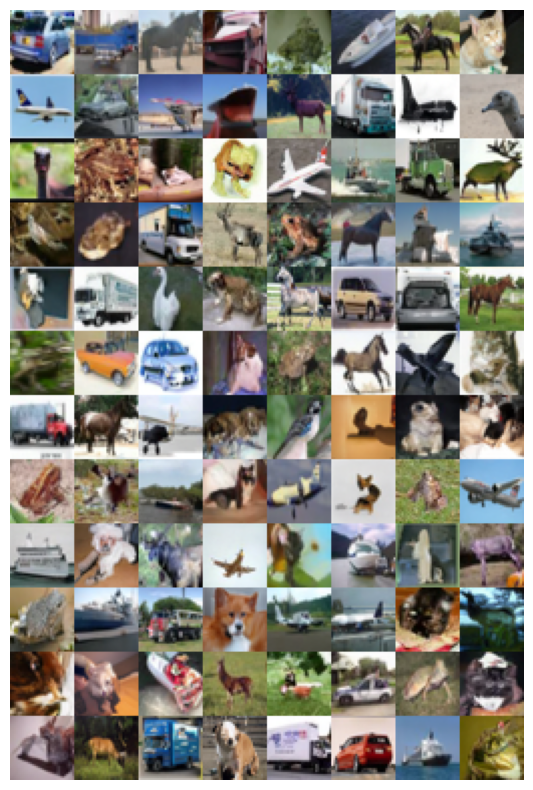} &
    \includegraphics[width=\linewidth]{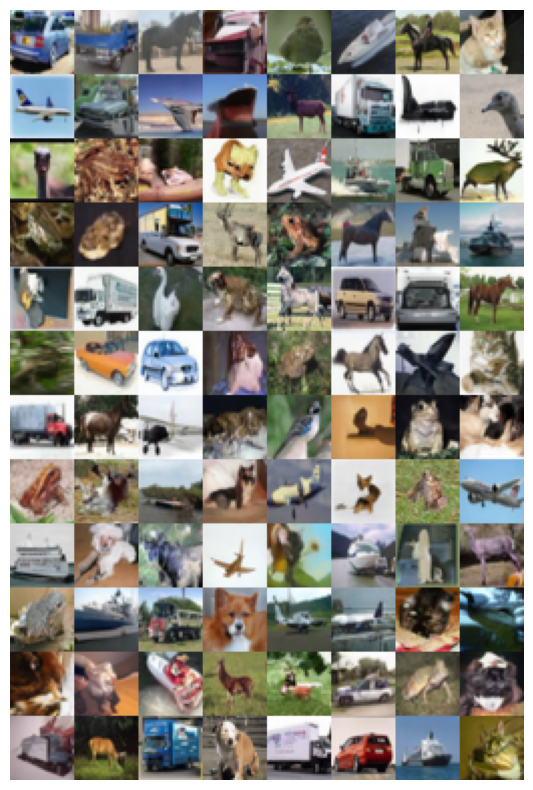} &
    \includegraphics[width=\linewidth]{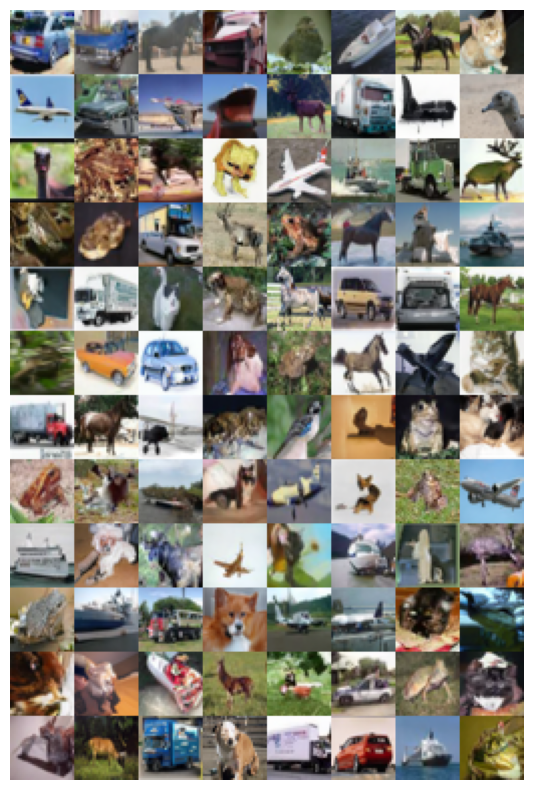} \\
    \midrule
    \multicolumn{3}{c}{\textbf{NFE = 15}} \\[2pt]
    \includegraphics[width=\linewidth]{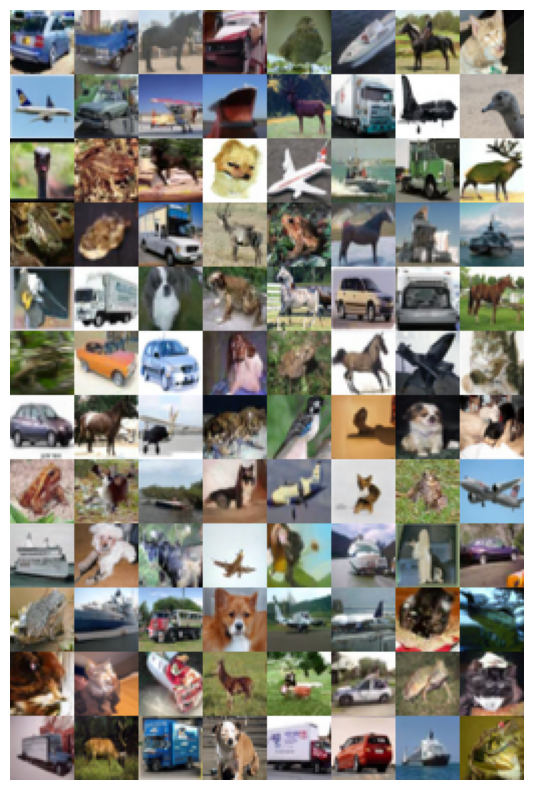} &
    \includegraphics[width=\linewidth]{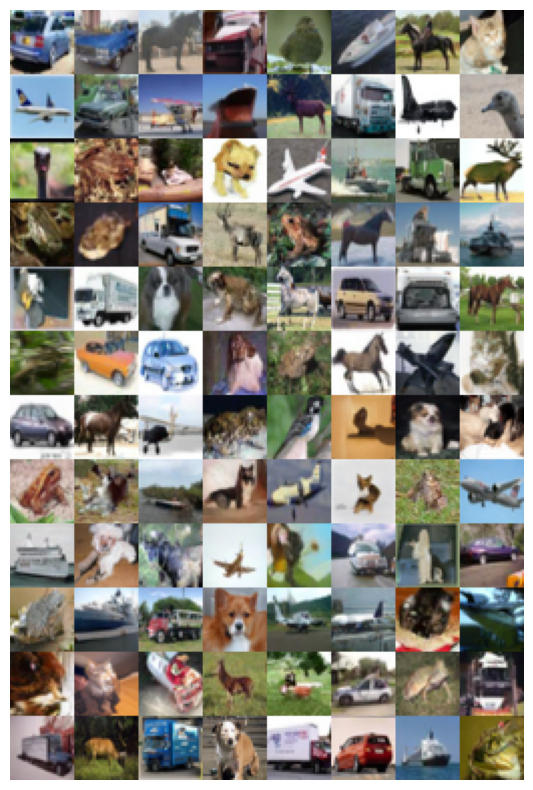} &
    \includegraphics[width=\linewidth]{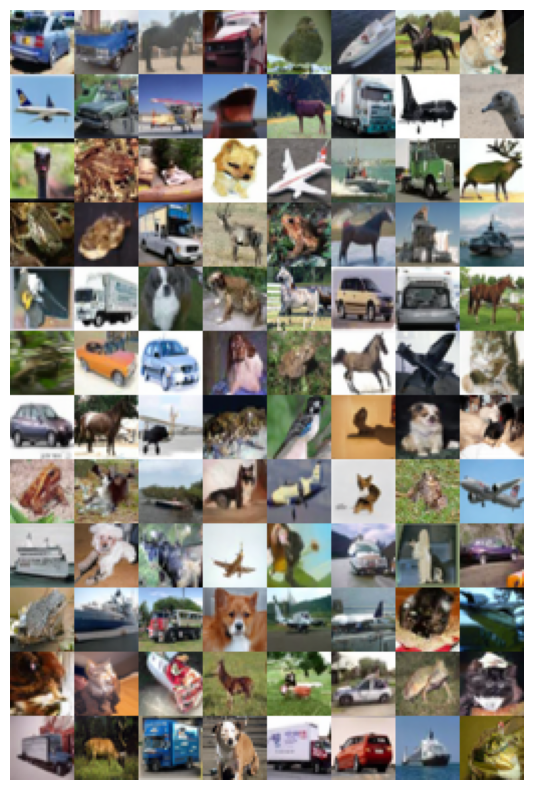} \\
    \bottomrule
  \end{tabular}
  \end{adjustbox}
\end{table*}

\subsubsection{CIFAR-10 (EDM) Results: Comparison with AMED-Solver and AMED-Plugin}
\label{app:cifar10_edm_amed}

\paragraph{Environment Setup}
\begin{itemize}
    \item \textbf{GPU}: Single NVIDIA RTX~3090 Ti
    \item \textbf{Model}: EDM~\cite{karras2022elucidating}, \url{https://github.com/NVlabs/edm}
    \item \textbf{Checkpoint}: \texttt{cifar10‑32x32‑uncond‑vp}
\end{itemize}

\paragraph{Sampler Settings}
\begin{itemize}
    \item \textbf{Common settings}: The multistep setting is applied to AMED-Plugin and RBF-Solver, and AMED-Solver is inherently singlestep. The solver order is set to 2 for AMED-Solver, 4 for AMED-Plugin, and both 3 and 4 for RBF-Solver. In both AMED-Solver and AMED-Plugin, geometric coefficients $(s_n,c_n)$ for NFE values of 5, 7, and 9 are adopted from the official repository (\url{https://github.com/zju-pi/diff-sampler/tree/main/amed-solver-main}). For other NFE values, we independently train the corresponding coefficients using the official training script on a dataset of 10k images. Following \cite{amed_solver}, we apply analytical first step (AFS) to both AMED-Solver and AMED-Plugin, and therefore report results only on odd NFEs.
    \item \textbf{AMED-Solver}: As proposed in \cite{amed_solver}, the teacher trajectories are generated using EDM Heun sampler~\cite{karras2022elucidating} with doubled NFE (i.e., one intermediate time step is inserted between each pair of adjacent sampling steps).
        \begin{itemize}
            \item NFE 5 coefficients: \texttt{cifar10-4-5-amed-heun-1-uni1.0=afs}
            \item NFE 7 coefficients: \texttt{cifar10-5-7-amed-heun-1-uni1.0=afs}
            \item NFE 9 coefficients : \texttt{cifar10-6-9-amed-heun-1-uni1.0=afs}
        \end{itemize}
    \item \textbf{AMED-Plugin}: Following the approach in \cite{amed_solver}, the student sampler is iPNDM~\cite{Zhang2023fast} (order 4), and the teacher trajectories are generated using the same solver, inserting two intermediate time steps between each pair of adjacent sampling steps.
        \begin{itemize}
            \item NFE 5 coefficients : \texttt{cifar10-4-5-ipndm-ipndm-2-poly7.0=afs}
            \item NFE 7 coefficients : \texttt{cifar10-5-7-ipndm-ipndm-2-poly7.0=afs}
            \item NFE 9 coefficients : \texttt{cifar10-6-9-ipndm-ipndm-2-poly7.0=afs}
        \end{itemize}

    \item \textbf{RBF-Solver}: The shape parameters are optimized using a target set of 128 samples generated by running UniPC with NFE$=$200. The solver is also evaluated with order$=$4.
\end{itemize}

\paragraph{Experiment Results}
As shown in Table \ref{tab:fid_amed_cifar10_cali}, RBF-Solver achieves the best FID score in both order$=$3 and order$=$4 when the NFE $\geq$ 25.

\begin{table}[h]
    \centering
    \caption{Sample quality measured by FID $\downarrow$ on CIFAR-10 32$\times$32 (EDM), evaluated on 50k samples. Numbers in parentheses indicate the solver order.} 
    \label{tab:fid_amed_cifar10_cali}
    \renewcommand{\arraystretch}{1.2}
    \setlength{\tabcolsep}{6pt}
    \begin{tabular}{lcccccccccc}
        \toprule
        \multirow{2}{*}{\textbf{Model}} & \multicolumn{10}{c}{\textbf{NFE}}\\
        \cmidrule(lr){2-11}
        & \textbf{5} & \textbf{7} & \textbf{9} & \textbf{11} & \textbf{15} & \textbf{21} & \textbf{25} & \textbf{31} & \textbf{35} & \textbf{41}\\
        \midrule
        AMED-Solver (2)      & 7.28 & 4.20 & 3.78 & 4.00 & 3.40 & 3.35 & 3.26 & 3.06 & 2.95 & 2.62 \\
        AMED-Plugin (4)   & \textbf{6.77} & \textbf{3.80} & \textbf{2.66} & \textbf{2.31} & \textbf{2.00} & \textbf{1.97} & \textbf{1.99} & 2.07 & 2.01 & 1.99 \\
        \addlinespace[0.5ex]
        \rowcolor{gray!20}
        RBF-Solver (3)        & 22.63 & 5.94 & 3.43 & 2.68 & 2.08 & 2.00 & \textbf{1.99} & \textbf{1.98} & \textbf{1.98} & \textbf{1.97} \\
        \rowcolor{gray!20}
        RBF-Solver (4)        & 15.01 & 5.29 & 5.06 & 3.01 & 2.09 & 2.01 & \textbf{1.99} & \textbf{1.98} & \textbf{1.98} & \textbf{1.97}  \\
        \bottomrule
    \end{tabular}
\end{table}

\subsubsection{FFHQ (EDM) Results: Comparison with AMED-Solver and AMED-Plugin}
\label{app:ffhq_edm_amed}

\paragraph{Environment Setup}
\begin{itemize}
    \item \textbf{GPU}: Single NVIDIA RTX~3090 Ti
    \item \textbf{Model}: EDM~\cite{karras2022elucidating}, \url{https://github.com/NVlabs/edm}
    \item \textbf{Checkpoint}: \texttt{edm-ffhq-64x64-uncond-vp}
\end{itemize}

\paragraph{Sampler Settings}
\begin{itemize}
    \item \textbf{Common settings}: The multistep setting is applied to AMED-Plugin and RBF-Solver, and AMED-Solver is inherently singlestep. The solver order is set to 2 for AMED-Solver, 4 for AMED-Plugin, and both 3 and 4 for RBF-Solver. In both AMED-Solver and AMED-Plugin, geometric coefficients $(s_n,c_n)$ for NFE values of 5, 7, and 9 are adopted from the official repository (\url{https://github.com/zju-pi/diff-sampler/tree/main/amed-solver-main}). For other NFE values, we independently train the corresponding coefficients using the official training script on a dataset of 10k images. Following \cite{amed_solver}, we apply analytical first step (AFS) to both AMED-Solver and AMED-Plugin, and therefore report results on odd NFEs.
    \item \textbf{AMED-Solver}: As proposed in \cite{amed_solver}, the teacher trajectories are generated using EDM Heun sampler~\cite{karras2022elucidating} with doubled NFE (i.e., one intermediate time step is inserted between each pair of adjacent sampling steps).
        \begin{itemize}
            \item NFE 5 coefficients: \texttt{ffhq-4-5-amed-heun-1-uni1.0=afs}
            \item NFE 7 coefficients: \texttt{ffhq-5-7-amed-heun-1-uni1.0=afs}
            \item NFE 9 coefficients : \texttt{ffhq-6-9-amed-heun-1-uni1.0=afs}
        \end{itemize}
    \item \textbf{AMED-Plugin}: Following the approach in \cite{amed_solver}, the student sampler is iPNDM~\cite{Zhang2023fast} (order 4), inserting two intermediate time steps between each pair of adjacent sampling steps.
        \begin{itemize}
            \item NFE 5 coefficients: \texttt{ffhq-4-5-ipndm-ipndm-2-poly7.0=afs}
            \item NFE 7 coefficients: \texttt{ffhq-5-7-ipndm-ipndm-2-poly7.0=afs}
            \item NFE 9 coefficients : \texttt{ffhq-6-9-ipndm-ipndm-2-poly7.0=afs}
        \end{itemize}
    \item \textbf{RBF-Solver}: The shape parameters are optimized using a target set of 128 samples generated by running UniPC with NFE$=$200. The solver is also evaluated with order$=$4.
\end{itemize}

\paragraph{Experiment Results}
As shown in Table~\ref{tab:fid_amed_ffhq_cali}, RBF-Solver achieves the best FID score with the order$=$3 at NFE$=$11, while the order$=$4 variant achieves the best sample quality when NFE $\geq$ 15.

\begin{table}[h]
    \centering
    \caption{Sample quality measured by FID $\downarrow$ on FFHQ 64$\times$64 (EDM), evaluated on 50k samples. Numbers in parentheses indicate the solver order.}
    \label{tab:fid_amed_ffhq_cali}
    \renewcommand{\arraystretch}{1.2}
    \setlength{\tabcolsep}{6pt}
    \begin{tabular}{lcccccccccc}
        \toprule
        \multirow{2}{*}{\textbf{Model}} & \multicolumn{10}{c}{\textbf{NFE}}\\
        \cmidrule(lr){2-11}
        & \textbf{5} & \textbf{7} & \textbf{9} & \textbf{11} & \textbf{15} & \textbf{21} & \textbf{25} & \textbf{31} & \textbf{35} & \textbf{41}\\
        \midrule
        AMED-Solver (2)      & 14.80 & 8.82 & 6.30 & 5.29 & 4.40 & 3.85 & 3.53 & 3.33 & 3.21 & 2.93 \\
        AMED-Plugin (4)   & \textbf{12.48} & \textbf{6.64} & \textbf{4.24} & 3.70 & 2.93 & 2.61 & 2.57 & 2.53 & 2.50 & 2.55 \\
        \addlinespace[0.5ex]
        \rowcolor{gray!20}
        RBF-Solver (3)         & 19.58 & 7.07 & 4.61 & \textbf{3.49} & 2.70 & 2.42 & 2.37 & 2.34 & 2.34 & 2.34\\
        \rowcolor{gray!20}
        RBF-Solver (4)        & 15.58 & 7.93 & 8.71 & 4.19 & \textbf{2.60} & \textbf{2.34} & \textbf{2.30} & \textbf{2.29} & \textbf{2.30} & \textbf{2.31} \\
        \bottomrule
    \end{tabular}
\end{table}

\subsection{Conditional Experiment Results for Auxiliary-Tuning Samplers}

\subsubsection{ImageNet 256$\times$256 Experiment Results for Auxiliary-Tuning Samplers}
\label{app:imagenet256_cali}

\paragraph{Environment Setup}
\begin{itemize}
    \item \textbf{GPUs}: 8-way NVIDIA RTX~4090
    \item \textbf{Model}: Guided-Diffusion~\cite{dhariwal2021diffusion},
    \url{https://github.com/openai/guided-diffusion}
    \item \textbf{Checkpoint}: \texttt{256x256\_classifier.pt, 256x256\_diffusion.pt}
\end{itemize}

\paragraph{Sampler Settings}
\begin{itemize}
    \item \textbf{DPM-Solver-v3}: The EMS statistics files are obtained from the official repository  
          (\url{https://github.com/thu-ml/DPM-Solver-v3/tree/main/codebases/guided-diffusion}).  
          Following Appendix.~J.2 of \cite{dpm_solver_v3}, the same EMS statistics are used for all guidance scale settings.
    \item \textbf{DC-Solver}: As proposed in \cite{dc_solver}, The dynamic-compensation ratios are searched using 10 images sampled with DDIM~\cite{song2020denoising} at NFE$=$1000. Due to memory issues, the search is conducted on an NVIDIA B200 GPU. We adopt the \(B_2(h)\) variant recommended by \cite{zhao2023unipc} for conditional generation.
    \item \textbf{RBF-Solver}: The shape parameters are optimized by averaging the results of 20 runs, each performed on a batch of 16 images randomly drawn from a 128-image target set generated with UniPC NFE$=$200. The solver is also evaluated with order$=$3.
\end{itemize}

\paragraph{Experiment Results}
As shown in Table~\ref{tab:fid_imagenet256_cali}, RBF-Solver exhibits improved FID scores with increasing guidance scale, achieving the best results across all NFE settings at a scale of 8.0.

\begin{table}[h]
    \centering
    \caption{Sample quality measured by FID $\downarrow$ on ImageNet 256$\times$256 (Guided-Diffusion), evaluated on 10k samples. Numbers in parentheses indicate the solver order.}
    \label{tab:fid_imagenet256_cali}
    \renewcommand{\arraystretch}{1.2}
    \setlength{\tabcolsep}{6pt}
    \setlength{\tabcolsep}{5.0pt}
    \resizebox{\linewidth}{!}{
    \begin{tabular}{llcccccccc}
        \toprule
        & & \multicolumn{8}{c}{\textbf{NFE}} \\
        \cmidrule(lr){3-10}
        & & \textbf{5} & \textbf{6} & \textbf{8} & \textbf{10} & \textbf{12} & \textbf{15} & \textbf{20} & \textbf{25} \\
        \midrule

        \multicolumn{10}{l}{\textbf{Guidance Scale = 2.0}} \\
        \addlinespace[0.5ex]
        & DPM-Solver-v3 (2)      & \textbf{14.75} & \textbf{10.99} & \textbf{8.79} & \textbf{8.05} & \textbf{7.74} & \textbf{7.61} & \textbf{7.35} & \textbf{7.23} \\
        & DC-Solver (3)          & 15.78 & 11.60 & 9.29 & 8.19 & 7.77 & \textbf{7.61} & 7.46 & 7.47 \\
        \rowcolor{gray!20}
        & RBF-Solver (2)         & 15.05 & 11.45 & 9.11 & 8.28 & 7.87 & 7.71 & 7.43 & 7.31 \\
        \rowcolor{gray!20}
        & RBF-Solver (3)         & 14.77 & 11.11 & 9.00 & 8.40 & 8.05 & 7.78 & 7.45 & 7.31 \\
        \midrule

        \multicolumn{10}{l}{\textbf{Guidance Scale = 4.0}} \\
        \addlinespace[0.5ex]
        & DPM-Solver-v3 (2)      & 26.37 & 16.35 & 10.31 & 9.04 & 8.48 & 8.23 & 8.06 & 8.03 \\
        & DC-Solver (3)          & 19.66 & 15.08 & 11.20 & 9.65 & \textbf{8.18} & \textbf{8.18} & \textbf{7.99} & 8.03 \\
        \rowcolor{gray!20}
        & RBF-Solver (2)         & \textbf{17.76} & 12.47 & 9.55 & 8.95 & 8.58 & 8.35 & 8.08 & \textbf{8.01} \\
        \rowcolor{gray!20}
        & RBF-Solver (3)         & 18.33 & \textbf{12.03} & \textbf{9.22} & \textbf{8.83} & 8.52 & 8.26 & 8.11 & 8.02 \\
        \midrule

        \multicolumn{10}{l}{\textbf{Guidance Scale = 6.0}} \\
        \addlinespace[0.5ex]
        & DPM-Solver-v3 (2)      & 53.52 & 37.72 & 20.35 & 12.90 & 10.44 & 9.26 & 8.98 & 8.82 \\
        & DC-Solver (3)          & 36.71 & 22.44 & 13.19 & 12.49 & 10.95 & \textbf{8.86} & 8.87 & \textbf{8.75} \\
        \rowcolor{gray!20}
        & RBF-Solver (2)         & 26.78 & 16.21 & 11.00 & 9.74 & 9.39 & 8.97 & 8.95 & 8.79 \\
        \rowcolor{gray!20}
        & RBF-Solver (3)         & \textbf{24.91} & \textbf{15.48} & \textbf{10.72} & \textbf{9.65} & \textbf{9.31} & 8.91 & \textbf{8.79} & 8.80 \\
        \midrule

        \multicolumn{10}{l}{\textbf{Guidance Scale = 8.0}} \\
        \addlinespace[0.5ex]
        & DPM-Solver-v3 (2)      & 79.14 & 64.93 & 41.99 & 26.08 & 17.48 & 11.63 & 9.99 & 9.64 \\
        & DC-Solver (3)          & 58.47 & 45.53 & 20.25 & 17.34 & 13.31 & 10.93 & 9.46 & 9.66 \\
        \rowcolor{gray!20}
        & RBF-Solver (2)         & 42.61 & \textbf{24.79} & \textbf{13.79} & 11.25 & 10.30 & 9.69 & 9.46 & \textbf{9.42} \\
        \rowcolor{gray!20}
        & RBF-Solver (3)         & \textbf{38.81} & \textbf{24.79} & 14.04 & \textbf{10.98} & \textbf{10.21} & \textbf{9.63} & \textbf{9.44} & 9.47 \\
        \bottomrule
    \end{tabular}}
\end{table}

\begin{table*}[p]
  \centering
\caption{Qualitative comparison of different samplers on ImageNet 256$\times$256. Each image is randomly sampled with guidance scale 8.0. DPM-Solver-v3 and RBF-Solver use order 2; DC-Solver uses order 3. Columns correspond to samplers; rows indicate NFE values of 5, 10, and 15.}
  \label{tab:imagenet256_scale4_samples}
  \renewcommand{\arraystretch}{1.02}
  \setlength{\tabcolsep}{2pt}

  \begin{adjustbox}{max width=\textwidth}
  \begin{tabular}{@{}C C C@{}}
    \toprule
    \textbf{DPM-Solver-v3} & \textbf{DC-Solver} & \textbf{RBF-Solver} \\
    \midrule
    \multicolumn{3}{c}{\textbf{NFE = 5}} \\[2pt]
    \includegraphics[width=\linewidth]{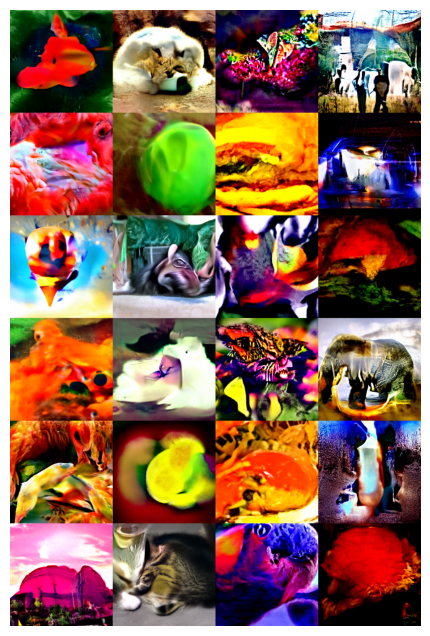} &
    \includegraphics[width=\linewidth]{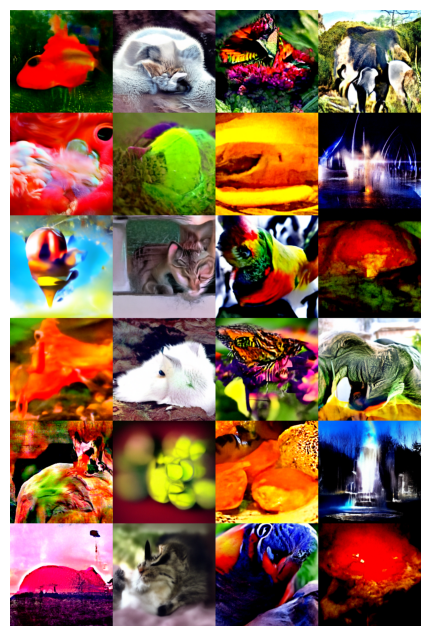} &
    \includegraphics[width=\linewidth]{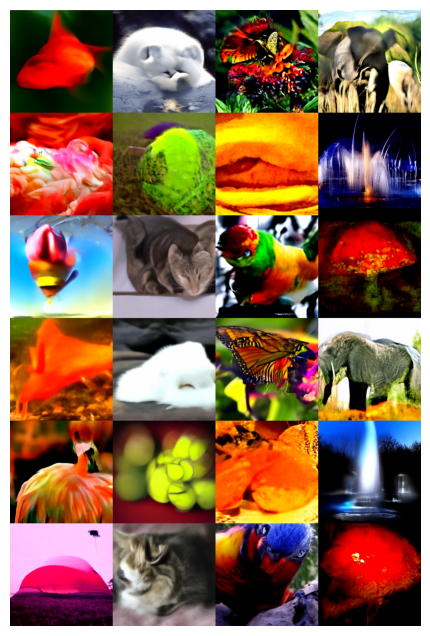} \\
    \midrule
    \multicolumn{3}{c}{\textbf{NFE = 10}} \\[2pt]
    \includegraphics[width=\linewidth]{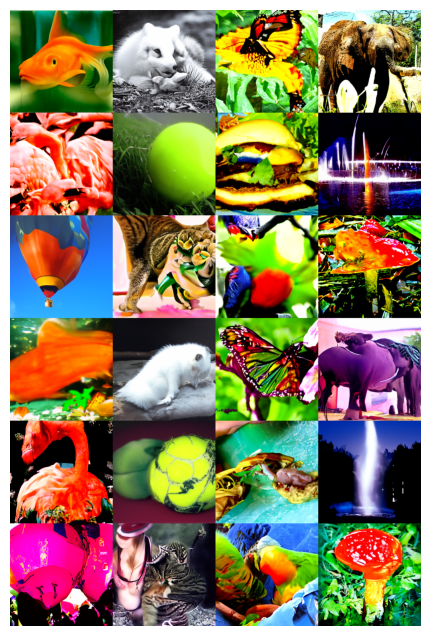} &
    \includegraphics[width=\linewidth]{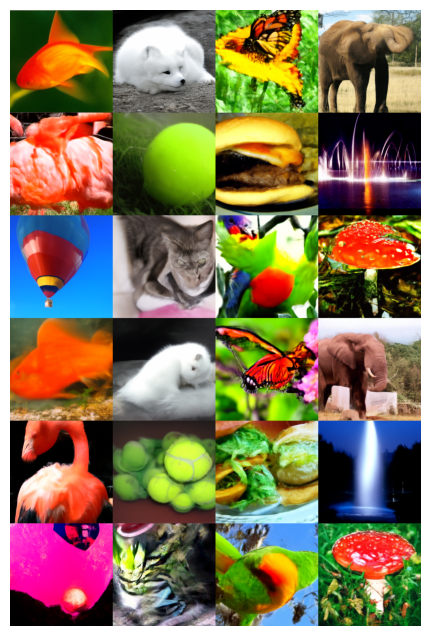} &
    \includegraphics[width=\linewidth]{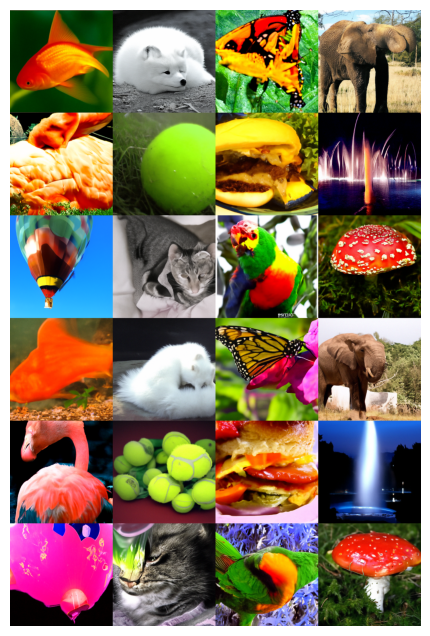} \\
    \midrule
    \multicolumn{3}{c}{\textbf{NFE = 15}} \\[2pt]
    \includegraphics[width=\linewidth]{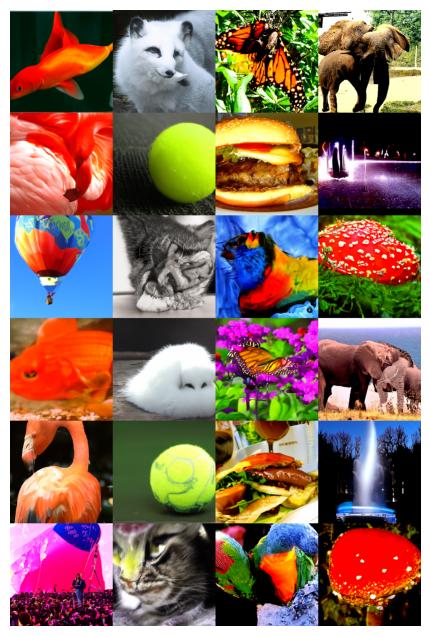} &
    \includegraphics[width=\linewidth]{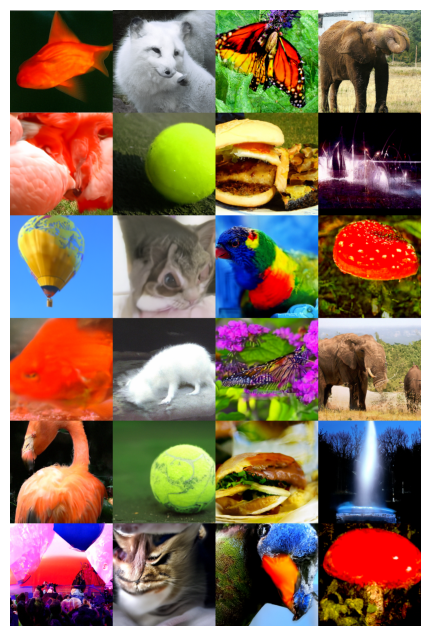} &
    \includegraphics[width=\linewidth]{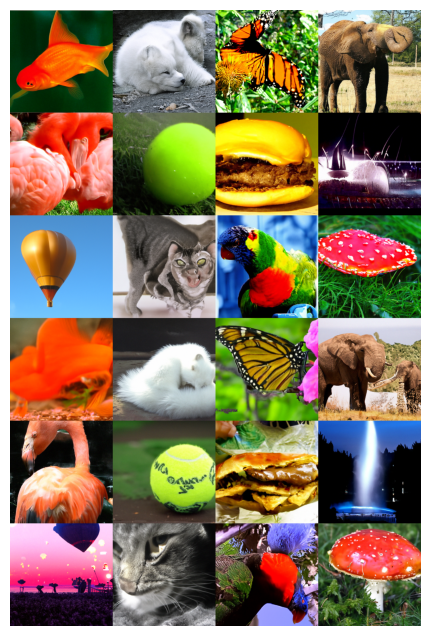} \\
    \bottomrule
  \end{tabular}
  \end{adjustbox}
\end{table*}

\subsubsection{Stable-Diffusion v1.4 Experiment Results for Auxiliary-Tuning Samplers}
\label{app:stable_cali}

\paragraph{Environment Setup}
\begin{itemize}
    \item \textbf{GPU}: Single NVIDIA RTX~4090
    \item \textbf{Model}: Stable-Diffusion~\cite{rombach2022high},
    \url{https://github.com/CompVis/stable-diffusion}
    \item \textbf{Checkpoint}: \texttt{sd-v1-4.ckpt}
\end{itemize}

\paragraph{Sampler Settings}
\begin{itemize}
    \item \textbf{DPM-Solver-v3}: The EMS statistics files are obtained from the official repository (\url{https://github.com/thu-ml/DPM-Solver-v3/tree/main/codebases/stable-diffusion}).
The EMS statistics are provided only for scales 1.5 and 7.5, and DPM-Solver-v3 are evaluated exclusively at these scales.
    \item \textbf{DC-Solver}: As proposed in \cite{dc_solver}, the dynamic-compensation ratios are searched using 10 images sampled with DDIM~\cite{song2020denoising} at NFE$=$1000. Due to memory issues, the search is conducted on an NVIDIA B200 GPU. We adopt the \(B_2(h)\) variant recommended by \cite{zhao2023unipc} for conditional generation.
    \item \textbf{RBF-Solver}: The shape parameters are optimized by averaging the results of 20 runs, each performed on a batch of 6 images randomly drawn from a 128-image target set generated with UniPC NFE$=$200. The solver is also evaluated with order$=$3.
\end{itemize}

\paragraph{Experiment Results}
As shown in Table \ref{tab:fid_stable_cali}, RBF-Solver yields strong performance in terms of cosine similarity and RMSE at the higher guidance scales of 7.5 and 9.5.

\begin{table*}[h]
    \centering
    \caption{CLIP embedding cosine similarity (↑) and RMSE loss (↓) with Stable Diffusion v1.4 for guidance scales 1.5, 3.5, 5.5, 7.5 and 9.5, evaluated on 10k samples. Bold numbers mark the best value per NFE in each metric.}
    \label{tab:fid_stable_cali}
    \renewcommand{\arraystretch}{0.92}
    \setlength{\tabcolsep}{6pt}


  \begin{adjustbox}{max height=\textheight}
      \begin{minipage}{\textwidth}
      \centering

    \ContinuedFloat
    \caption*{(a) CLIP cosine similarity (↑).}

  \begin{tabular}{llccccccc}
    \toprule
            & & \multicolumn{7}{c}{\textbf{NFE}} \\
    \cmidrule(lr){3-9}
            & & \textbf{5} & \textbf{6} & \textbf{8} & \textbf{10} & \textbf{12} & \textbf{15} & \textbf{20} \\
    \midrule
    \multicolumn{9}{l}{\textbf{Guidance Scale = 1.5}}\\[0.4ex]
      & DPM-Solver-v3 (2)  & \textbf{0.9116} & \textbf{0.9316} & \textbf{0.9517} & \textbf{0.9682} & 0.9873 & \textbf{0.9887} & 0.9938 \\
      & DC-Solver (3)      & 0.8979 & 0.9214 & 0.9497 & 0.9678 & \textbf{0.9878} & 0.9878 & \textbf{0.9941} \\
\rowcolor{gray!20}
      & RBF-Solver (2)     & 0.8887 & 0.9165 & 0.9453 & 0.9614 & 0.9717 & 0.9810 & 0.9888 \\
\rowcolor{gray!20}
      & RBF-Solver (3)     & 0.8926 & 0.9199 & 0.9468 & 0.9629 & 0.9731 & 0.9814 & 0.9893 \\
    \midrule
    \multicolumn{9}{l}{\textbf{Guidance Scale = 3.5}}\\[0.4ex]
      & DC-Solver (3)      & 0.8916 & 0.9087 & 0.9277 & \textbf{0.9463} & \textbf{0.9561} & \textbf{0.9673} & \textbf{0.9790} \\
\rowcolor{gray!20}
      & RBF-Solver (2)     & 0.8916 & 0.9097 & \textbf{0.9302} & 0.9448 & 0.9551 & 0.9663 & 0.9785 \\
\rowcolor{gray!20}
      & RBF-Solver (3)     & \textbf{0.8931} & \textbf{0.9106} & \textbf{0.9302} & 0.9443 & 0.9546 & 0.9663 & 0.9780 \\
    \midrule
    \multicolumn{9}{l}{\textbf{Guidance Scale = 5.5}}\\[0.4ex]
      & DC-Solver (3)      & \textbf{0.8877} & 0.8887 & 0.9175 & \textbf{0.9331} & 0.9419 & 0.9487 & 0.9644 \\
\rowcolor{gray!20}
      & RBF-Solver (2)     & 0.8853 & \textbf{0.9004} & \textbf{0.9194} & 0.9321 & \textbf{0.9424} & \textbf{0.9531} & \textbf{0.9663} \\
\rowcolor{gray!20}
      & RBF-Solver (3)     & 0.8843 & 0.8999 & 0.9189 & 0.9316 & 0.9409 & 0.9521 & 0.9648 \\
    \midrule
    \multicolumn{9}{l}{\textbf{Guidance Scale = 7.5}}\\[0.4ex]
      & DPM-Solver-v3 (2)  & 0.8638 & 0.8843 & 0.9082 & 0.9219 & \textbf{0.9321} & \textbf{0.9434} & \textbf{0.9561} \\
      & DC-Solver (3)      & 0.8721 & 0.8809 & 0.9004 & 0.9204 & 0.9238 & 0.9424 & 0.9497 \\
\rowcolor{gray!20}
      & RBF-Solver (2)     & \textbf{0.8760} & \textbf{0.8931} & \textbf{0.9106} & \textbf{0.9233} & \textbf{0.9321} & 0.9429 & \textbf{0.9561} \\
\rowcolor{gray!20}
      & RBF-Solver (3)     & 0.8735 & 0.8906 & 0.9097 & 0.9224 & 0.9316 & 0.9424 & 0.9551 \\
    \midrule
    \multicolumn{9}{l}{\textbf{Guidance Scale = 9.5}}\\[0.4ex]
      & DC-Solver (3)      & 0.8218 & 0.8540 & 0.8945 & 0.9062 & 0.9155 & 0.9268 & 0.9389 \\
\rowcolor{gray!20}
      & RBF-Solver (2)     & 0.8613 & \textbf{0.8828} & \textbf{0.9033} & \textbf{0.9150} & \textbf{0.9238} & \textbf{0.9346} & \textbf{0.9468} \\
\rowcolor{gray!20}
      & RBF-Solver (3)     & \textbf{0.8618} & 0.8809 & 0.9019 & 0.9146 & \textbf{0.9238} & 0.9341 & \textbf{0.9468} \\
    \bottomrule
  \end{tabular}

\vspace{1em}


   \ContinuedFloat
    \caption*{(b) RMSE loss (↓).}
  
  \begin{tabular}{llccccccc}
    \toprule
            & & \multicolumn{7}{c}{\textbf{NFE}} \\
    \cmidrule(lr){3-9}
            & & \textbf{5} & \textbf{6} & \textbf{8} & \textbf{10} & \textbf{12} & \textbf{15} & \textbf{20} \\
    \midrule
    \multicolumn{9}{l}{\textbf{Guidance Scale = 1.5}}\\[0.4ex]
      & DPM-Solver-v3 (2)  & \textbf{0.2019} & \textbf{0.1756} & \textbf{0.1344} & \textbf{0.1040} & \textbf{0.0806} & 0.0584 & 0.0381 \\
      & DC-Solver (3)      & 0.2180 & 0.1893 & 0.1441 & 0.1101 & 0.0810 & \textbf{0.0576} & \textbf{0.0356} \\
\rowcolor{gray!20}
      & RBF-Solver (2)     & 0.2240 & 0.1926 & 0.1506 & 0.1203 & 0.0976 & 0.0741 & 0.0521 \\
\rowcolor{gray!20}
      & RBF-Solver (3)     & 0.2202 & 0.1905 & 0.1500 & 0.1174 & 0.0932 & 0.0720 & 0.0503 \\
    \midrule
    \multicolumn{9}{l}{\textbf{Guidance Scale = 3.5}}\\[0.4ex]
      & DC-Solver (3)      & \textbf{0.3472} & \textbf{0.3220} & \textbf{0.2793} & \textbf{0.2426} & \textbf{0.2122} & \textbf{0.1748} & 0.1288 \\
\rowcolor{gray!20}
      & RBF-Solver (2)     & 0.3474 & 0.3239 & 0.2829 & 0.2462 & 0.2145 & 0.1756 & \textbf{0.1284} \\
\rowcolor{gray!20}
      & RBF-Solver (3)     & 0.3493 & 0.3276 & 0.2863 & 0.2488 & 0.2167 & 0.1769 & 0.1309 \\
    \midrule
    \multicolumn{9}{l}{\textbf{Guidance Scale = 5.5}}\\[0.4ex]
      & DC-Solver (3)      & 0.4663 & \textbf{0.4386} & \textbf{0.3963} & \textbf{0.3573} & \textbf{0.3294} & 0.3008 & 0.2360 \\
\rowcolor{gray!20}
      & RBF-Solver (2)     & 0.4659 & 0.4437 & 0.4018 & 0.3639 & 0.3302 & \textbf{0.2871} & \textbf{0.2267} \\
\rowcolor{gray!20}
      & RBF-Solver (3)     & \textbf{0.4647} & 0.4411 & 0.3989 & 0.3631 & 0.3315 & 0.2896 & 0.2333 \\
    \midrule
    \multicolumn{9}{l}{\textbf{Guidance Scale = 7.5}}\\[0.4ex]
      & DPM-Solver-v3 (2)  & 0.6252 & 0.5835 & 0.5191 & 0.4722 & 0.4307 & \textbf{0.3847} & \textbf{0.3229} \\
      & DC-Solver (3)      & 0.5840 & 0.5529 & 0.5159 & 0.4618 & 0.4388 & 0.3865 & 0.3539 \\
\rowcolor{gray!20}
      & RBF-Solver (2)     & 0.5664 & 0.5407 & 0.5002 & 0.4624 & 0.4310 & 0.3878 & 0.3257 \\
\rowcolor{gray!20}
      & RBF-Solver (3)     & \textbf{0.5655} & \textbf{0.5373} & \textbf{0.4956} & \textbf{0.4592} & \textbf{0.4265} & 0.3869 & 0.3273 \\
    \midrule
    \multicolumn{9}{l}{\textbf{Guidance Scale = 9.5}}\\[0.4ex]
      & DC-Solver (3)      & 0.7142 & 0.6697 & 0.6126 & 0.5850 & 0.5534 & 0.5130 & 0.4577 \\
\rowcolor{gray!20}
      & RBF-Solver (2)     & 0.6666 & 0.6347 & 0.5859 & 0.5487 & 0.5193 & 0.4757 & 0.4152 \\
\rowcolor{gray!20}
      & RBF-Solver (3)     & \textbf{0.6661} & \textbf{0.6327} & \textbf{0.5842} & \textbf{0.5462} & \textbf{0.5147} & \textbf{0.4724} & \textbf{0.4142} \\
    \bottomrule
  \end{tabular}

     \end{minipage}
\end{adjustbox}

\end{table*}

\end{document}